\documentclass{article}

\usepackage{fullpage}

\usepackage{graphicx}
\usepackage{hyperref}
\usepackage{url}
\usepackage{epsfig}  
\usepackage{amsmath,amssymb,amsthm}  
\usepackage{url}  
\usepackage{lscape}  
\usepackage[justification=raggedright]{caption}	
\usepackage{bm} 
\usepackage{algorithm,algorithmic} 
\usepackage{epstopdf} 
\usepackage[T1]{fontenc} 
\usepackage{float}
\usepackage[lofdepth,lotdepth]{subfig}
\usepackage{color}
\usepackage{mathrsfs}
\usepackage{bbm}
\usepackage{UIUCcolors}
\usepackage{tikz}
\usepackage{etoolbox}

\usetikzlibrary{calc,trees,positioning,arrows,chains,shapes.geometric,%
    decorations.pathreplacing,decorations.pathmorphing,shapes,%
    matrix,shapes.symbols,positioning,fit}
    
\usepackage{mathptmx}

\tikzstyle{decision} = [diamond, draw, fill=blue!20, 
    text width=4.5em, text badly centered, node distance=3cm, inner sep=0pt]
\tikzstyle{block} = [rectangle, draw, fill=none, 
    text width=6em, text centered, rounded corners, minimum height=4em]
\tikzstyle{line} = [draw, -latex']
\tikzstyle{cloud} = [draw, ellipse,fill=red!20, node distance=3cm,
    minimum height=2em]

\makeatletter
\let\orgdescriptionlabel\descriptionlabel
\renewcommand*{\descriptionlabel}[1]{%
  \let\orglabel\label
  \let\label\@gobble
  \phantomsection
  \edef\@currentlabel{#1}%
  \let\label\orglabel
  \orgdescriptionlabel{#1}%
}
\makeatother

\newtheorem{lem}{Lemma}
\newtheorem{thm}{Theorem}

\def\bx{\bm{x}}
\def\by{\bm{y}}
\def\bz{\bm{z}}
\def\bw{\bm{w}}
\def\bv{\bm{v}}

\def\xSp{\mathcal{X}}
\def\lossLG{L}

\newcommand{\inprod}[2]{\left\langle #1 , #2 \right\rangle }

\newcommand{\ipm}[3]{ \gamma_{#1}(#2,#3) }
\newcommand{\ipmv}[3]{ \gamma^{\text{V}}_{#1}(#2,#3) }

\def\numIter{K}
\def\iterIndex{k}
\def\lossFunc{\ell}

\newtoggle{useMeanGap}
\togglefalse{useMeanGap}

\iftoggle{useMeanGap}{
	\def\meangap{mean gap}
	
	\def\MeanGap{Mean Gap}
	\def\aMG{a}
}{
	\def\meangap{excess risk}
	
	\def\MeanGap{Excess Risk}
	\def\aMG{an}
}

\title{Adaptive Sequential Optimization with Applications to Machine Learning}

\author{Craig Wilson and Venugopal V. Veeravalli\thanks{This work was supported by the NSF under award CCF 11-11342 through the University of Illinois at Urbana-Champaign.}\\
Coordinated Science Lab and Electrical and Computer Engineering\\
       University of Illinois at Urbana-Champaign\\
       Urbana, IL 61801, USA\\
       \texttt{\{wilson60,vvv\}@illinois.edu}}

\begin{document}

\maketitle
\begin{abstract}
A framework is introduced for solving a sequence of slowly changing optimization problems, including those arising in regression and classification applications, using optimization algorithms such as stochastic gradient descent (SGD). The optimization problems change slowly in the sense that the minimizers change at either a fixed or bounded rate. A method based on estimates of the change in the minimizers and properties of the optimization algorithm is introduced for adaptively selecting the number of samples needed from the distributions underlying each problem in order to ensure that the \meangap{}, i.e., the expected gap between the loss achieved by the approximate minimizer produced by the optimization algorithm and the exact minimizer, does not exceed a target level. Experiments with synthetic and real data are used to confirm that this approach performs well.
\end{abstract}

\section{Introduction}
Consider solving a sequence of machine learning problems such as regression or classification by minimizing the expected value of a fixed loss function $\lossFunc(\bx,\bz)$ at each time $n$s:
\begin{equation}
\label{probState:seqProbs}
	\min_{\bx \in \xSp} \left\{ f_{n}(\bx) \triangleq \mathbb{E}_{\bz_{n} \sim p_{n}} \left[  \lossFunc(\bx,\bz_{n}) \right] \right\} \;\;\; \forall n \geq 1
\end{equation}
For regression, $\bz_{n}$ corresponds to the predictors and response pair at time $n$ and $\bx$ parameterizes the regression model. For classification $\bz_{n}$ corresponds to the feature and label pair at time $n$ and $\bx$ parameterizes the classifier. Although, motivated by regression and classification, our framework works for any loss function $\lossFunc(\bx,\bz)$ that satisfies certain properties discussed later. In the learning context, a \emph{task} consists of the loss function $\lossFunc(\bx,\bz)$ and the distribution $p_{n}$, and so our problem can be viewed as learning a sequence of tasks.

The problems change slowly at a constant but unknown rate in the sense that
\begin{equation}
\label{probState:slowChangeConstDef}
	\| \bx_{n}^{*} - \bx_{n-1}^{*} \| = \rho \;\;\;\;\;\;\;\;\;\;\;\;\;\; \forall n \geq 2
\end{equation}
with $\bx_{n}^{*}$ the minimizer of $f_{n}(\bx)$. In an extended version of this paper \cite{Wilson2015}, we also consider slow changes at a bounded but unknown rate
\begin{equation}
\label{probState:slowChangeDef}
	\| \bx_{n}^{*} - \bx_{n-1}^{*} \| \leq \rho \;\;\;\;\;\;\;\;\;\;\;\;\;\; \forall n \geq 2
\end{equation}

Under this model, we find approximate minimizers $\bx_{n}$ of each function $f_{n}(\bx)$ using $\numIter_{n}$ samples from distribution $p_{n}$ by applying an optimization algorithm. We evaluate the quality of our approximate minimizers $\bx_{n}$ through \aMG{} \meangap{} criterion $\epsilon_{n}$, i.e.,
\begin{equation*}
\label{probState:meanGapDef}
		\mathbb{E}\left[ f_{n}(\bx_{n}) \right] - f_{n}(\bx_{n}^{*}) \leq \epsilon_{n}
\end{equation*}
\iftoggle{useMeanGap}{
which is a standard criterion for stochastic optimization problems \cite{Nemirovski2009}.
}{
which is a standard criterion for optimization and learning problems \cite{Mohri2012}.
}
Our goal is to determine adaptively the number of samples $\numIter_{n}$ required to achieve a desired \meangap{} $\epsilon$ for each $n$ with $\rho$ unknown. As $\rho$ is unknown, we will construct estimates of $\rho$. Given an estimate of $\rho$, we determine selection rules for the number of samples $\numIter_{n}$ to achieve a target \meangap{} $\epsilon$.

\subsection{Related Work}

Our problem has connections with \textit{multi-task learning} (MTL) and \textit{transfer learning}. In multi-task learning, one tries to learn several tasks simultaneously as in \cite{Agarwal2011},\cite{Evgeniou2004}, and \cite{Zhang2012} by exploiting the relationships between the tasks. In transfer learning, knowledge from one source task is transferred to another target task either with or without additional training data for the target task \cite{Pan2010}. Multi-task learning could be applied to our problem by running a MTL algorithm each time a new task arrives, while remembering all prior tasks. However, this approach incurs a memory and computational burden. Transfer learning lacks the sequential nature of our problem. For multi-task and transfer learning, there are theoretical guarantees on regret for some algorithms \cite{Agarwal2008}.

We can also consider the \textit{concept drift} problem in which we observe a stream of incoming data that potentially changes over time, and the goal is to predict some property of each piece of data as it arrives. After prediction, we incur a loss that is revealed to us. For example, we could observe a feature $\bw_{n}$ and predict the label $y_{n}$ as in \cite{Towfic2013}. Some approaches for concept drift use iterative algorithms such as SGD, but without specific models on how the data changes. As a result, only simulation results showing good performance are available. There are also some bandit approaches in which one of a finite number of predictors must be applied to the data as in \cite{Tekin2014}. For this approach, there are regret guarantees using techniques for analyzing bandit problems.

Another relevant model is \textit{sequential supervised learning} (see \cite{Dietterich2002}) in which we observe a stream of data consisting of feature/label pairs $(\bw_{n},y_{n})$ at time $n$, with $\bw_{n}$ being the feature vector and $y_{n}$ being the label. At time $n$, we want to predict $y_{n}$ given $\bx_{n}$. One approach to this problem, studied in \cite{Fawcett1997} and \cite{Qian1988}, is to look at $L$ consecutive pairs $\{(\bw_{n-i},y_{n-i})\}_{i=1}^{L}$ and develop a predictor at time $n$ by applying a supervised learning algorithm to this training data. Another approach is to assume that there is an underlying hidden Markov model (HMM) \cite{BengioFrasconi1996}. The label $y_{n}$ represents the hidden state and the pair $(\bw_{n},\overline{y}_{n})$ represents the observation with $\overline{y}_{n}$ being a noisy version of $y_{n}$. HMM inference techniques are used to estimate $y_{n}$.

\section{Adaptive Sequential Optimization With $\rho$ Known}
\label{withRhoKnown}

For analysis, we need the following assumptions on our functions $f_{n}(\bx)$ and the optimization algorithm:

\begin{description}
\item[A.1 \label{probState:assump1}] For the optimization algorithm under consideration, there is a function $b(d_{0},\numIter_{n})$ such that
\[
\mathbb{E}\left[ f_{n}(\bx_{n}) \right] - f_{n}(\bx_{n}^{*}) \leq b(d_{0},\numIter_{n})
\]
with $\numIter_{n}$ the number of samples from $p_{n}$ and $\mathbb{E} \| \bx_{n}(0) - \bx_{n}^{*}\|^{2} \leq d_{0}$, where $\bx_{n}(0)$ is the initial point of the optimization algorithm at time $n$. Finally, $b(d_{0},\numIter_{n})$ is non-decreasing in $d_{0}$.		
\item[A.2 \label{probState:assump2}] Each loss function $\lossFunc(\bx,\bz)$ is differentiable in $\bx$. Each $f_{n}(\bx)$ is strongly convex with parameter $m$, i.e.,
\[
f_{n}(\by) \geq f_{n}(\bx) + \langle \nabla_{\bx} f_{n}(\bx) , \by - \bx \rangle + \frac{1}{2} m \| \by - \bx\|^{2}
\]
\item[A.3 \label{probState:assump3}] $\text{diam}(\xSp) < +\infty$			
\item[A.4 \label{probState:assump4}] We can find initial points $\bx_{1}$ and $\bx_{2}$ that satisfy the \meangap{} criterion with $\epsilon_{1}$ and $\epsilon_{2}$ known, i.e.,
\[
\mathbb{E}\left[ f_{i}(\bx_{i})  \right] - f_{i}(\bx_{i}^{*}) \leq \epsilon_{i} \;\;\;\;\;\;\; i =1,2
\]
\end{description}

\emph{Remarks:} For assumption \ref{probState:assump1}, we assume that the bound $b(d_{0},\numIter_{n})$ depends on the number of samples $\numIter_{n}$ and not the number of iterations. For SGD, generally the number of iterations equals $\numIter_{n}$ as each sample is used to produce a noisy gradient. In addition, we often set $\bx_{n}(0) = \bx_{n-1}$. See Appendix~\ref{bBounds} for a discussion of useful $b(d_{0},\numIter_{n})$ bounds. For assumption \ref{probState:assump4}, we can fix $\numIter_{i}$ and set $\epsilon_{i} = b(\text{diam}(\mathcal{X})^{2},\numIter_{i})$ for $i=1,2$.

Now, we examine the case when the change in minimizers, $\rho$ in~\eqref{probState:slowChangeConstDef} or \eqref{probState:slowChangeDef}, is known. For the analysis of the section, whether \eqref{probState:slowChangeConstDef} or \eqref{probState:slowChangeDef} holds does not affect the analysis. Later we will estimate $\rho$ and in this case whether \eqref{probState:slowChangeConstDef} or \eqref{probState:slowChangeDef} holds matters substantially.

 We want to find a bound $\epsilon_{n}$ on the \meangap{} at time $n$ in terms of $\numIter_{n}$ and $\rho$, i.e., $\epsilon_{n}$ such that ${\mathbb{E}[f_{n}(\bx_{n})] - f_{n}(\bx_{n}^{*}) \leq \epsilon_{n}}$. The idea is to start with the bounds from assumption~\ref{probState:assump4} and proceed inductively using the previous $\epsilon_{n-1}$ and $\rho$ from~\eqref{probState:slowChangeConstDef}. Suppose that $\epsilon_{n-1}$ bounds the \meangap{} at time $n-1$. Using the triangle inequality, strong convexity, and \eqref{probState:slowChangeConstDef} we have
\begin{eqnarray}
\mathbb{E}\|\bx_{n-1} - \bx_{n}^{*}\|^{2} &\leq& \left( \|\bx_{n-1} - \bx_{n-1}^{*}\| + \| \bx_{n}^{*} - \bx_{n-1}^{*}\| \right)^{2}  \nonumber \\
&\leq& \left( \sqrt{\frac{2}{m} \mathbb{E}\left[ f_{n-1}(\bx_{n-1})\right] - f_{n-1}(\bx_{n-1}^{*})} + \| \bx_{n}^{*} - \bx_{n-1}^{*}\| \right)^{2}  \nonumber \\
\label{rhoKnown:basicDoBound} &\leq& \left( \sqrt{\frac{2 \epsilon_{n-1}}{m}} + \rho \right)^{2}
\end{eqnarray}
In comparison, we could use the estimate $\text{diam}^{2}(\xSp)$ to bound $\mathbb{E}\|\bx_{n-1} - \bx_{n}^{*}\|^{2}$ and select $\numIter_{n}$. If the bound in \eqref{rhoKnown:basicDoBound} is much smaller than $\text{diam}(\xSp)^{2}$, then we need significantly fewer samples $\numIter_{n}$ to guarantee a desired \meangap{}. Now, by using the bound $b(d_{0},\numIter_{n})$ from assumption~\ref{probState:assump1}, we can set
\begin{eqnarray}
\label{withRhoKnown:epsNRecursion}
\epsilon_{n} &=& b\left( \left( \sqrt{\frac{2 \epsilon_{n-1}}{m}} + \rho \right)^{2}  , \numIter_{n} \right) \;\;\; \forall n \geq 3 \nonumber
\end{eqnarray}
which yields a sequence of bounds on the \meangap{}. Note that this recursion only relies on the immediate past at time $n-1$ through $\epsilon_{n-1}$. To achieve $\epsilon_{n} \leq \epsilon$ for all $n$, we set 
\[
\numIter_{1} = \min\{ \numIter \geq 1 \;|\; b\left( \text{diam}(\xSp)^{2}, \numIter  \right) \leq \epsilon \}
\]
and $\numIter_{n} = \numIter^{*}$ for $n \geq 2$ with
\begin{equation}
\label{withRhoKnown:KChoice}
\numIter^{*} = \min\left\{ \numIter \geq 1 \;\Bigg|\; b\left(\left( \sqrt{\frac{2 \epsilon}{m}} + \rho \right)^{2} , \numIter  \right) \leq \epsilon \right\}
\end{equation}

\section{Estimating $\rho$}
\label{estRho}

In practice, we do not know $\rho$, so we must construct an estimate $\hat{\rho}_{n}$ using the samples from each distribution $p_{n}$. We introduce two approaches to estimate $\rho$ at one time step, $\|\bx_{i}^{*} - \bx_{i-1}^{*}\|$, and methods to combine these estimates under assumptions~\eqref{probState:slowChangeConstDef} and \eqref{probState:slowChangeDef}. We show that for our estimate $\hat{\rho}_{n}$ and appropriately chosen sequences $\{t_{n}\}$ for all $n$ large enough $\hat{\rho}_{n} + t_{n} \geq \rho$ almost surely. With this property, analysis similar to that in Section~\ref{withRhoKnown} holds. 

\subsection{Allowed Ways to Choose $\numIter_{n}$}
One of the sources of difficulty in estimating $\rho$ is that we will allow $\numIter_{n}$ to be selected in a data dependent way, so $\numIter_{n}$ is itself a random variable. We make the assumption that $\numIter_{n}$ is selected using only information available at the end of time $n-1$. To make this precise we define a filtration of sigma algebras to describe the available information. First, we define the sigma algebra $\mathcal{K}_{0}$ containing all the information on the initial conditions of our algorithm. For example, we may start at a random point $\bx_{0}$ and then
\begin{equation*}
\mathcal{K}_{0} = \sigma(\bx_{0})
\end{equation*}
The sigma algebra $\mathcal{K}_{0}$ may also contain information about $\numIter_{1}$ and $\numIter_{2}$. Next, we define the filtration
\begin{equation}
\label{estRho:H0SigAlg}
\mathcal{K}_{n} = \sigma\left( \{\bz_{n}(\iterIndex)\}_{\iterIndex=1}^{\numIter_{n}}  \right) \vee \mathcal{K}_{n-1} \;\;\;\;\;\; \forall n \geq 1
\end{equation}
where 
\[
\mathcal{F} \vee \mathcal{G} = \sigma\left( \mathcal{F} \cup \mathcal{G} \right)
\]
is the merge operator for sigma algebras. The sigma algebra $\mathcal{K}_{n}$ contains all the information available to us at the end of time $n$. We assume that $\numIter_{n}$ is $\mathcal{K}_{n-1}$-measurable to capture the idea that $\numIter_{n}$ is chosen only using information available at the end of time $n-1$.

\subsection{Estimating One Step Change}
First, we estimate the one step changes $\|\bx_{i}^{*} - \bx_{i-1}^{*}\|$ denoted by $\tilde{\rho}_{i}$. Implicitly, we assume that all one step estimates are capped by $\text{diam}(\xSp)$, since trivially $\| \bx_{n}^{*} - \bx_{n-1}^{*}\| \leq \text{diam}(\xSp)$.

\subsubsection{Direct Estimate}
First, we construct an estimate $\tilde{\rho}_{i}$ of the one step changes $\|\bx_{i}^{*} - \bx_{i-1}^{*}\|$. Using the triangle inequality and variational inequalities from \cite{Dontchev2009} yields
\begin{align}
\| \bx_{i}^{*} - \bx_{i-1}^{*} \| &\leq \| \bx_{i} - \bx_{i-1} \| + \|\bx_{i} - \bx_{i}^{*} \| + \| \bx_{i-1} - \bx_{i-1}^{*} \| \nonumber \\
&\leq \| \bx_{i} - \bx_{i-1} \| + \frac{1}{m} \| \nabla_{\bx} f_{i}(\bx_{i}) \| + \frac{1}{m} \| \nabla_{\bx} f_{i}(\bx_{i-1}) \| \nonumber
\end{align}
We then approximate $\| \nabla_{\bx} f_{i}(\bx_{i}) \| =  \| \mathbb{E}_{\bz_{i} \sim p_{i}} \left[ \nabla_{\bx} \lossFunc(\bx_{i},\bz_{i})  \right] \|$
by
\[\bigg\| \frac{1}{\numIter_{i}} \sum_{\iterIndex=1}^{\numIter_{i}} \nabla_{\bx} \lossFunc(\bx_{i},\bz_{i}(\iterIndex))  \bigg\|
\]
to yield the following estimate that we call the \emph{direct estimate}:
\begin{align}
\tilde{\rho}_{i} &\triangleq \| \bx_{i} - \bx_{i-1} \| + \frac{1}{m} \Bigg\| \frac{1}{\numIter_{i}} \sum_{\iterIndex=1}^{\numIter_{i}} \nabla_{\bx} \lossFunc(\bx_{i},\bz_{i}(\iterIndex))  \Bigg\|  + \frac{1}{m} \Bigg\| \frac{1}{\numIter_{i-1}} \sum_{\iterIndex=1}^{\numIter_{i-1}} \nabla_{\bx} \lossFunc(\bx_{i-1},\bz_{i-1}(\iterIndex))  \Bigg\| \nonumber
\end{align}

\subsubsection{Vector Integral Probability Metric Estimate} 
Given a class of functions $\mathscr{F}$ where each $f \in \mathscr{F}$ maps $\mathcal{Z} \to \mathbb{R}$, an integral probability metric (IPM) \cite{Sriperumbudur12} between two distributions $p$ and $q$ is defined to be
\begin{equation*}
\label{estRho:ipmBound:scalarDef}
\ipm{\mathscr{F}}{p}{q} \triangleq \sup_{f \in \mathscr{F}} \big| \mathbb{E}_{\bz \sim p}[f(\bz)] - \mathbb{E}_{\tilde{\bz} \sim q}[f(\tilde{\bz})] \big|
\end{equation*}
We consider an extension of this idea, which we call a \emph{vector IPM}, in which the class of functions $\mathscr{F}$ maps $\mathcal{Z} \to \xSp$:
\begin{equation}
\label{estRho:ipmBound:vectorDef}
\ipmv{\mathscr{F}}{p}{q} \triangleq \sup_{f \in \mathscr{F}} \| \mathbb{E}_{\bz \sim p}[f(\bz)] - \mathbb{E}_{\tilde{\bz} \sim q}[f(\tilde{\bz})] \|
\end{equation}
Lemma~\ref{estRho:ipmBound:exactBoundVector} shows that a vector IPM can be used to bound the change in minimizer at time $i$ and follows from variational inequalities in \cite{Dontchev2009} and the assumption that $\{\nabla_{\bx} \lossFunc(\bx,\cdot) \::\: \bx \in \xSp \} \subset \mathscr{F}$.
\begin{lem}
\label{estRho:ipmBound:exactBoundVector}
Assume that $\{\nabla_{\bx} \lossFunc(\bx,\cdot) \::\: \bx \in \xSp \} \subset \mathscr{F}$. Then ${\| \bx_{i}^{*} - \bx_{i-1}^{*}\| \leq \frac{1}{m} \ipmv{\mathscr{F}}{p_{i}}{p_{i-1}}}$.
\end{lem}
\begin{proof}
By exploiting variational inequalities from \cite{Dontchev2009}, we can show that
\begin{eqnarray}
\| \bx_{i}^{*} - \bx_{i-1}^{*} \| &\leq& \frac{1}{m} \| \nabla_{\bx}f_{i}(\bx_{i-1}^{*}) - \nabla_{\bx}f_{i-1}(\bx_{i-1}^{*}) \| \nonumber \\
&=& \frac{1}{m} \| \mathbb{E}_{\bz_{i} \sim p_{i}}\left[ \nabla_{\bx}\lossFunc(\bx_{i-1}^{*},\bz_{i}) \right] - \mathbb{E}_{\bz_{i-1} \sim p_{i-1}}\left[ \nabla_{\bx}\lossFunc(\bx_{i-1}^{*},\bz_{i-1}) \right] \| \nonumber
\end{eqnarray}
By assumption $\{\nabla_{\bx} \lossFunc(\bx_{i-1}^{*},\cdot) \::\: \bx \in \xSp \} \subset \mathscr{F}$, so
\begin{eqnarray}
\| \nabla_{\bx}f_{i}(\bx_{i-1}^{*}) - \nabla_{\bx}f_{i-1}(\bx_{i-1}^{*}) \| &=& \| \mathbb{E}_{\bz_{i} \sim p_{i}}\left[ \lossFunc(\bx_{i-1}^{*},\bz_{i}) \right] - \mathbb{E}_{\bz_{i-1} \sim p_{i-1}}\left[ \lossFunc(\bx_{i-1}^{*},\bz_{i-1}) \right] \| \nonumber \\
&\leq& \sup_{f \in \mathscr{F}} \|  \mathbb{E}_{\bz_{i} \sim p_{i}}\left[ f(\bz_{i}) \right] - \mathbb{E}_{\bz_{i-1} \sim p_{i-1}}\left[  f(\bz_{i-1}) \right] \| \nonumber \\
&=& \ipmv{\mathscr{F}}{p_{i}}{p_{i-1}} \nonumber
\end{eqnarray}
\end{proof}

We cannot compute this vector IPM, since we do not know the distributions $p_{i}$ and $p_{i-1}$. Instead, we plug in the empiricals $\hat{p}_{i}$ and $\hat{p}_{i-1}$ to yield the estimate $\frac{1}{m} \ipmv{\mathscr{F}}{\hat{p}_{i}}{\hat{p}_{i-1}}$. This estimate is biased upward, which ensures that $\| \bx_{i}^{*} - \bx_{i-1}^{*}\| \leq \mathbb{E}\left[ \frac{1}{m} \ipmv{\mathscr{F}}{\hat{p}_{i}}{\hat{p}_{i-1}} \right]$. 

Our estimate is still not in a closed form since there is a supremum over $\mathcal{F}$ in the computation of $\ipmv{\mathscr{F}}{\hat{p}_{i}}{\hat{p}_{i-1}}$. For the class of functions 
\begin{equation}
\label{estRho:ipmBound:vectorRClass}
\mathscr{F} = \left\{ f \;\big|\; \|f(\bz) - f(\tilde{\bz})\| \leq r(\bz,\tilde{\bz}) \right\}.
\end{equation}
we can compute an upper bound $\Gamma_{i}$ on $\ipmv{\mathscr{F}}{\hat{p}_{i}}{\hat{p}_{i-1}}$ yielding a computable estimate $\tilde{\rho}_{i} = \frac{1}{m} \Gamma_{i}$. Set ${\tilde{\bz}_{i}(\iterIndex) = \bz_{i}(\iterIndex)}$ if ${1 \leq \iterIndex \leq \numIter_{i}}$ and ${\tilde{\bz}_{i}(\iterIndex) = \bz_{i-1}(\iterIndex)}$ if ${\numIter_{i}+1 \leq \iterIndex \leq \numIter_{i} + \numIter_{i-1}}$. From \eqref{estRho:ipmBound:vectorDef}, we have
\[
\ipmv{\mathscr{F}}{\hat{p}_{i}}{\hat{p}_{i-1}} = \sup_{f \in \mathscr{F}} \Bigg\| \frac{1}{\numIter_{i}} \sum_{\iterIndex=1}^{\numIter_{i}} f(\tilde{\bz}_{i}(\iterIndex)) - \frac{1}{\numIter_{i-1}} \sum_{\iterIndex=1}^{\numIter_{i-1}} f(\tilde{\bz}_{i}(\numIter_{i} + \iterIndex)) \Bigg\|
\] 
We can relax this supremum by maximizing over the function value $f(\tilde{\bz}_{i}(\iterIndex))$ denoted by $\alpha_{\iterIndex}$ in the following non-convex quadratically constrained quadratic program (QCQP):
\begin{equation*}
\arraycolsep=1.4pt\def\arraystretch{1.5}
\begin{array}{ll@{}ll}
\text{maximize}  & \displaystyle \Bigg\| \frac{1}{\numIter_{i}}\sum_{\iterIndex=1}^{\numIter_{i}} \alpha_{\iterIndex} - \frac{1}{\numIter_{i-1}}\sum_{\iterIndex=1}^{\numIter_{i-1}} \alpha_{\numIter_{i} + \iterIndex} \Bigg\| &\\
\vspace{3mm}
\text{subject to}& \displaystyle \| \alpha_{\iterIndex} - \alpha_{j} \| \leq r(\tilde{\bz}_{i}(\iterIndex),\tilde{\bz}_{i}(j)) \;\;\; \forall \iterIndex < j
\end{array}
\end{equation*}
The constraints are imposed to ensure that the function values $\alpha_{\iterIndex}$ can correspond to a function in $\mathscr{F}$ from \eqref{estRho:ipmBound:vectorRClass}. The value of this QCQP exactly  may not equal the vector IPM but at least provides an upper bound. Finally, we note that this QCQP can be converted to its dual form to yield an SDP, which is often easier to solve.

\subsubsection{Comparison of Estimates}
The direct estimate is easier to compute but may be loose if $\|\bx_{n}-\bx_{n}^{*}\|$ is large. If $\|\bx_{n}-\bx_{n}^{*}\|$ is large, then the vector IPM approach is in general tighter. However, the vector IPM is more difficult to compute due to need to solve a QCQP or SDP and check the inclusion conditions in Lemma~\ref{estRho:ipmBound:exactBoundVector}. Also, the number of constraints in the QCQP or SDP grows quadratically in the number of samples.

\subsection{Combining One Step Estimates For Constant Change}
Assuming that $\| \bx_{i}^{*} - \bx_{i-1}^{*} \| = \rho$ from \eqref{probState:slowChangeConstDef}, we average the one step estimates $\tilde{\rho}_{i}$ to yield a better estimate 
\[
\hat{\rho}_{n} = \frac{1}{n-1} \sum_{i=2}^{n} \tilde{\rho}_{i}
\]
of $\rho$ at each time $n$ under \eqref{probState:slowChangeConstDef}. To analyze the behavior of our  combined estimates, we use sub-Gaussian concentration inequalities detailed in Appendix~\ref{usefulConcIneq}. Lemma~\ref{subgauss:subgaussDepLem} is of particular importance to our analysis.

\subsubsection{Direct Estimate}
The difficulty in analyzing the direct estimate comes because in approximating $\frac{1}{m}\| \nabla f_{i}(\bx_{i}) \|$
by
\[
\frac{1}{m} \Bigg\| \frac{1}{\numIter_{i}} \sum_{\iterIndex=1}^{\numIter_{i}} \nabla_{\bx} \lossFunc(\bx_{i},\bz_{i}(\iterIndex))  \Bigg\|
\]
$\bx_{i}$ is dependent on all the samples $\{\bz_{i}(\iterIndex)\}_{\iterIndex=1}^{\numIter_{i}}$. To illustrate the problem further, consider drawing two independent copies $\{\bz_{i}(\iterIndex)\}_{\iterIndex=1}^{\numIter_{i}} \overset{\text{iid}}{\sim} p_{i}$ and $\{\tilde{\bz}_{i}(\iterIndex)\}_{\iterIndex=1}^{\numIter_{i}} \overset{\text{iid}}{\sim} p_{i}$ of the samples. Suppose that we use the second copy $\{\tilde{\bz}_{i}(\iterIndex)\}_{\iterIndex=1}^{\numIter_{i}}$ to compute $\bx_{i}$ using our optimization algorithm of choice starting from $\bx_{i-1}$. Then we approximate $\frac{1}{m}\| \nabla f_{i}(\bx_{i}) \|$ by
\[
\frac{1}{m} \Bigg\| \frac{1}{\numIter_{i}} \sum_{\iterIndex=1}^{\numIter_{i}} \nabla_{\bx} \lossFunc(\bx_{i},\bz_{i}(\iterIndex))  \Bigg\|
\]
Now, since $\bx_{i}$ is independent of $\{\bz_{i}(\iterIndex)\}_{\iterIndex=1}^{\numIter_{i}}$ the quantity
\[
\frac{1}{m} \Bigg\| \frac{1}{\numIter_{i}} \sum_{\iterIndex=1}^{\numIter_{i}} \nabla_{\bx} \lossFunc(\bx_{i},\bz_{i}(\iterIndex))  \Bigg\|
\]
is the norm of an average of independent random variables conditioned on $\bx_{i}$. This allows us to apply standard concentration inequalities for norms of random variables as in \cite{Very2012}. In this section, we argue that re-using the samples $\{\bz_{i}(\iterIndex)\}_{\iterIndex=1}^{\numIter_{i}}$ to compute $\bx_{i}$ is not too far from using a second independent draw $\{\tilde{\bz}_{i}(\iterIndex)\}_{\iterIndex=1}^{\numIter_{i}}$.

For analysis, we need the following additional assumptions:
\begin{description}
\item[B.1 \label{probState:assumpB1}] The loss function $\lossFunc(\bx,\bz)$ has uniform Lipschitz continuous gradients in $\bx$ with modulus $\lossLG$, i.e. 
\[
\| \nabla_{\bx} \lossFunc(\bx,\bz) - \nabla_{\bx} \lossFunc(\tilde{\bx},\bz) \| \leq \lossLG \| \bx - \tilde{\bx} \| \;\;\; \forall \bz \in \mathcal{Z}
\]
\item[B.2 \label{probState:assumpB2}] Assuming $\xSp$ is $d$-dimensional, each component $j$ of the gradient error $\nabla_{\bx} \lossFunc(\bx,\bz_{n}) - f_{n}(\bx)$ satisfies
\[
\mathbb{E}\left[ \exp\left\{ s \left( \nabla_{\bx} \lossFunc(\bx,\bz_{n}) - \nabla f_{n}(\bx) \right)_{j}  \right\} \;\bigg|\; \bx  \right] \leq \exp\left\{\frac{1}{2} \frac{C_{g}}{d^2} s^{2} \right\}
\]
\end{description}
Assumption~\ref{probState:assumpB1} is reasonable if the space $\mathcal{Z}$ containing $\bz$ is compact. Although in practice, the distribution of gradient error could depend on $\bx$, we assume that the bound $C_{g}$ does not depend on $\bx$. We can view this as a pessimistic assumption corresponding to choosing the worst case bound as a function of $\bx$ and the resulting $C_{g}$. This is a common assumption for in high probability analysis of optimization algorithms as in \cite{Nemirovski2009} for example.

To proceed, we first define two other useful estimates for $\rho$. As discussed before, suppose that we make a second independent draw of samples $\{\tilde{\bz_{i}}(\iterIndex)\}_{\iterIndex=1}^{\numIter_{i}}$ from $p_{i}$. We use these samples to compute $\tilde{\bx}_{i}$ in the same manner as $\bx_{i}$ starting from $\bx_{i-1}$ except with $\{\tilde{\bz_{i}}(\iterIndex)\}_{\iterIndex=1}^{\numIter_{i}}$ used in place of $\{\bz_{i}(\iterIndex)\}_{\iterIndex=1}^{\numIter_{i}}$. Then define
\[
\tilde{\rho}_{i}^{(2)} \triangleq \| \tilde{\bx}_{i} - \tilde{\bx}_{i-1} \| + \frac{1}{m} \Bigg\| \frac{1}{\numIter_{i}} \sum_{\iterIndex=1}^{\numIter_{i}} \nabla_{\bx} \lossFunc(\tilde{\bx}_{i},\bz_{i}(\iterIndex))  \Bigg\| + \frac{1}{m} \Bigg\| \frac{1}{\numIter_{i-1}} \sum_{\iterIndex=1}^{\numIter_{i-1}} \nabla_{\bx} \lossFunc(\tilde{\bx}_{i-1},\bz_{i-1}(\iterIndex))  \Bigg\|
\]
This is the same form as the direct estimate with $\tilde{\bx}_{i}$ in place of $\bx_{i}$. Next, define
\[
\tilde{\rho}_{i}^{(3)} \triangleq \| \tilde{\bx}_{i} - \tilde{\bx}_{i-1} \| + \frac{1}{m} \| \nabla f_{i}(\bx_{i})\| + \frac{1}{m} \| \nabla f_{i-1}(\bx_{i-1})\|
\]
This is in fact the bound that inspired the direct estimate. We also define the averaged estimates
\[
\hat{\rho}_{n}^{(2)} \triangleq \frac{1}{n-1} \sum_{i=2}^{n} \tilde{\rho}_{i}^{(2)}
\]
and
\[
\hat{\rho}_{n}^{(3)} \triangleq \frac{1}{n-1} \sum_{i=2}^{n} \tilde{\rho}_{i}^{(3)}
\]
We know that $\hat{\rho}_{n}^{(3)} \geq \rho$. Thus, if we can control the gap between the pair $\hat{\rho}_{n}$ and $\hat{\rho}_{n}^{(2)}$ and the pair $\hat{\rho}_{n}^{(2)}$ and $\hat{\rho}_{n}^{(3)}$, then we can ensure that $\hat{\rho}_{n}$ plus an appropriate constant upper bounds $\rho$ for all $n$ large enough as desired.

First, we show that $\hat{\rho}_{n}^{(2)}$ upper bounds $\rho$ eventually.
\begin{lem}
\label{estRho:directEst:combineOneLemma}
Suppose that the following conditions hold:
\begin{enumerate}
\item \ref{probState:assumpB1}-\ref{probState:assumpB2} hold
\item The sequence $\{t_{n}\}$ satisfies
\[
\sum_{n=2}^{\infty} \exp\left\{-\frac{(n-1)m^2 t_{n}^2}{72 C_{g}}\right\} < \infty
\]
\end{enumerate}
Then for all $n$ large enough it holds that $\hat{\rho}_{n}^{(2)} + \hat{C}_{n}^{(2)} + t_{n} \geq   \rho$ almost surely with
\[
\hat{C}_{n}^{(2)} \triangleq \frac{1}{d m(n-1)} \left( \sqrt{\frac{C_{g}}{\numIter_{1}}} + 2\sum_{i=1}^{n} \sqrt{\frac{C_{g}}{\numIter_{i}}} + \sqrt{\frac{C_{g}}{\numIter_{n}}} \right)
\]
\end{lem}
\begin{proof}
First, we have by the triangle equality and reverse triangle inequality
\begin{align}
m&|\tilde{\rho}_{i}^{(2)} - \tilde{\rho}_{i}^{(3)}| \nonumber \\
&= \Bigg| \left( \Bigg\| \frac{1}{\numIter_{i}} \sum_{\iterIndex=1}^{\numIter_{i}} \nabla_{\bx} \lossFunc(\tilde{\bx}_{i},\bz_{i}(\iterIndex)) \Bigg\| - \|\nabla_{\bx} f_{i}(\tilde{\bx}_{i})\| \right) + \left( \Bigg\| \frac{1}{\numIter_{i-1}} \sum_{\iterIndex=1}^{\numIter_{i-1}} \nabla_{\bx} \lossFunc(\tilde{\bx}_{i-1},\bz_{i-1}(\iterIndex)) \Bigg\| - \|\nabla_{\bx} f_{i-1}(\tilde{\bx}_{i-1})\| \right) \Bigg| \nonumber \\
&\leq \Bigg| \Bigg\| \frac{1}{\numIter_{i}} \sum_{\iterIndex=1}^{\numIter_{i}} \nabla_{\bx} \lossFunc(\tilde{\bx}_{i},\bz_{i}(\iterIndex)) \Bigg\| - \|\nabla_{\bx} f_{i}(\tilde{\bx}_{i})\| \Bigg| + \Bigg| \Bigg\| \frac{1}{\numIter_{i-1}} \sum_{\iterIndex=1}^{\numIter_{i-1}} \nabla_{\bx} \lossFunc(\tilde{\bx}_{i-1},\bz_{i-1}(\iterIndex)) \Bigg\| - \|\nabla_{\bx} f_{i-1}(\tilde{\bx}_{i-1})\| \Bigg| \nonumber \\
&\leq \Bigg\| \frac{1}{\numIter_{i}} \sum_{\iterIndex=1}^{\numIter_{i}} \left( \nabla_{\bx} \lossFunc(\tilde{\bx}_{i},\bz_{i}(\iterIndex)) - \nabla_{\bx} f_{i}(\tilde{\bx}_{i}) \right) \Bigg\| + \Bigg\| \frac{1}{\numIter_{i-1}} \sum_{\iterIndex=1}^{\numIter_{i-1}} \left( \nabla_{\bx} \lossFunc(\tilde{\bx}_{i-1},\bz_{i-1}(\iterIndex)) - \nabla_{\bx} f_{i-1}(\tilde{\bx}_{i-1}) \right) \Bigg\| \nonumber
\end{align}
Then by the triangle inequality, we have
\begin{eqnarray}
|\hat{\rho}_{n}^{(2)} - \hat{\rho}_{n}^{(3)}| &\leq& \frac{1}{m(n-1)} \sum_{i=2}^{n} \left( \Bigg\| \frac{1}{\numIter_{i}} \sum_{\iterIndex=1}^{\numIter_{i}} \left( \nabla_{\bx} \lossFunc(\tilde{\bx}_{i},\bz_{i}(\iterIndex)) - \nabla_{\bx}f_{i}(\tilde{\bx}_{i}) \right)  \Bigg\| \right. \nonumber \\
&& \hspace{15mm} \left. + \Bigg\| \frac{1}{\numIter_{i-1}} \sum_{\iterIndex=1}^{\numIter_{i-1}} \left( \nabla_{\bx} \lossFunc(\tilde{\bx}_{i-1},\bz_{i-1}(\iterIndex)) - \nabla_{\bx}f_{i-1}(\tilde{\bx}_{i-1}) \right)  \Bigg\| \right) \nonumber \\
&\leq& \frac{1}{m(n-1)} \left( \Bigg\| \frac{1}{\numIter_{1}} \sum_{\iterIndex=1}^{\numIter_{1}} \left( \nabla_{\bx} \lossFunc(\tilde{\bx}_{1},\bz_{1}(\iterIndex)) - \nabla_{\bx}f_{1}(\tilde{\bx}_{1}) \right)  \Bigg\| \right. \nonumber \\
&& \hspace{15mm}  + 2 \sum_{i=2}^{n-1} \Bigg\| \frac{1}{\numIter_{i}} \sum_{\iterIndex=1}^{\numIter_{i}} \left( \nabla_{\bx} \lossFunc(\tilde{\bx}_{i},\bz_{i}(\iterIndex)) - \nabla_{\bx}f_{i}(\tilde{\bx}_{i}) \right) \Bigg\| \nonumber \\
\label{estRho:directEst:combineOneProofBound}
&& \hspace{15mm} \left. + \Bigg\| \frac{1}{\numIter_{n}} \sum_{\iterIndex=1}^{\numIter_{n}} \left( \nabla_{\bx} \lossFunc(\tilde{\bx}_{n},\bz_{n}(\iterIndex)) - \nabla_{\bx}f_{n}(\tilde{\bx}_{n}) \right) \Bigg\| \right)
\end{eqnarray}
We will analyze the behavior of this bound on $|\hat{\rho}_{i}^{(2)} - \hat{\rho}_{i}^{(3)}|$ using Lemma~\ref{subgauss:subgaussDepLem} in Appendix~\ref{usefulConcIneq}. Define the filtration
\begin{equation}
\label{estRho:directEst:combineOneProofSigAlg}
\mathcal{F}_{i} = \sigma\left( \bigcup_{j=1}^{i} \{\bz_{j}(\iterIndex)\}_{\iterIndex=1}^{\numIter_{j}} \cup \bigcup_{j=1}^{i+1} \{\tilde{\bz}_{j}(\iterIndex)\}_{\iterIndex=1}^{\numIter_{j}}  \right) \vee \mathcal{K}_{0}  \;\;\; i=0,\ldots,n
\end{equation}
with $\mathcal{K}_{0}$ from \eqref{estRho:H0SigAlg}. Note that $\mathcal{K}_{i-1} \subset \mathcal{F}_{i-1}$, so $\numIter_{i}$ is $\mathcal{F}_{i-1}$-measurable. In addition, $\tilde{\bx}_{i}$ but not $\bx_{i}$ is $\mathcal{F}_{i-1}$-measurable. Define the random variables
\[
V_{i} = \Bigg\| \frac{1}{\numIter_{i}} \sum_{\iterIndex=1}^{\numIter_{i}} \left( \nabla_{\bx}\lossFunc(\tilde{\bx}_{i},\bz_{i}(\iterIndex)) - \nabla_{\bx}f_{i}(\tilde{\bx}_{i}) \right) \Bigg\| - \mathbb{E}\left[ \Bigg\| \frac{1}{\numIter_{i}} \sum_{\iterIndex=1}^{\numIter_{i}} \left( \nabla_{\bx}\lossFunc(\tilde{\bx}_{i},\bz_{i}(\iterIndex)) - \nabla_{\bx}f_{i}(\tilde{\bx}_{i}) \right) \Bigg\| \;\Bigg|\; \mathcal{F}_{i-1} \right] \;\;\; i=1,\ldots,n
\]
Clearly, $V_{i}$ is $\mathcal{F}_{i}$-measurable, since $V_{i}$ is a function of $\tilde{\bx}_{i}$, $\numIter_{i}$, and $\{\bz_{i}(\iterIndex)\}_{\iterIndex=1}^{\numIter_{i}}$ all of which are $\mathcal{F}_{i}$-measurable. Conditioned on $\mathcal{F}_{i-1}$, the sum
\begin{equation}
\label{estRho:directEst:combineOneLemma:sumTau}
\frac{1}{\numIter_{i}} \sum_{\iterIndex=1}^{\numIter_{i}} \left( \nabla_{\bx}\lossFunc(\tilde{\bx}_{i},\bz_{i}(\iterIndex)) - \nabla_{\bx}f_{i}(\tilde{\bx}_{i}) \right)
\end{equation}
is a sum of iid random variables. We now work with the conditional measure $\mathbb{P}\{\cdot\;|\; \mathcal{F}_{i-1}\}$ to compute sub-Gaussian norms of \eqref{estRho:directEst:combineOneLemma:sumTau} define in \eqref{concDep:tauDef} and \eqref{concDep:tauBDef} of Appendix~\ref{usefulConcIneq}. By assumption~\ref{probState:assumpB2}, we have
\[
\tau^{2}\left( \left( \nabla_{\bx}\lossFunc(\tilde{\bx}_{i},\bz_{i}(\iterIndex)) - \nabla_{\bx} f_{i}(\tilde{\bx}_{i}) \right)_{j} \right) \leq \frac{C_{g}}{d^{2}}
\]
Therefore, applying Lemma~\ref{subgauss:aveTauLemma} yields
\[
B\left( \sum_{\iterIndex=1}^{\numIter_{i}} \left( \nabla_{\bx}\lossFunc(\tilde{\bx}_{i},\bz_{i}(\iterIndex)) - \nabla_{\bx} f_{i}(\tilde{\bx}_{i}) \right)  \right) \leq \sqrt{\frac{C_{g}}{\numIter_{i}}}
\]
due to the independence conditioned on $\mathcal{F}_{i-1}$.
By applying Lemma~\ref{subgauss:buldLemma} from \cite{Buld2010} to the conditional distribution $\mathbb{P}\{ \cdot | \mathcal{F}_{i-1} \}$, we have
\begin{eqnarray}
\mathbb{P}\left\{ \Bigg\| \frac{1}{\numIter_{i}} \sum_{\iterIndex=1}^{\numIter_{i}} \left( \nabla_{\bx}\lossFunc(\tilde{\bx}_{i},\bz_{i}(\iterIndex)) - \nabla_{\bx}f_{i}(\tilde{\bx}_{i}) \right) \Bigg\| > t \;\Bigg|\; \mathcal{F}_{i-1} \right\} &\leq& 2 \exp\left\{ - \frac{t^2}{2 (\sqrt{C_{g}/\numIter_{i}})^{2}} \right\} \nonumber \\
&=& 2 \exp\left\{ - \frac{\numIter_{i} t^2}{2 C_{g}} \right\} \nonumber
\end{eqnarray}
Since
\[
\mathbb{E}\left[ \Bigg\| \frac{1}{\numIter_{i}} \sum_{\iterIndex=1}^{\numIter_{i}} \left( \nabla_{\bx}\lossFunc(\tilde{\bx}_{i},\bz_{i}(\iterIndex)) - \nabla_{\bx}f_{i}(\tilde{\bx}_{i}) \right) \Bigg\|  \;\Bigg|\; \mathcal{F}_{i-1} \right] \geq 0,
\]
we have
\begin{align}
\mathbb{P}&\left\{ V_{i} > t \;\Bigg|\; \mathcal{F}_{i-1} \right\} \nonumber \\
&\;\;\;\;\;\; = \mathbb{P}\left\{ \Bigg\| \frac{1}{\numIter_{i}} \sum_{\iterIndex=1}^{\numIter_{i}} \left( \nabla_{\bx}\lossFunc(\tilde{\bx}_{i},\bz_{i}(\iterIndex)) - \nabla_{\bx}f_{i}(\tilde{\bx}_{i}) \right) \Bigg\| \right. \nonumber \\
&\;\;\;\;\;\;\;\;\;\;\;\;\;\;\;\;\;\;\;\;\;\;\;\; \left. - \mathbb{E}\left[ \Bigg\| \frac{1}{\numIter_{i}} \sum_{\iterIndex=1}^{\numIter_{i}} \left( \nabla_{\bx}\lossFunc(\tilde{\bx}_{i},\bz_{i}(\iterIndex)) - \nabla_{\bx}f_{i}(\tilde{\bx}_{i}) \right) \Bigg\|  \;\Bigg|\; \mathcal{F}_{i-1} \right] > t \;\Bigg|\; \mathcal{F}_{i-1} \right\} \nonumber \\
&\;\;\;\;\;\; \leq \mathbb{P}\left\{ \Bigg\| \frac{1}{\numIter_{i}} \sum_{\iterIndex=1}^{\numIter_{i}} \left( \nabla_{\bx}\lossFunc(\tilde{\bx}_{i},\bz_{i}(\iterIndex)) - \nabla_{\bx}f_{i}(\tilde{\bx}_{i}) \right) \Bigg\|  > t \;\Bigg|\; \mathcal{F}_{i-1} \right\} \nonumber \\
&\;\;\;\;\;\; \leq  2 \exp\left\{ - \frac{\numIter_{i} t^2}{2 C_{g}} \right\} \nonumber \\
&\;\;\;\;\;\; \leq  2 \exp\left\{ - \frac{ t^2}{2 C_{g}} \right\} \nonumber
\end{align}
Since $\mathbb{E}[V_{i} \;|\; \mathcal{F}_{i-1}] = 0$, we can apply Lemma~\ref{subgauss:subgaussConsts} with $c=1/(2C_{g})$ to yield
\[
\mathbb{E}\left[e^{s V_{i}}\;\big|\; \mathcal{F}_{i-1}  \right] \leq \exp\left\{ \frac{1}{2} \left( 18 C_{g} \right) s^{2}  \right\}
\]
This shows that the collection of random variables $\{V_{i}\}_{i=1}^{n}$ and the filtration $\{\mathcal{F}_{i}\}_{i=0}^{n}$ satisfies the conditions of Lemma~\ref{subgauss:subgaussDepLem}. Before applying Lemma~\ref{subgauss:subgaussDepLem}, we bound the conditional expectations
\[
\mathbb{E}\left[ \Bigg\| \frac{1}{\numIter_{i}} \sum_{\iterIndex=1}^{\numIter_{i}} \left( \nabla_{\bx}\lossFunc(\tilde{\bx}_{i},\bz_{i}(\iterIndex)) - \nabla_{\bx}f_{i}(\tilde{\bx}_{i}) \right) \Bigg\|^{2}  \;\Bigg|\; \mathcal{F}_{i-1} \right]
\]
By a straightforward calculation conditioned on $\mathcal{F}_{i-1}$, we have
\begin{align}
&\mathbb{E}\left[ \Bigg\| \frac{1}{\numIter_{i}} \sum_{\iterIndex=1}^{\numIter_{i}} \left( \nabla_{\bx}\lossFunc(\tilde{\bx}_{i},\bz_{i}(\iterIndex)) - \nabla_{\bx}f_{i}(\tilde{\bx}_{i}) \right) \Bigg\|^{2}  \;\Bigg|\; \mathcal{F}_{i-1} \right] \nonumber \\
&\hspace{10mm} = \frac{1}{\numIter_{i}^{2}} \sum_{\iterIndex=1}^{\numIter_{i}} \sum_{j=1}^{\numIter_{i}} \mathbb{E}\left[ \inprod{\nabla_{\bx}\lossFunc(\tilde{\bx}_{i},\bz_{i}(\iterIndex)) - \nabla_{\bx}f(\tilde{\bx}_{i})}{\nabla_{\bx}\lossFunc(\tilde{\bx}_{i},\bz_{i}(j)) - \nabla_{\bx}f(\tilde{\bx}_{i})} \;|\; \mathcal{F}_{i-1} \right] \nonumber \\
&\hspace{10mm} = \frac{1}{\numIter_{i}^{2}} \sum_{\iterIndex=1}^{\numIter_{i}} \mathbb{E}\left[ \| \nabla_{\bx}\lossFunc(\tilde{\bx}_{i},\bz_{i}(\iterIndex)) - \nabla_{\bx}f(\tilde{\bx}_{i})\|^{2} \;|\; \mathcal{F}_{i-1} \right] \nonumber \\
&\hspace{10mm} \overset{(a)}{=} \frac{1}{\numIter_{i}^{2}} \sum_{\iterIndex=1}^{\numIter_{i}} \sum_{q=1}^{d} \mathbb{E}\left[ ( \nabla_{\bx}\lossFunc(\tilde{\bx}_{i},\bz_{i}(\iterIndex)) - \nabla_{\bx}f(\tilde{\bx}_{i}))_{q}^{2} \;|\; \mathcal{F}_{i-1} \right] \nonumber \\
&\hspace{10mm} \overset{(b)}{\leq} \frac{1}{\numIter_{i}^{2}} \sum_{\iterIndex=1}^{\numIter_{i}} d \frac{C_{g}}{d^{2}} \nonumber \\
&\hspace{10mm} \leq \frac{C_{g}}{d \numIter_{i}} \nonumber
\end{align}
where (a) is a decomposition into each component of the vector and (b) follows since a centered sub-Gaussian random variable with parameter $C_{g}/d^{2}$ satisfies
\[
\mathbb{E}\left[ ( \nabla_{\bx}\lossFunc(\tilde{\bx}_{i},\bz_{i}(\iterIndex)) - \nabla_{\bx}f(\tilde{\bx}_{i}))_{q}^{2} \;|\; \mathcal{F}_{i-1} \right] \leq \frac{C_{g}}{d^{2}}
\]
Then by Jensen's inequality
\[
\mathbb{E}\left[ \Bigg\| \frac{1}{\numIter_{i}} \sum_{\iterIndex=1}^{\numIter_{i}} \left( \nabla_{\bx}\lossFunc(\tilde{\bx}_{i},\bz_{i}(\iterIndex)) - \nabla_{\bx}f_{i}(\tilde{\bx}_{i}) \right) \Bigg\| \;\Bigg|\; \mathcal{F}_{i-1} \right] \leq \sqrt{\frac{C_{g}}{d \numIter_{i}}}
\]

Define the constants
\begin{eqnarray}
a_{1} &=& a_{n} = \frac{1}{m(n-1)} \nonumber \\
a_{2} &=& \cdots = a_{n-1} = \frac{2}{m(n-1)} \nonumber
\end{eqnarray}
resulting in
\[
\|\bm{a}\|_{2}^{2} = \frac{2}{m^2(n-1)}
\]
Using the bound in \eqref{estRho:directEst:combineOneProofBound} and Lemma~\ref{subgauss:subgaussDepLem} from Appendix~\ref{usefulConcIneq} with this choice of $\bm{a}$, it holds that
\begin{align}
\mathbb{P}&\left\{ | \hat{\rho}_{n}^{(2)} - \hat{\rho}_{n}^{(3)} | > \sum_{i=1}^{n} a_{i} \sqrt{\frac{C_{g}}{d \numIter_{i}}} + t  \right\} \nonumber \\
&\leq \mathbb{P}\left\{ \sum_{i=1}^{n} a_{i} \Bigg\| \frac{1}{\numIter_{i}} \sum_{\iterIndex=1}^{\numIter_{i}} \left( \nabla_{\bx}\lossFunc(\tilde{\bx}_{i},\bz_{i}(\iterIndex)) - \nabla_{\bx}f_{i}(\tilde{\bx}_{i}) \right) \Bigg\| \right. \nonumber \\
&\;\;\;\;\;\;\;\;\;\;\;\;\;\;\; \left. > \sum_{i=1}^{n} a_{i} \mathbb{E}\left[ \Bigg\| \frac{1}{\numIter_{i}} \sum_{\iterIndex=1}^{\numIter_{i}} \left( \nabla_{\bx}\lossFunc(\tilde{\bx}_{i},\bz_{i}(\iterIndex)) - \nabla_{\bx}f_{i}(\tilde{\bx}_{i}) \right) \Bigg\| \;\Bigg|\; \mathcal{F}_{i-1} \right] +t \right\} \nonumber \\
&= \mathbb{P}\left\{ \sum_{i=1}^{n} a_{i} V_{i} > t \right\} \nonumber \\
&\leq \exp\left\{ - \frac{m^2(n-1)t^2}{72 C_{g}} \right\} \nonumber
\end{align}
Combining this bound with $\hat{\rho}_{n}^{(3)} \geq \rho$ yields
\begin{eqnarray}
\sum_{n=2}^{\infty} \mathbb{P}\left\{ \hat{\rho}_{n}^{(2)} < \rho - \sum_{i=1}^{n} a_{i} \sqrt{\frac{C_{g}}{d \numIter_{i}}} -  t_{n}\right\} &\leq& \sum_{n=2}^{\infty} \mathbb{P}\left\{ \hat{\rho}_{n}^{(2)} < \hat{\rho}_{n}^{(3)} - \sum_{i=1}^{n} a_{i} \sqrt{\frac{C_{g}}{d \numIter_{i}}} -  t_{n}\right\} \nonumber \\
&\leq& \sum_{n=2}^{\infty} \mathbb{P}\left\{ |\hat{\rho}_{n}^{(2)} - \hat{\rho}_{n}^{(3)} | > \sum_{i=1}^{n} a_{i} \sqrt{\frac{C_{g}}{d \numIter_{i}}} +  t_{n} \right\} \nonumber \\ 
&\leq& \sum_{n=2}^{\infty} \exp\left\{ - \frac{m^2(n-1)t_{n}^2}{72 C_{g}} \right\} < \infty \nonumber
\end{eqnarray}
The result follows from the Borel-Cantelli lemma. Note that as claimed
\[
\hat{C}_{n}^{(2)} = \frac{1}{d m(n-1)} \left(  \sqrt{\frac{C_{g}}{\numIter_{1}}} + 2 \sum_{i=2}^{n-1} \sqrt{\frac{C_{g}}{\numIter_{i}}} +  \sqrt{\frac{C_{g}}{\numIter_{n}}} \right)
\]
\end{proof}

Next, we show that $\hat{\rho}_{n}$ upper bounds $\hat{\rho}_{n}^{(2)}$ eventually with a general assumption on the optimization algorithm. When the conditions of Lemmas~\ref{estRho:directEst:combineOneLemma} and \ref{estRho:directEst:combineTwoLemma} are satisfied, it holds that $\hat{\rho}_{n}$ plus a constant upper bounds $\rho$.

\begin{lem}
\label{estRho:directEst:combineTwoLemma}
Suppose the following conditions hold:
\begin{enumerate}
\item B.1-B.2 hold

\item There exist bounds
\[
\mathbb{E}\left[ \| \bx_{i} - \tilde{\bx}_{i} \| \;\big|\; \mathcal{F}_{i-1} \right] \leq C(\numIter_{i}) \;\;\;\; i=1,\ldots,n
\]
\item The sequence $\{t_{n}\}$ satisfies
\[
\sum_{n=2}^{\infty} \exp\left\{ - \frac{(n-1)^{2}t_{n}^2}{2 n \left(1 + \frac{\lossLG}{m}  \right)^{2}\text{diam}^{2}(\xSp)} \right\} < +\infty
\]
\end{enumerate}
Then for all $n$ large enough it holds that $\hat{\rho}_{n} + \hat{C}_{n} + t_{n} \geq \hat{\rho}_{n}^{(2)}$ almost surely with
\[
\hat{C}_{n} \triangleq \frac{\left(1 + \frac{\lossLG}{m} \right)}{n-1} \left( C(\numIter_{1}) + 2 \sum_{i=2}^{n-1} C(\numIter_{i}) + C(\numIter_{n}) \right)
\]
\end{lem}
\begin{proof}
We have by the triangle inequality, reverse triangle inequality, and the Lipschitz continuity of $\nabla_{\bx}\lossFunc(\bx,\bz)$ in $\bx$ from assumption~\ref{probState:assumpB1}
\begin{eqnarray}
|\tilde{\rho}_{i} - \tilde{\rho}_{i}^{(2)}| &\leq&\big| \| \bx_{i} - \bx_{i-1} \| - \| \tilde{\bx}_{i} - \tilde{\bx}_{i-1} \|  \big| \nonumber \\
&& \;\;\;\;\;\; + \Bigg| \frac{1}{m} \Bigg\| \frac{1}{\numIter_{i}} \sum_{\iterIndex=1}^{\numIter_{i}} \nabla_{\bx} \lossFunc(\bx_{i},\bz_{i}(\iterIndex))  \Bigg\| - \frac{1}{m} \Bigg\| \frac{1}{\numIter_{i}} \sum_{\iterIndex=1}^{\numIter_{i}} \nabla_{\bx} \lossFunc(\tilde{\bx}_{i},\bz_{i}(\iterIndex))  \Bigg\| \Bigg| \nonumber \\
&& \;\;\;\;\;\; + \Bigg| \frac{1}{m} \Bigg\| \frac{1}{\numIter_{i-1}} \sum_{\iterIndex=1}^{\numIter_{i-1}} \nabla_{\bx} \lossFunc(\bx_{i-1},\bz_{i-1}(\iterIndex))  \Bigg\| - \frac{1}{m} \Bigg\| \frac{1}{\numIter_{i-1}} \sum_{\iterIndex=1}^{\numIter_{i-1}} \nabla_{\bx} \lossFunc(\tilde{\bx}_{i-1},\bz_{i-1}(\iterIndex))  \Bigg\| \Bigg| \nonumber \\
&\leq& \| \left( \bx_{i} - \tilde{\bx}_{i} \right) - \left( \bx_{i-1} - \tilde{\bx}_{i-1}  \right) \| \nonumber \\
&& \;\;\;\;\;\; + \frac{1}{m} \Bigg\| \frac{1}{\numIter_{i}} \sum_{\iterIndex=1}^{\numIter_{i}} \left( \nabla_{\bx} \lossFunc(\bx_{i},\bz_{i}(\iterIndex))  -  \nabla_{\bx} \lossFunc(\tilde{\bx}_{i},\bz_{i}(\iterIndex)) \right)  \Bigg\|  \nonumber \\
&& \;\;\;\;\;\; + \frac{1}{m} \Bigg\| \frac{1}{\numIter_{i-1}} \sum_{\iterIndex=1}^{\numIter_{i-1}} \left( \nabla_{\bx} \lossFunc(\bx_{i-1},\bz_{i-1}(\iterIndex))  -  \nabla_{\bx} \lossFunc(\tilde{\bx}_{i-1},\bz_{i-1}(\iterIndex)) \right)  \Bigg\|  \nonumber  \\
\label{estRho:directEst:combineTwoLemma:exactBound}
&\leq& \left(1 + \frac{\lossLG}{m} \right) \left( \| \bx_{i} - \tilde{\bx}_{i} \| + \| \bx_{i-1} - \tilde{\bx}_{i-1} \|  \right) \nonumber
\end{eqnarray}
so
\begin{eqnarray}
|\hat{\rho}_{n} - \hat{\rho}_{n}^{(2)}| &\leq& \frac{1}{n-1} \sum_{i=2}^{n} |\tilde{\rho}_{i} - \tilde{\rho}_{i}^{(2)}| \nonumber \\
&\leq& \frac{\left(1 + \frac{\lossLG}{m} \right)}{n-1} \sum_{i=2}^{n} \left( \| \bx_{i} - \tilde{\bx}_{i} \| + \| \bx_{i-1} - \tilde{\bx}_{i-1} \|  \right) \nonumber \\
\label{estRho:directEst:combineTwoProofBound}
&=& \frac{\left(1 + \frac{\lossLG}{m} \right)}{n-1} \left( \| \bx_{1} - \tilde{\bx}_{1}\| + 2 \sum_{i=2}^{n-1} \| \bx_{i} - \tilde{\bx}_{i} \| + \| \bx_{n} - \tilde{\bx}_{n}\|  \right) \nonumber
\end{eqnarray}

We will again apply Lemma~\ref{subgauss:subgaussDepLem} of Appendix~\ref{usefulConcIneq} to analyze this upper bound using the sigma algebra
\begin{equation}
\label{estRho:directEst:combineTwoProofSigAlg}
\mathcal{F}_{i} = \sigma\left( \bigcup_{j=1}^{i} \{\bz_{j}(\iterIndex)\}_{\iterIndex=1}^{\numIter_{j}} \cup \bigcup_{j=1}^{i} \{\tilde{\bz}_{j}(\iterIndex)\}_{\iterIndex=1}^{\numIter_{j}}  \right) \vee \mathcal{K}_{0}  \;\;\; i=0,\ldots,n
\end{equation}
Define the random variable
\[
V_{i} = \| \bx_{i} - \tilde{\bx}_{i} \| - \mathbb{E}\left[ \| \bx_{i} - \tilde{\bx}_{i} \| \;\big|\; \mathcal{F}_{i-1} \right]
\]
Clearly, $V_{i}$ is $\mathcal{F}_{i}$-measurable. Since
\[
-\text{diam}(\xSp) \leq V_{i} \leq \text{diam}(\xSp),
\]
and $\mathbb{E}\left[ V_{i} \;|\; \mathcal{F}_{i-1} \right] = 0$, we can apply the conditional version Hoeffding's Lemma from Lemma~\ref{estRho:condHoeffdingLemma} to yield
\[
\mathbb{E}\left[ e^{s V_{i}}  \;\big|\; \mathcal{F}_{i-1} \right] \leq \exp\left\{ \frac{1}{2} \text{diam}^{2}(\xSp) s^{2}  \right\}
\]
The collection of random variables $\{V_{i}\}_{i=1}^{n}$ and the filtration $\{\mathcal{F}_{i}\}_{i=0}^{n}$ satisfy the conditions of Lemma~\ref{subgauss:subgaussDepLem}. Before applying Lemma~\ref{subgauss:subgaussDepLem}, we bound the conditional expectations
\[
\mathbb{E}\left[ \| \bx_{i} - \tilde{\bx}_{i} \| \;\big|\; \mathcal{F}_{i-1} \right]
\]
By assumption, we have
\[
\mathbb{E}\left[ \| \bx_{i} - \tilde{\bx}_{i} \| \;\big|\; \mathcal{F}_{i-1} \right] \leq C(\numIter_{i}) \;\;\;\; i=1,\ldots,n
\]
and so
\begin{align}
\frac{\left(1 + \frac{\lossLG}{m} \right)}{n-1} &\left( \mathbb{E}\left[  \| \bx_{1} - \tilde{\bx}_{1}\| \;\big|\; \mathcal{F}_{0} \right] + 2 \sum_{i=2}^{n-1} \mathbb{E}\left[  \| \bx_{i} - \tilde{\bx}_{i}\| \;\big|\; \mathcal{F}_{i-1} \right] \| + \mathbb{E}\left[  \| \bx_{n} - \tilde{\bx}_{n}\| \;\big|\; \mathcal{F}_{n-1} \right] \right) \nonumber \\
&\leq \frac{\left(1 + \frac{\lossLG}{m} \right)}{n-1} \left( C(\numIter_{1}) + 2 \sum_{i=2}^{n-1} C(\numIter_{i}) + C(\numIter_{n}) \right) \triangleq \hat{C}_{n} \nonumber
\end{align}
Set
\[
a_{1} = a_{n} = \frac{\left(1+ \frac{\lossLG}{m}\right)}{n-1} 
\]
and
\[
a_{2} = \cdots = a_{n-1} = \frac{\left(1+ \frac{\lossLG}{m}\right)}{n-1} 
\]
resulting in
\[
\|\bm{a}\|_{2}^{2} = \frac{n \left(1+\frac{\lossLG}{m}\right)^{2}}{(n-1)^{2}}
\]
Applying our bound in \eqref{estRho:directEst:combineTwoProofBound} and Lemma~\ref{subgauss:subgaussDepLem} with this choice of $\bm{a}$ yields
\begin{align}
\mathbb{P}&\left\{ |\hat{\rho}_{n} - \hat{\rho}_{n}^{(2)}| > \hat{C}_{n} +t  \right\} \nonumber \\
&\;\;\;\;\;\leq \mathbb{P}\left\{ \frac{\left(1 + \frac{\lossLG}{m} \right)}{n-1} \left( \| \bx_{1} - \tilde{\bx}_{1}\| + 2 \sum_{i=2}^{n-1} \| \bx_{i} - \tilde{\bx}_{i} \| + \| \bx_{n} - \tilde{\bx}_{n}\|  \right) \right. \nonumber \\
&\;\;\;\;\;\;\;\;\;\;\;\;\;\;\;\;\;\;\;\;\;\; \left. >\frac{\left(1 + \frac{\lossLG}{m} \right)}{n-1} \left( \mathbb{E}\left[  \| \bx_{1} - \tilde{\bx}_{1}\| \;\big|\; \mathcal{F}_{0} \right] + 2 \sum_{i=2}^{n-1} \mathbb{E}\left[  \| \bx_{i} - \tilde{\bx}_{i}\| \;\big|\; \mathcal{F}_{i-1} \right] \| + \mathbb{E}\left[  \| \bx_{n} - \tilde{\bx}_{n}\| \;\big|\; \mathcal{F}_{n-1} \right] \right) + t \right\} \nonumber \\
&\;\;\;\;\;=\mathbb{P}\left\{ \frac{\left(1 + \frac{\lossLG}{m} \right)}{n-1} \left( V_{1} + 2 \sum_{i=2}^{n-1} V_{i} + V_{n} \right) > t \right\} \nonumber \\
&\;\;\;\;\;=\mathbb{P}\left\{ \sum_{i=1}^{n} a_{i} V_{i} > t \right\} \nonumber \\
&\;\;\;\;\; \leq \exp\left\{ - \frac{(n-1)^{2}t^2}{2 n \left(1 + \frac{\lossLG}{m}  \right)^{2}\text{diam}^{2}(\xSp)} \right\} \nonumber
\end{align}

Finally, we have
\begin{eqnarray}
\sum_{n=2}^{\infty} \mathbb{P}\left\{ \hat{\rho}_{n} < \hat{\rho}_{n}^{(2)}  - \hat{C}_{n} - t_{n}  \right\} &\leq& \sum_{n=2}^{\infty} \mathbb{P}\left\{ |\hat{\rho}_{n} - \hat{\rho}_{n}^{(2)}| > \hat{C}_{n} + t_{n}  \right\} \nonumber \\
&\leq& \sum_{n=2}^{\infty} \exp\left\{ - \frac{(n-1)^{2}t_{n}^2}{2 n \left(1 + \frac{\lossLG}{m}  \right)^{2}\text{diam}^{2}(\xSp)} \right\} < +\infty \nonumber
\end{eqnarray}
The claim follows from the Borel-Cantelli Lemma.
\end{proof}

If Lemmas~\ref{estRho:directEst:combineOneLemma} and \ref{estRho:directEst:combineTwoLemma} hold for the sequence $\{t_{n}/2\}$, then for all $n$ large enough it holds that
\[
\hat{\rho}_{n} + \hat{C}_{n} + \hat{C}_{n}^{(2)} + t_{n} \geq \rho
\]
almost surely.

\begin{lem}
\label{estRho:directEst:sgdLemma}
It always holds that
\[
\mathbb{E}\left[ \| \bx_{i} - \tilde{\bx}_{i} \| \;\big|\; \mathcal{F}_{i-1} \right] \leq 2 \sqrt{\frac{1}{m} b\left(\text{diam}^{2}(\xSp),\numIter_{i} \right)} \nonumber
\]
Therefore, the choice
\[
C(\numIter_{i}) \triangleq 2 \sqrt{\frac{2}{m} b\left(\text{diam}^{2}(\xSp),\numIter_{i} \right)}
\]
satisfies the conditions of Lemma~\ref{estRho:directEst:combineTwoLemma}.
\end{lem}
\begin{proof}
Using the sigma algebras defined in \eqref{estRho:directEst:combineTwoProofSigAlg} yields
\begin{eqnarray}
\mathbb{E}\left[ \| \bx_{i} - \tilde{\bx}_{i} \| \;|\; \mathcal{F}_{i-1} \right] &\leq& \mathbb{E}\left[ \| \bx_{i} - \bx_{i}^{*} \| \;|\; \mathcal{F}_{i-1} \right] + \mathbb{E}\left[ \| \tilde{\bx}_{i} - \bx_{i}^{*} \| \;|\; \mathcal{F}_{i-1} \right] \nonumber \\
&\leq& \mathbb{E}\left[ \sqrt{\frac{2}{m} \left( f_{i}(\bx_{i}) - f_{i}(\bx_{i}^{*})  \right)  } \;|\; \mathcal{F}_{i-1} \right] + \mathbb{E}\left[ \sqrt{\frac{2}{m} \left( f_{i}(\tilde{\bx}_{i}) - f_{i}(\bx_{i}^{*})  \right)  } \;|\; \mathcal{F}_{i-1} \right] \nonumber \\
&\leq& \sqrt{\frac{2}{m} \mathbb{E}\left[ \left( f_{i}(\bx_{i}) - f_{i}(\bx_{i}^{*})  \right)   \;|\; \mathcal{F}_{i-1} \right]}  + \sqrt{\frac{2}{m} \mathbb{E}\left[ \left( f_{i}(\tilde{\bx}_{i}) - f_{i}(\bx_{i}^{*})  \right)   \;|\; \mathcal{F}_{i-1} \right]} \nonumber \\
&\leq& 2 \sqrt{\frac{2}{m} b(\text{diam}^{2}(\xSp),\numIter_{i})} \nonumber
\end{eqnarray}
where the third inequality follows from Jensen's inequality.
\end{proof}

This choice of $C(\numIter_{n})$works for any algorithm with the associated $b(d_{0},\numIter)$. For any particular algorithm, we believe that we can produce tighter bounds independent of $\text{diam}(\xSp)$ by copying the Lyapunov analysis used to analyze SGD as in Appendix~\ref{bBounds}. The analysis becomes algorithm dependent in this case and is omitted. 

Finally, we state an overall theorem for the direct estimate that gives general combined conditions under which $\hat{\rho}_{n}$ upper bounds $\rho$.
\begin{thm}
\label{estRho:directEst:combineTheorem}
If \ref{probState:assumpB1}-\ref{probState:assumpB2} hold and the sequence $\{t_{n}\}$ satisfies $\sum_{n=2}^{\infty} e^{- C n t_{n}^2} < \infty$ for all $C > 0$, then for a sequence of constants $\{C_{n}\}$ and for all $n$ large enough it holds that
$\hat{\rho}_{n} + C_{n} + t_{n} \geq \rho$
almost surely.
\end{thm}
\begin{proof}
Combine Lemmas~\ref{estRho:directEst:combineOneLemma} and \ref{estRho:directEst:combineTwoLemma} to yield the result with
\[
C_{n} = \hat{C}_{n} + \hat{C}_{n}^{(2)}
\]
\end{proof}

\subsubsection{Vector IPM Estimate}

We first derive a version of Hoeffding's inequality that allows for some dependence among the random variables. We use this concentration inequality to analyze $\hat{\rho}_{n}$ for the IPM estimate. Given an integer $W$, we construct a cover of $\{1,2,\ldots,n\}$ by dividing the set into $W$ groups of integers spaced by $W$, i.e.,
\begin{equation}
\label{estRho:ipmEstimate:coverDef}
\mathcal{A}_{j} = \left\{j,j+W,j+2W\ldots,j+ \bigg\lfloor \frac{n-j}{W} \bigg\rfloor W \right\} \;\;\;\;\;\; j=1,\ldots,W
\end{equation}
Note that 
\[
\{1,2,\ldots,n\} = \bigcup_{j=1}^{W} \mathcal{A}_{j}
\]
and $\mathcal{A}_{i} \cap \mathcal{A}_{j} = \emptyset$ for $i \neq j$. The proof of Lemma~\ref{conc:sum_dep_hoeffding} is nearly identical to the proof of the extension of Hoeffding's inequality from \cite{Janson04} with Lemma~\ref{subgauss:subgaussDepLem} used instead. We assume that if we refer to a filtration $\mathcal{F}_{i}$ with $i<0$, then we implicitly refer to $\mathcal{F}_{0}$.

\begin{lem}[Dependent Hoeffding's Inequality]
\label{conc:sum_dep_hoeffding}
Suppose we are given a collection of random variable $\{V_{i}\}_{i=1}^{n}$ and a filtration $\{\mathcal{F}\}_{i=0}^{n}$ such that
\begin{enumerate}
\item $a_{i} \leq V_{i} \leq b_{i}$ for constants $a_{i}$ and $b_{i}$  $\;\;\;\;\;i=1,\ldots,n$
\item $V_{i}$ is $\mathcal{F}_{i}$-measurable $\;\;\;\;\;i=1,\ldots,n$

\item Given an integer $W$ and a cover $\{\mathcal{A}_{j}\}_{j=1}^{W}$ as in \eqref{estRho:ipmEstimate:coverDef} for each $j$ it holds that
\[
\mathbb{E}\left[ V_{j+iW} \;\Big|\; \mathcal{F}_{j+(i-1)W}  \right] = 0 \;\;\;\;\; i=1,\ldots,\bigg\lfloor \frac{n-j}{W} \bigg\rfloor
\]
and
\[
\mathbb{E}\left[ V_{j} \;\Big|\; \mathcal{F}_{0}  \right] = 0
\]
\end{enumerate}
Then it holds that
\[
\mathbb{P}\left\{ \sum_{i=1}^{n} V_{i} >  t  \right\} \leq \exp\left\{ - \frac{2t^2}{ W \sum_{i=1}^{n}(b_{i}-a_{i})^2}  \right\}
\]
and
\[
\mathbb{P}\left\{ \sum_{i=1}^{n} V_{i} < - t  \right\} \leq \exp\left\{ - \frac{2t^2}{ W \sum_{i=1}^{n}(b_{i}-a_{i})^2}  \right\}
\]
\end{lem}
\begin{proof}
Define
\[
U_{j} \triangleq \sum_{i=0}^{\big\lfloor \frac{n-j}{W} \big\rfloor} V_{j+iW}
\] 
for $j = 1,\ldots,W$. Let $\{p_{j}\}_{j=1}^{W}$ be a probability distribution on $\{1,\ldots,W\}$ to be specified later. By Jensen's inequality, we have
\begin{eqnarray}
\exp\left\{ s \sum_{i=1}^{n} V_{i} \right\} &=& \exp\left\{ \sum_{j=1}^{W} p_{j} \frac{s}{p_{j}} U_{j}  \right\} \nonumber \\
&\leq& \sum_{j=1}^{W} p_{j} \exp\left\{ \frac{s}{p_{j}} U_{j}  \right\} \nonumber 
\end{eqnarray}
Then it holds that
\begin{eqnarray}
\mathbb{E}\left[ \exp\left\{ s \sum_{i=1}^{n} V_{i} \right\} \right] &\leq& \sum_{j=1}^{W} p_{j} \mathbb{E}\left[ \exp\left\{ \frac{s}{p_{j}} U_{j}  \right\} \right] \nonumber
\end{eqnarray}
Now consider one term 
\[
\mathbb{E}\left[ \exp\left\{ \frac{s}{p_{j}} U_{j}  \right\} \right] = \mathbb{E}\left[ \exp\left\{ \frac{s}{p_{j}} \sum_{i=0}^{\big\lfloor \frac{n-j}{W} \big\rfloor} V_{j+iW} \right\} \right]
\]
Since $a_{j+iW} \leq V_{j+iW} \leq b_{j+iW}$ and
\[
\mathbb{E}\left[ V_{j+iW} \;\Big|\; \mathcal{F}_{j+(i-1)W}  \right] = 0,
\]
we can apply the conditional version Hoeffding's Lemma from Lemma~\ref{estRho:condHoeffdingLemma} to yield
\begin{equation*}
\mathbb{E}\left[ e^{s V_{j+iW}} \;\big|\; \mathcal{F}_{j+(i-1)W}  \right] \leq \exp\left\{ \frac{1}{8} \left( b_{j+iW} - a_{j+iW} \right)^{2} s^{2}  \right\}
\end{equation*}
Then we can apply Lemma~\ref{subgauss:subgaussDepLem} to $\{V_{j+iW}\}_{i=0}^{\big\lfloor \frac{n-j}{W} \big\rfloor}$ and $\{\mathcal{F}_{j+iW}\}_{i=0}^{\big\lfloor \frac{n-j}{W} \big\rfloor}$
to yield
\begin{eqnarray}
\mathbb{E}\left[ \exp\left\{ \frac{s}{p_{j}} U_{j}  \right\} \right] &\leq&  \exp\left\{ \frac{s^2}{8p_{j}^2}\sum_{i=0}^{\big\lfloor \frac{n-j}{W} \big\rfloor}(b_{j+iW}-a_{j+iW})^2  \right\} \nonumber \\
&=& \prod_{i=0}^{\big\lfloor \frac{n-j}{W} \big\rfloor} \exp\left\{ \frac{s^2}{8p_{j}^2}(b_{\alpha}-a_{\alpha})^2  \right\} \nonumber
\end{eqnarray}
Then we have
\begin{eqnarray}
\mathbb{E}\left[ \exp\left\{ s \sum_{i=1}^{n} V_{i} \right\} \right] &\leq& \sum_{j=1}^{W} p_{j} \prod_{i=0}^{\big\lfloor \frac{n-j}{W} \big\rfloor} \exp\left\{ \frac{s^2}{8p_{j}^2}(b_{\alpha}-a_{\alpha})^2  \right\} \nonumber \\
&=& \sum_{j=1}^{W} p_{j} \exp\left\{ \frac{s^2 c_{j}}{8p_{j}^2}  \right\} \nonumber
\end{eqnarray}
with 
\[
c_{j} = \sum_{i=0}^{\big\lfloor \frac{n-j}{W} \big\rfloor} (b_{j+iW}-a_{j+iW})^2
\]
Let $p_{j} = \sqrt{c_{j}}/T$ and
\[
T = \sum_{j=1}^{W} \sqrt{c_{j}}.
\] 
Therefore, we have
\[
\mathbb{E}\left[ \exp\left\{ s \sum_{i=1}^{n} V_{i} \right\} \right] \leq \exp\left\{\frac{1}{8}T^2s^2\right\}
\]
Applying the Chernoff bound \cite{Boucheron13} and optimizing yields
\[
\mathbb{P}\left\{ \sum_{i=1}^{n} V_{i} > t \right\} \leq \exp\left\{ -2t^2/T^2 \right\}
\]
Bounding $T$ with Cauchy-Schwarz yields
\[
T^2 \leq \left( \sum_{j=1}^{W} 1  \right) \left( \sum_{j=1}^{W} c_{j}  \right) = W \sum_{i=1}^{n} (b_{i}-a_{i})^2
\]
and the results follows. The proof for the other tail is nearly identical.
\end{proof}

If we do not have the condition 3 of Lemma~\ref{conc:sum_dep_hoeffding}, then it holds that
\[
\mathbb{P}\left\{ \sum_{i=1}^{n} V_{i} > \sum_{j=1}^{W} \sum_{i=0}^{\big\lfloor \frac{n-j}{W} \big\rfloor} \mathbb{E}\left[ V_{j+iW} \;\big|\; \mathcal{F}_{j+(i-1)W}\right] + t  \right\} \leq \exp\left\{ - \frac{2t^2}{ W \sum_{i=1}^{n}(b_{i}-a_{i})^2}  \right\}
\]
If we can bound the conditional expectation
\[
\mathbb{E}\left[ V_{j+iW} \;\big|\; \mathcal{F}_{j+(i-1)W}\right] \leq C_{j+iW},
\]
by a $\mathcal{F}_{j+(i-1)W}$-measurable random variable, then we have
\begin{eqnarray}
\mathbb{P}\left\{ \sum_{i=1}^{n} V_{i} > \sum_{i=1}^{n} C_{i} + t  \right\} &=& \mathbb{P}\left\{ \sum_{i=1}^{n} V_{i} > \sum_{j=1}^{W} \sum_{i=0}^{\big\lfloor \frac{n-j}{W} \big\rfloor} C_{j+iW} + t  \right\}  \nonumber \\
&\leq& \mathbb{P}\left\{ \sum_{i=1}^{n} V_{i} > \sum_{j=1}^{W} \sum_{i=0}^{\big\lfloor \frac{n-j}{W} \big\rfloor} \mathbb{E}\left[ V_{j+iW} \;\big|\; \mathcal{F}_{j+(i-1)W}\right] + t  \right\} \nonumber \\
&\leq& \mathbb{P}\left\{ \sum_{j=1}^{W} \sum_{i=0}^{\big\lfloor \frac{n-j}{W} \big\rfloor} \left( V_{j+iW} - \mathbb{E}\left[ V_{j+iW} \;\big|\; \mathcal{F}_{j+(i-1)W}\right] \right) >  t  \right\} \nonumber \\
&\leq& \exp\left\{ - \frac{2t^2}{ W \sum_{i=1}^{n}(b_{i}-a_{i})^2}  \right\} \nonumber
\end{eqnarray}

We have the following lemma characterizing the performance of the IPM estimate.
\begin{lem}
\label{estRho:ipmEstimate:concLemma}
For the IPM estimate and any sequence $\{t_{n}\}$ such that
\[
\sum_{n=2}^{\infty} \exp\left\{ - \frac{nt_{n}^2}{4 \text{diam}(\xSp)^2}  \right\} < \infty
\]
for all $n$ large enough it holds that $\hat{\rho}_{n} + t_{n} \geq \rho$ almost surely.
\end{lem}
\begin{proof}
Define the random variables
\[
V_{i} = \tilde{\rho}_{i} - \mathbb{E}\left[ \tilde{\rho}_{i} \;|\; \mathcal{K}_{i-2} \right]
\]
with $\{\mathcal{K}_{i}\}_{i=1}^{n}$ defined in \eqref{estRho:H0SigAlg}. We have 
\[
-\text{diam}(\xSp) \leq V_{i} \leq \text{diam}(\xSp)
\]
Clearly, $V_{i}$ is $\mathcal{K}_{i}$-measurable and $\mathbb{E}[V_{i}\;|\;\mathcal{K}_{i-2}] = 0$. Now, we can apply Lemma~\ref{conc:sum_dep_hoeffding} with $W=2$ to yield
\begin{eqnarray}
\mathbb{P}\left\{ \sum_{i=1}^{n} V_{i} < -nt \right\} &\leq& \exp\left\{ - \frac{2 (nt)^2}{(2) \left(4 n \text{diam}^{2}(\xSp) \right) } \right\} \nonumber \\
&=& \exp\left\{ - \frac{nt^2}{4 \text{diam}^{2}(\xSp) } \right\} \nonumber
\end{eqnarray}

None of the random variables $\{\bz_{i}(\iterIndex)\}_{\iterIndex=1}^{\numIter_{i}}$ and $\{\bz_{i-1}(\iterIndex)\}_{\iterIndex=1}^{\numIter_{i-1}}$ are $\mathcal{K}_{i-2}$ measurable. Also, regardless of how many samples $\numIter_{i}$ and $\numIter_{i-1}$ are taken, the IPM estimate is biased upward. Thus, it holds that
\[
\mathbb{E}\left[ \tilde{\rho}_{i} \;|\; \mathcal{K}_{i-2} \right] \geq \rho
\]
Therefore, it follows that
\begin{eqnarray}
\mathbb{P}\left\{ \hat{\rho}_{n} < \rho - t \right\} &\leq& \mathbb{P}\left\{ \sum_{i=1}^{n} \tilde{\rho}_{i} < \sum_{i=1}^{n} \mathbb{E}\left[ \tilde{\rho}_{i} \;|\; \mathcal{K}_{i-2} \right] - nt \right\} \nonumber \\
&=& \mathbb{P}\left\{ \sum_{i=1}^{n} V_{i} < - nt \right\} \nonumber \\
&\leq& \exp\left\{ - \frac{nt^2}{4 \text{diam}^{2}(\xSp) } \right\} \nonumber
\end{eqnarray}
Note that we pay a price of two in the exponent due to $\tilde{\rho}_{i}$ and $\tilde{\rho}_{i-1}$ both depending on the samples from $p_{i-1}$. Since
\[
\sum_{n=2}^{\infty} \exp\left\{ - \frac{nt_{n}^2}{4 \text{diam}(\xSp)^2}  \right\} < \infty
\]
it follows that
\[
\sum_{n=2}^{\infty} \mathbb{P}\left\{ \hat{\rho}_{n} + t < \rho \right\} < + \infty,
\]
This in turn guarantees by way of the Borel-Cantelli Lemma that for $n$ large enough
\[
\hat{\rho}_{n} + t_{n} \geq \rho
\]
almost surely.
\end{proof}

\subsection{Combining One Step Estimates For Bounded Change}
We now look at estimating $\rho$ in the case that
\[
\| \bx_{n}^{*} - \bx_{n-1}^{*} \| \leq \rho.
\]
We set
\[
\rho_{i} \triangleq \| \bx_{i}^{*} - \bx_{i-1}^{*} \|
\]

\begin{description}
\item[B.3 \label{probState:assumpB3}] Assume that we have estimators $\hat{h}_{W}: \mathbb{R}^{W} \to \mathbb{R}$ such that
\begin{enumerate}
\item $\mathbb{E}[\hat{h}_{W}(\rho_{j},\ldots,\rho_{j-W+1})] \geq \rho$ for all $j \geq 1$ and $W \geq 1$
\item For any random variables $\{\tilde{\rho}_{i}\}$ such that $\mathbb{E}[\tilde{\rho}_{i}] \geq \mathbb{E}[\rho_{i}]$, we have
\[
\mathbb{E}\left[ \hat{h}_{W}(\tilde{\rho}_{j},\ldots,\tilde{\rho}_{j-W+1}) \right] \geq \mathbb{E}\left[ \hat{h}_{W}(\rho_{j},\ldots,\rho_{j-W+1}) \right]
\]
\end{enumerate}
\end{description}

For example, if $\rho_{i} \overset{\text{iid}}{\sim} \text{Unif}[0,\rho]$, then
\[
\hat{h}_{W}\left( \rho_{i},\rho_{i+1},\ldots,\rho_{i+W-1} \right) = \frac{W+1}{W} \max\{ \rho_{i},\rho_{i+1},\ldots,\rho_{i+W-1} \}
\]
is an estimator of $\rho$ with the required properties. Also, note that the two conditions on the estimator in~\ref{probState:assumpB3} imply that
\[
\mathbb{E}\left[ \hat{h}_{W}(\tilde{\rho}_{j},\ldots,\tilde{\rho}_{j-W+1}) \right] \geq \mathbb{E}\left[ \hat{h}_{W}(\rho_{j},\ldots,\rho_{j-W+1}) \right] \geq \rho
\]

Given an estimator satisfying assumption~\ref{probState:assumpB3}, we compute
\[
\tilde{\rho}^{(i)} = \hat{h}_{W}(\tilde{\rho}_{i},\tilde{\rho}_{i-1},\ldots,\tilde{\rho}_{i-W+1})
\]
and set
\begin{equation}
\label{ineqCond:basicEst}
\hat{\rho}_{n} = \frac{1}{n-1} \sum_{i=2}^{n} \tilde{\rho}^{(i)} = \frac{1}{n-1} \sum_{i=2}^{n} \hat{h}_{\min\{W,i-1\}}(\tilde{\rho}_{i},\tilde{\rho}_{i-1},\ldots,\tilde{\rho}_{\max\{i-W+1,2\}})
\end{equation}
We have 
\[
\mathbb{E}[\hat{\rho}_{n}] = \frac{1}{n-1} \sum_{i=2}^{n} \mathbb{E}[\tilde{\rho}^{(i)}] \geq \rho
\]

\begin{lem}[IPM Single Step Estimates]
\label{ineqCond:ipmEst}
For the estimator in $\eqref{ineqCond:basicEst}$ computed using the IPM estimate for $\tilde{\rho}_{i}$ and any sequence $\{t_{n}\}$ such that
\[
\sum_{n=2}^{\infty} \exp\left\{-\frac{2(n-1)t_{n}^2}{(W+1) \text{diam}(\xSp)^{2}} \right\} < \infty
\]
it holds that for all $n$ large enough $\hat{\rho}_{n} + t_{n} \geq \rho$ almost surely.
\end{lem}
\begin{proof}
We copy the proof of Lemma~\ref{estRho:ipmEstimate:concLemma} with $W+1$ in place of $2$ and note that $\tilde{\rho}^{(i)}$  and $\tilde{\rho}^{(j)}$ with $|i-j| > W+1$ do not depend on the same samples. Lemma~\ref{conc:sum_dep_hoeffding} and some simple algebra yields
\[
\mathbb{P}\left\{ \hat{\rho}_{n} < \rho - t \right\} \leq \exp\left\{-\frac{2(n-1)t^2}{(W+1) \text{diam}(\xSp)^{2}} \right\}
\]
We pay a price of $W+1$ in the denominator of the exponent due to the dependence of the $\tilde{\rho}^{(i)}$. By the Borel-Cantelli Lemma, for all $n$ large enough it holds that $\hat{\rho}_{n} + t_{n} \geq \rho$ almost surely as long as
\[
\sum_{n=2}^{\infty} \exp\left\{-\frac{2(n-1)t_{n}^2}{(W+1) \text{diam}(\xSp)^{2}} \right\} < \infty
\]
\end{proof}

To analyze the direct estimate, we need the following assumption
\begin{description}
\item[B.4 \label{probState:assumpB4}] Suppose that there exists absolute constants $\{b_{i}\}_{i=1}^{W}$ for any fixed $W$ such that
\[
|\hat{h}_{W}(p_{1},\ldots,p_{W}) - \hat{h}_{W}(q_{1},\ldots,q_{W})| \leq \sum_{i=1}^{W} b_{i} |p_{i} - q_{i}| \;\;\;\;\; \forall \bm{p},\bm{q} \in \mathbb{R}^{W}_{\geq 0}
\]
\end{description}

For the uniform case, we have
\begin{eqnarray}
\Big| \frac{W+1}{W} \max\{p_{1},\ldots,p_{W}\} - \frac{W+1}{W} \max\{q_{1},\ldots,q_{W}\} \Big| &\leq& \frac{W+1}{W} \max\left\{ |p_{1} - q_{1}|,\ldots , |p_{W} - q_{W}|  \right\} \nonumber \\
&\leq& \frac{W+1}{W} \sum_{i=1}^{W} |p_{i} - q_{i}| \nonumber
\end{eqnarray}
so
\[
b_{1} = \cdots = b_{W} = \frac{W+1}{W}
\]
Under assumption~\ref{probState:assumpB4}, we can then show that
\[
\hat{\rho}_{n} = \frac{1}{n-W} \sum_{i=W+1}^{n} \tilde{\rho}^{(i)}
\]
eventually upper bounds $\rho$ by copying the proofs of the lemmas behind Theorem~\ref{estRho:directEst:combineTheorem}.

\begin{lem}[Direct Single Step Estimates]
\label{ineqCond:dirEst}
Suppose that the following conditions hold:
\begin{enumerate}
\item \ref{probState:assumpB1}-\ref{probState:assumpB4} hold
\item The sequence $\{t_{n}\}$ satisfies
\[
\sum_{n=W+1}^{\infty} \exp\left\{ -\frac{(n-W)^2 t_{n}^2}{32 n \left(1+\frac{\lossLG}{m} \right)^{2} \left( \sum_{j=1}^{W} b_{j} \right)^{2} \text{diam}^{2}(\xSp)} \right\} < +\infty
\]
and
\[
\sum_{n=W+1}^{\infty} \exp\left\{ -\frac{(n-W)^2 m^{2} t_{n}^2}{144 n C_{g} \left( \sum_{j=1}^{W} b_{j} \right)^{2}} \right\} < +\infty
\]
\item There are bounds $C(\numIter)$ such that
\[
\mathbb{E}\left[  \| \bx_{i} - \tilde{\bx}_{i} \| \;|\; \mathcal{F}_{i-1} \right] \leq C(\numIter_{i})
\]
\end{enumerate}
Then for all $n$ large enough it holds that $\hat{\rho}_{n} + \hat{U}_{n} + \hat{V}_{n} + t_{n} \geq \rho$ almost surely with
\[
\hat{U}_{n} = \frac{2\left( 1 + \frac{\lossLG}{m} \right) \sum_{j = 1}^{W} b_{j}}{n-W} \sum_{i=1}^{n} C(\numIter_{i})
\]
and
\[
\hat{V}_{n} = \frac{2 \sum_{j = 1}^{W} b_{j}}{m(n-W)} \sum_{i=1}^{n} \sqrt{\frac{C_{g}}{d \numIter_{i}}}
\]
\end{lem}
\begin{proof}
Define $\tilde{\rho}_{i}^{(2)}$, $\tilde{\rho}_{i}^{(3)}$, $\hat{\rho}_{i}^{(2)}$, and $\hat{\rho}_{i}^{(3)}$ as in Lemmas~\ref{estRho:directEst:combineOneLemma} and \ref{estRho:directEst:combineTwoLemma}. First, we have
\begin{eqnarray}
|\hat{\rho}_{n} - \hat{\rho}_{n}^{(3)}| &\leq& \frac{1}{n-W} \sum_{i=W+1}^{n} |\tilde{\rho}^{(i)} - \tilde{\rho}^{(i)}_{3}| \nonumber \\
&\leq& \frac{1}{n-W} \sum_{i=W+1}^{n} \sum_{j = i-W+1}^{i} b_{j} |\tilde{\rho}_{j} - \tilde{\rho}_{j}^{(3)}| \nonumber \\
&\leq& \frac{1}{n-W} \sum_{i=W+1}^{n} \sum_{j = i-W+1}^{i} b_{j} \left( |\tilde{\rho}_{j} - \tilde{\rho}_{j}^{(2)}| +  |\tilde{\rho}_{j}^{(2)} - \tilde{\rho}_{j}^{(3)}| \right) \nonumber  \\
&\leq& \frac{\sum_{j = 1}^{W} b_{j}}{n-W} \sum_{i=2}^{n}  \left( |\tilde{\rho}_{i} - \tilde{\rho}_{i}^{(2)}| +  |\tilde{\rho}_{i}^{(2)} - \tilde{\rho}_{i}^{(3)}| \right) \nonumber 
\end{eqnarray}
Second, define
\[
U_{i} \triangleq \| \bx_{i} - \tilde{\bx}_{i} \|
\]
and
\[
V_{i} \triangleq \Bigg\| \frac{1}{\numIter_{i}} \sum_{\iterIndex=1}^{\numIter_{i}} \left( \nabla_{\bx} \lossFunc(\tilde{\bx}_{i},\bz_{i}(\iterIndex)) - \nabla f_{i}(\tilde{\bx}_{i}) \right) \Bigg\|
\]
Then we have
\begin{eqnarray}
|\tilde{\rho}_{i} - \tilde{\rho}_{i}^{(2)}| &\leq& \| \bx_{i} - \tilde{\bx}_{i} \| +  \frac{1}{m} \Bigg\| \frac{1}{\numIter_{i}} \sum_{\iterIndex=1}^{\numIter_{i}} \left( \nabla_{\bx} \lossFunc(\bx_{i},\bz_{i}(\iterIndex))  -  \nabla_{\bx} \lossFunc(\tilde{\bx}_{i},\bz_{i}(\iterIndex)) \right) \Bigg\|  \nonumber \\
&\leq&  \left(1 + \frac{\lossLG}{m} \right) (U_{i} + U_{i-1}) \nonumber
\end{eqnarray}
and
\begin{eqnarray}
|\tilde{\rho}_{i}^{(2)} - \tilde{\rho}_{i}^{(3)}| &\leq& \frac{1}{m} \left(V_{i} + V_{i-1} \right) \nonumber
\end{eqnarray}
Then it follows that
\begin{eqnarray}
|\hat{\rho}_{n} - \hat{\rho}_{n}^{(3)}| &\leq& \frac{\sum_{j = 1}^{W} b_{j}}{n-W} \sum_{i=2}^{n}  \left( |\tilde{\rho}_{i} - \tilde{\rho}_{i}^{(2)}| +  |\tilde{\rho}_{i}^{(2)} - \tilde{\rho}_{i}^{(3)}| \right) \nonumber \\
&\leq& \frac{2\left( 1 + \frac{\lossLG}{m} \right) \sum_{j = 1}^{W} b_{j}}{n-W} \sum_{i=1}^{n} U_{i}  + \frac{2 \sum_{j = 1}^{W} b_{j}}{m(n-W)} \sum_{i=1}^{n} V_{i} \nonumber
\end{eqnarray}
Suppose that
\[
\frac{2\left( 1 + \frac{\lossLG}{m} \right) \sum_{j = 1}^{W} b_{j}}{n-W} \sum_{i=1}^{n} \mathbb{E}\left[ U_{i} \;|\; \mathcal{F}_{i-1} \right] \leq \hat{U}_{n}
\]
and
\[
\frac{2 \sum_{j = 1}^{W} b_{j}}{m(n-W)} \sum_{i=1}^{n}  \mathbb{E}\left[ V_{i} \;|\; \mathcal{F}_{i-1} \right] \leq \hat{V}_{n}
\]
Then it holds that
\begin{align}
\mathbb{P}&\left\{ |\hat{\rho}_{n} - \hat{\rho}_{n}^{(3)}| > \hat{U}_{n} + \hat{V}_{n} + t \right\} \nonumber \\
&\;\;\; \leq \mathbb{P}\left\{ \frac{2\left( 1 + \frac{\lossLG}{m} \right) \sum_{j = 1}^{W} b_{j}}{n-W} \sum_{i=1}^{n} U_{i}  + \frac{2 \sum_{j = 1}^{W} b_{j}}{m(n-W)} \sum_{i=1}^{n} V_{i} > \hat{U}_{n} + \hat{V}_{n} + t \right\} \nonumber \\
&\;\;\; \leq \mathbb{P}\left\{ \frac{2\left( 1 + \frac{\lossLG}{m} \right) \sum_{j = 1}^{W} b_{j}}{n-W} \sum_{i=1}^{n} U_{i}  > \hat{U}_{n} + \frac{t}{2} \right\} + \mathbb{P}\left\{ \frac{2 \sum_{j = 1}^{W} b_{j}}{m(n-W)} \sum_{i=1}^{n} V_{i} > \hat{V}_{n} + \frac{t}{2} \right\} \nonumber
\end{align}
We can apply Lemma~\ref{subgauss:subgaussDepLem} to each term to yield
\begin{equation*}
\mathbb{P}\left\{ \frac{2\left( 1 + \frac{\lossLG}{m} \right) \sum_{j = 1}^{W} b_{j}}{n-W} \sum_{i=1}^{n} U_{i}  > \hat{U}_{n} + \frac{t}{2} \right\} \leq \exp\left\{ -\frac{(n-W)^2 t^2}{32 n \left(1+\frac{\lossLG}{m} \right)^{2} \left( \sum_{j=1}^{W} b_{j} \right)^{2} \text{diam}^{2}(\xSp)} \right\}
\end{equation*}
and
\begin{equation*}
\mathbb{P}\left\{ \frac{2 \sum_{j = 1}^{W} b_{j}}{m(n-W)} \sum_{i=1}^{n} V_{i} > \hat{V}_{n} + \frac{t}{2} \right\} \leq \exp\left\{ -\frac{(n-W)^2 m^{2} t^2}{144 n C_{g} \left( \sum_{j=1}^{W} b_{j} \right)^{2}} \right\}
\end{equation*}
Then it holds that
\begin{align}
\mathbb{P}&\left\{ |\hat{\rho}_{n} - \hat{\rho}_{n}^{(3)}| > \hat{U}_{n} + \hat{V}_{n} + t \right\} \nonumber \\
&\;\;\; \leq \exp\left\{ -\frac{(n-W)^2 t^2}{32 n \left(1+\frac{\lossLG}{m} \right)^{2} \left( \sum_{j=1}^{W} b_{j} \right)^{2} \text{diam}^{2}(\xSp)} \right\} + \exp\left\{ -\frac{(n-W)^2 m^{2} t^2}{144 n C_{g} \left( \sum_{j=1}^{W} b_{j} \right)^{2}} \right\} \nonumber
\end{align}
We have by straightforward computation
\[
\hat{U}_{n} = \frac{2\left( 1 + \frac{\lossLG}{m} \right) \sum_{j = 1}^{W} b_{j}}{n-W} \sum_{i=1}^{n} C(\numIter_{i})
\]
and
\[
\hat{V}_{n} = \frac{2 \sum_{j = 1}^{W} b_{j}}{m(n-W)} \sum_{i=1}^{n} \sqrt{\frac{C_{g}}{d \numIter_{i}}}
\]
Then it holds that
\begin{align}
\sum_{n=W+1}^{\infty} &\mathbb{P}\left\{ \hat{\rho}_{n} < \rho - \hat{U}_{n} - \hat{V}_{n} - t_{n} \right\} \nonumber \\
&\;\;\;\; \leq \sum_{n=W+1}^{\infty} \mathbb{P}\left\{ \hat{\rho}_{n} < \hat{\rho}_{n}^{(3)} - \hat{U}_{n} - \hat{V}_{n} - t_{n} \right\} \nonumber \\
&\;\;\;\; \leq \sum_{n=W+1}^{\infty} \mathbb{P}\left\{ |\hat{\rho}_{n} - \hat{\rho}_{n}^{(3)}| > \hat{U}_{n} + \hat{V}_{n} + t_{n} \right\} \nonumber \\
&\;\;\;\; \leq \sum_{n=W+1}^{\infty} \exp\left\{ -\frac{(n-W)^2 t_{n}^2}{32 n \left(1+\frac{\lossLG}{m} \right)^{2} \left( \sum_{j=1}^{W} b_{j} \right)^{2} \text{diam}^{2}(\xSp)} \right\} + \sum_{n=W+1}^{\infty} \exp\left\{ -\frac{(n-W)^2 m^{2} t_{n}^2}{144 n C_{g} \left( \sum_{j=1}^{W} b_{j} \right)^{2}} \right\} \nonumber \\
&\;\;\;\; < \infty \nonumber
\end{align}
By the Borel-Cantelli lemma, it follows that for all $n$ large enough
\[
\hat{\rho}_{n} + \hat{U}_{n} + \hat{V}_{n} + t_{n} \leq \rho
\]
almost surely.
\end{proof}

\subsection{Parameter Estimation}
We may need to estimate parameters of the functions $\{f_{n}\}$ such as the strong convexity parameter $m$ to compute $b(d_{0},\numIter)$. We need the following assumption on our bound:
\begin{description}
\item[D.1 \label{probState:assumpD1}] Suppose that our bound $b(d_{0},\numIter,\psi)$ is parameterized by $\psi$, which depends on properties of the function $\lossFunc(\bx,\bz)$ and the distributions $\{p_{n}\}_{n=1}^{\infty}$. Suppose that
\[
\psi_{1} \leq \psi_{2} \;\;\Leftrightarrow\;\; b(d_{0},\numIter,\psi_{1}) \leq b(d_{0},\numIter,\psi_{2})
\]
\item[D.2 \label{probState:assumpD2}] There exists a true set of parameters $\psi^{*}$ such that
\[
\psi_{n} = \psi^{*} \;\;\;\; \forall n \geq 1
\]
\item[D.3 \label{probState:assumpD3}] The spaces $\mathcal{X}$ and $\mathcal{Z}$ are compact
\item[D.4 \label{probState:assumpD4}] There exists a constant $L$ such that
\[
\| \nabla_{\bx} \lossFunc(\bx,\bz) - \nabla_{\bx} \lossFunc(\tilde{\bx},\bz) \| \leq L \| \bx - \tilde{\bx} \|
\]
\item[D.5 \label{probState:assumpD5}] Suppose that we know that the parameters $\psi \in \mathcal{P}$ with $\mathcal{P}$ compact

\item[D.6 \label{probState:assumpD6}] Suppose that $\nabla f_{n}(\bx_{n})$ has Lipschitz continuous gradients with modulus $M$

\end{description}
As a consequence of Assumption~\ref{probState:assumpD4}, it follows that there exists a constant $G$ such that there exists a constant $G$ such that
\[
\| \nabla_{\bx} \lossFunc(\bx,\bz) \| \leq G \;\;\;\; \forall \bx \in \xSp, \bz \in \mathcal{Z}
\]
Satisfying Assumption~\ref{probState:assumpD5} is usually easy due to the compactness assumptions in Assumption~\ref{probState:assumpD4}.

In most cases, we have
\[
\psi = \left[ \begin{array}{c}
-m \\
M \\
A \\
B
\end{array} \right]
\]
where $m$ is the parameter of strong convexity, $M$ is the Lipschitz gradient modulus, and the pair $(A,B)$ controls gradient growth, i.e.,
\[
\mathbb{E}\| \nabla_{\bx} \lossFunc(\bx,\bz) \|^{2} \leq A +  B \| \bx-  \bx^{*} \|^{2}
\]
We parameterize using $-m$, since smaller $m$ increase the bound $b(d_{0},\numIter)$. We present several general methods for estimating these parameters, although in practice, problem specific estimators based on the form of the function may offer better performance. As an example, we present problem specific estimates for 
\[
\lossFunc(\bx,\bz) = \frac{1}{2} \left( y - \bw^{\top} \bx \right)^{2} + \frac{1}{2} \lambda \| \bx \|^{2}
\]

As in estimating $\rho$, we produce one time instant estimates $\tilde{m}_{i}$, $\tilde{M}_{i}$, $\tilde{A}_{i}$, and $\tilde{B}_{i}$ at time $i$ and combine them. We only examine the case under Assumption~\ref{probState:assumpD4}, although we could examine an inequality constraints as with estimating $\rho$. We combine estimates by averaging to yield
\begin{enumerate}
\item $\hat{m}_{n} = \frac{1}{n} \sum_{i=1}^{n} \tilde{m}_{i}$
\item $\hat{M}_{n} = \frac{1}{n} \sum_{i=1}^{n} \tilde{M}_{i}$
\item $\hat{A}_{n} = \frac{1}{n} \sum_{i=1}^{n} \tilde{A}_{i}$
\item $\hat{B}_{n} = \frac{1}{n} \sum_{i=1}^{n} \tilde{B}_{i}$
\end{enumerate}

\subsubsection{Estimating Strong Convexity Parameter and Lipschitz Gradient Modulus}

We seek one step estimators $\tilde{m}_{n}$ and $\tilde{M}_{n}$ such that
\[
\mathbb{E}[\tilde{m}_{n} \;|\; \mathcal{K}_{n-1}  ] \leq m
\]
and
\[
\mathbb{E}[\tilde{M}_{n} \;|\; \mathcal{K}_{n-1} ] \geq M
\]
with $\{\mathcal{K}_{n}\}$ defined in \eqref{estRho:H0SigAlg}.

\noindent \emph{Hessian Method:} We exploit the fact that
\[
\nabla_{\bx\bx}^{2} f_{n}(\bx) \succeq m \bm{I} \;\;\;\; \forall \bx \in \xSp
\] 
This in turn implies that
\[
\lambda_{\text{min}} \left( \nabla_{\bx\bx}^{2} f_{n}(\bx) \right) \geq m \;\;\;\; \forall \bx \in \xSp
\]
This suggests that given $\{\bz_{n}(\iterIndex)\}_{\iterIndex=1}^{\numIter_{n}}$ we set
\[
\tilde{m}_{n} \triangleq \min_{\bx \in \xSp} \lambda_{\text{min}} \left( \frac{1}{\numIter_{n}} \sum_{\iterIndex=1}^{\numIter_{n}} \nabla_{\bx \bx}^{2} \lossFunc(\bx,\bz_{n}(\iterIndex)) \right)
\]
Since
\[
\lambda_{\text{min}}(A) = \min_{\bm{v}:\|\bm{v}\|=1} \inprod{A \bm{v}}{\bm{v}},
\]
$\lambda_{\text{min}}(A)$ is a concave function of $A$. Then by Jensen's inequality, we have
\begin{eqnarray}
\mathbb{E}[\tilde{m}_{n}] &=& \mathbb{E}\left[ \min_{\bx \in \xSp} \lambda_{\text{min}} \left( \frac{1}{\numIter_{n}} \sum_{\iterIndex=1}^{\numIter_{n}} \nabla_{\bx \bx}^{2} \lossFunc(\bx,\bz_{n}(\iterIndex)) \right) \;\bigg|\; \mathcal{K}_{n-1}  \right] \nonumber \\
&\leq& \min_{\bx \in \xSp} \mathbb{E}\left[ \lambda_{\text{min}} \left( \frac{1}{\numIter_{n}} \sum_{\iterIndex=1}^{\numIter_{n}} \nabla_{\bx \bx}^{2} \lossFunc(\bx,\bz_{n}(\iterIndex)) \right) \;\bigg|\; \mathcal{K}_{n-1} \right] \nonumber \\
&\leq& \min_{\bx \in \xSp} \lambda_{\text{min}} \left( \mathbb{E}\left[ \frac{1}{\numIter_{n}} \sum_{\iterIndex=1}^{\numIter_{n}} \nabla_{\bx \bx}^{2} \lossFunc(\bx,\bz_{n}(\iterIndex)) \;\bigg|\; \mathcal{K}_{n-1} \right] \right)  \nonumber \\
&=& \min_{\bx \in \xSp} \lambda_{\text{min}} \left( \nabla_{\bx \bx}^{2} f_{n}(\bx)  \right)  \nonumber \\
&=& m \nonumber
\end{eqnarray}
Similarly, we can set
\[
\tilde{M}_{n} \triangleq \max_{\bx \in \xSp} \lambda_{\text{max}} \left( \frac{1}{\numIter_{n}} \sum_{\iterIndex=1}^{\numIter_{n}} \nabla_{\bx \bx}^{2} \lossFunc(\bx,\bz_{n}(\iterIndex)) \right)
\]
Since
\[
\lambda_{\text{max}}(A) = \max_{\bm{v}:\|\bm{v}\|=1} \inprod{A \bm{v}}{\bm{v}},
\]
$\lambda_{\text{max}}(A)$ is a convex function of $A$. By Jensen's inequality, it holds that
\[
\mathbb{E}[\tilde{M}_{n} \;|\; \mathcal{K}_{n-1} ] \geq M
\]

\emph{Gradient Method To Compute $\tilde{m}_{n}$:} To actually minimize over $\bx$, we can use gradient descent. To apply gradient descent, we use eigenvalue perturbation results \cite{Trefethen1997}. Suppose that we have a base matrix $T_{0}$ with eigenvectors $\bm{v}_{0i}$ and eigenvalues $\lambda_{0i}$. We want to find the eigenvectors $\bv_{i}$ and eigenvalues $\lambda_{i}$ of a perturbed matrix $T$:
\begin{eqnarray}
\bm{T}_{0} \bm{v}_{0i} &=& \lambda_{0i} \bm{v}_{0i} \nonumber \\
\bm{T} \bm{v}_{i} &=& \lambda_{i} \bm{v}_{i} \nonumber
\end{eqnarray}
In particular, we want to relate $\lambda_{0i}$ to $\lambda_{i}$. With
\[
\delta \bm{T} \triangleq \bm{T} - \bm{T}_{0},
\]
we have
\[
\delta \lambda_{i} = \bm{v}_{0i}^{\top} \left( \delta \bm{T}  \right) \bm{v}_{0i}
\]
and
\[
\frac{\partial \lambda_{i}}{\partial \bm{T}_{ij}} = \bm{v}_{0i}(i) \bm{v}_{0j} (2 - \delta_{ij})
\]
Suppose we are given a matrix-valued function $\bm{T}(x)$ with
\[
\bm{T}(\bx) \bm{v}(\bx) = \lambda_{\text{min}}(\bx) \bm{v}(\bx)
\]
Then it holds that
\begin{eqnarray}
\nabla_{\bx} \lambda_{\text{min}}\left( \bm{T}(\bx) \right) &=& \sum_{i,j} \frac{\partial \lambda_{\text{min}}}{\partial \bm{T}_{ij}} \nabla_{\bx} \bm{T}_{ij}(\bx) \nonumber \\
&=& \sum_{i,j} \bm{v}_{i}(\bx) \bm{v}_{j}(\bx) (2 - \delta_{ij})   \nabla_{\bx} \bm{T}_{ij}(\bx) \nonumber
\end{eqnarray}
Then we can use gradient descent to solve
\[
\min_{\bx \in \xSp} \lambda_{\text{min}} \left( \frac{1}{\numIter_{n}} \sum_{\iterIndex=1}^{\numIter_{n}} \nabla_{\bx} \lossFunc(\bx,\bz_{n}(\iterIndex))  \right)
\]
Starting from any $\bx(0)$, we can compute
\[
\bx(p) = \Pi_{\xSp}\left[ \bx(p-1) - \mu \nabla_{\bx} \lambda_{\text{min}} \left( \frac{1}{\numIter_{n}} \sum_{\iterIndex=1}^{\numIter_{n}} \nabla_{\bx \bx}^{2}\lossFunc(\bx,\bz_{n}(\iterIndex))  \right)  \right] \;\;\;\; p =1 ,\ldots,P
\]
and set
\begin{equation}
\label{paramEst:strCvxTwo}
\hat{m}_{n} \triangleq \lambda_{\text{min}} \left( \frac{1}{\numIter_{n}} \sum_{\iterIndex=1}^{\numIter_{n}} \nabla_{\bx \bx}^{2}\lossFunc(\bx(P),\bz_{n}(\iterIndex)) \right)
\end{equation}

\noindent \emph{Heuristic Method:} For any two points $\bx$ and $\by$, we have by strong convexity
\[
f_{n}(\by) \geq f_{n}(\bx) + \inprod{\nabla f_{n}(\bx)}{\by -\bx} + \frac{1}{2} m \|\by - \bx\|^{2}
\]
Suppose that we have $N$ points $\bx(1),\ldots,\bx(N)$. Then we know that for any two distinct points $\bx_{i}$ and $\bx_{j}$
\[
m \leq \frac{f_{n}(\bx(i)) - f_{n}(\bx(j)) - \inprod{\nabla f_{n}(\bx(j))}{\bx(i) - \bx(j)}}{\frac{1}{2} \| \bx(i) - \bx(j) \|^{2}}
\]
This suggests the estimator
\begin{equation}
\label{paramEst:strCvxOne}
\hat{m}_{n} \triangleq \min_{i \neq j} \frac{ \frac{1}{\numIter_{n}} \sum_{\iterIndex=1}^{\numIter_{n}} \lossFunc(\bx(i),\bz_{n}(\iterIndex)) - \frac{1}{\numIter_{n}} \sum_{\iterIndex=1}^{\numIter_{n}} \lossFunc(\bx(j),\bz_{n}(\iterIndex)) -\inprod{ \frac{1}{\numIter_{n}} \sum_{\iterIndex=1}^{\numIter_{n}} \nabla_{\bx} \lossFunc(\bx(j),\bz_{n}(\iterIndex)) }{ \bx(i) - \bx(j) } }{\frac{1}{2} \| \bx(i) - \bx(j) \|^{2}}
\end{equation}
for the strong convexity parameter. Then we have
\begin{align}
\mathbb{E}&[\hat{m}_{n}] \nonumber \\
&\;\;\;\; = \mathbb{E}\left[ \min_{i \neq j} \frac{ \frac{1}{\numIter_{n}} \sum_{\iterIndex=1}^{\numIter_{n}} \lossFunc(\bx(i),\bz_{n}(\iterIndex)) - \frac{1}{\numIter_{n}} \sum_{\iterIndex=1}^{\numIter_{n}} \lossFunc(\bx(j),\bz_{n}(\iterIndex)) -\inprod{ \frac{1}{\numIter_{n}} \sum_{\iterIndex=1}^{\numIter_{n}} \nabla_{\bx} \lossFunc(\bx(j),\bz_{n}(\iterIndex)) }{ \bx(i) - \bx(j) } }{\frac{1}{2} \| \bx(i) - \bx(j) \|^{2}}  \right] \nonumber \\
&\;\;\;\; \leq \min_{i \neq j} \mathbb{E}\left[ \frac{ \frac{1}{\numIter_{n}} \sum_{\iterIndex=1}^{\numIter_{n}} \lossFunc(\bx(i),\bz_{n}(\iterIndex)) - \frac{1}{\numIter_{n}} \sum_{\iterIndex=1}^{\numIter_{n}} \lossFunc(\bx(j),\bz_{n}(\iterIndex)) -\inprod{ \frac{1}{\numIter_{n}} \sum_{\iterIndex=1}^{\numIter_{n}} \nabla_{\bx} \lossFunc(\bx(j),\bz_{n}(\iterIndex)) }{ \bx(i) - \bx(j) } }{\frac{1}{2} \| \bx(i) - \bx(j) \|^{2}}  \right] \nonumber \\
&\;\;\;\; \leq \min_{i \neq j} \frac{f_{n}(\bx(i)) - f_{n}(\bx(j)) - \inprod{\nabla f_{n}(\bx(j))}{\bx(i) - \bx(j)}}{\frac{1}{2} \| \bx(i) - \bx(j) \|^{2}} \nonumber
\end{align}
It is difficult to compare this estimator to $m$ exactly. All we can say is that
\[
m \leq \min_{i \neq j} \frac{f_{n}(\bx(i)) - f_{n}(\bx(j)) - \inprod{\nabla f_{n}(\bx(j))}{\bx(i) - \bx(j)}}{\frac{1}{2} \| \bx(i) - \bx(j) \|^{2}}
\]
as well. In practice, this method produces estimates close to $m$.

Similarly, we can set
\begin{equation}
\label{paramEst:lipGradOne}
\hat{M}_{n} \triangleq \max_{i \neq j} \frac{ \frac{1}{\numIter_{n}} \sum_{\iterIndex=1}^{\numIter_{n}} \lossFunc(\bx(i),\bz_{n}(\iterIndex)) - \frac{1}{\numIter_{n}} \sum_{\iterIndex=1}^{\numIter_{n}} \lossFunc(\bx(j),\bz_{n}(\iterIndex)) -\inprod{ \frac{1}{\numIter_{n}} \sum_{\iterIndex=1}^{\numIter_{n}} \nabla_{\bx} \lossFunc(\bx(j),\bz_{n}(\iterIndex)) }{ \bx(i) - \bx(j) } }{\frac{1}{2} \| \bx(i) - \bx(j) \|^{2}}
\end{equation}

\noindent \emph{Problem Specific:} For the penalized quadratic, we have
\[
\nabla_{\bx \bx}^{2} \lossFunc(\bx,\bz) = \lambda \bm{I} + \bm{w} \bm{w}^{\top}
\]
so
\[
\nabla_{\bx \bx}^{2} f_{n}(\bx) = \lambda \bm{I} + \mathbb{E}[\bm{w}_{n} \bm{w}_{n}^{\top}]
\]
This suggests the simple closed-form estimates
\[
\tilde{m}_{n} = \lambda + \lambda_{\text{min}}\left( \frac{1}{\numIter_{n}} \sum_{\iterIndex=1}^{\numIter_{n}} \bm{w}_{n}(\iterIndex) \bm{w}_{n}(\iterIndex)^{\top}   \right)
\]
and
\[
\tilde{M}_{n} = \lambda + \lambda_{\text{max}}\left( \frac{1}{\numIter_{n}} \sum_{\iterIndex=1}^{\numIter_{n}} \bm{w}_{n}(\iterIndex) \bm{w}_{n}(\iterIndex)^{\top}   \right)
\]
Again, by Jensen's inequality, it holds that
\[
\mathbb{E}[\tilde{m}_{n} \;|\; \mathcal{K}_{n-1}] \leq m
\]
and
\[
\mathbb{E}[\tilde{M}_{n} \;|\; \mathcal{K}_{n-1}] \geq M
\]

\noindent \emph{Combining Estimates:} We now look at combining the single time instant estimates of the strong convexity parameter and the Lipschitz gradient modulus.

\begin{lem}
\label{estRho:paramEst:strCvx}
Choose $t_{n}$ such that for all $C > 0$ it holds that
\[
\sum_{n=1}^{\infty} e^{-Cnt_{n}^{2}} < + \infty
\]
Then for all $n$ large enough it holds that
\begin{enumerate}
\item $\hat{m}_{n} - t_{n} \leq m$
\item $\hat{M}_{n} + t_{n} \geq M$
\end{enumerate}
almost surely.
\end{lem}
\begin{proof}
 By the compactness of the space $\mathcal{P}$ containing $\psi$, we can apply the dependent version of Hoeffding's lemma (Lemma~\ref{estRho:condHoeffdingLemma}) to yield
\[
\mathbb{E}\left[ e^{s \tilde{m}_{i}} \;\big|\; \mathcal{K}_{i-1} \right] \leq \exp\left\{ \frac{1}{2} \sigma_{m}^{2} s^{2}  \right\}
\]
and
\[
\mathbb{E}\left[ e^{s \tilde{M}_{i}} \;\big|\; \mathcal{K}_{i-1} \right] \leq \exp\left\{ \frac{1}{2} \sigma_{M}^{2} s^{2}  \right\}
\]
for some constants $\sigma_{m}^{2}$ and $\sigma_{M}^{2}$ derived from Hoeffding's lemma. Then applying Lemma~\ref{subgauss:subgaussDepLem}, it follows that
\[
\mathbb{P}\left\{ \hat{m}_{n} > \frac{1}{n} \sum_{i=1}^{n} \mathbb{E}[\tilde{m}_{i} \;|\; \mathcal{K}_{i-1}] + t_{n} \right\} \leq \exp\left\{ - \frac{nt_{n} ^{2}}{2 \sigma_{m}^{2}} \right\}
\]
We know that
\[
\frac{1}{n} \sum_{i=1}^{n} \mathbb{E}[\tilde{m}_{i} \;|\; \mathcal{K}_{i-1}] > m
\]
so it follows that
\[
\mathbb{P}\left\{ \hat{m}_{n} > m + t_{n} \right\} \leq \exp\left\{ - \frac{nt_{n} ^{2}}{2 \sigma_{m}^{2}} \right\}
\]
Similarly, for the Lipschitz gradient modulus, it holds that
\[
\mathbb{P}\left\{ \hat{M}_{n} < M - t_{n} \right\} \leq \exp\left\{ - \frac{nt_{n} ^{2}}{2 \sigma_{M}^{2}} \right\}
\]
As before, we have
\[
\sum_{n=1}^{\infty} \mathbb{P}\left\{ \hat{m}_{n} > m + t_{n} \right\} \leq \sum_{n=1}^{\infty} \exp\left\{ - \frac{n t_{n}^{2}}{2 \sigma_{m}^{2}} \right\} < + \infty
\]
and
\[
\sum_{n=1}^{\infty} \mathbb{P}\left\{ \hat{M}_{n} < M - t_{n} \right\} \leq \sum_{n=1}^{\infty} \exp\left\{ - \frac{n t_{n}^{2}}{2 \sigma_{M}^{2}} \right\} < + \infty
\]
to ensure that almost surely for all $n$ large enough it holds that 
\[
\hat{m}_{n} - t_{n} \leq m
\]
and
\[
\hat{M}_{n} + t_{n} \geq m
\]
\end{proof}
For Lemma~\ref{estRho:paramEst:strCvx}, we need $t_{n}$ to decay no faster that $\mathcal{O}(n^{-1/2})$.

\subsubsection{Estimating Gradient Parameters}

From Assumption~\ref{probState:assumpD6}, it holds that
\begin{eqnarray}
\mathbb{E}\| \nabla_{\bx} \lossFunc(\bx,\bz) \|^{2} &=&  \mathbb{E}\| \nabla_{\bx} \lossFunc(\bx^{*},\bz)  + \left( \nabla_{\bx} \lossFunc(\bx,\bz) - \nabla_{\bx} \lossFunc(\bx^{*},\bz) \right) \|^{2} \nonumber \\
&\leq& 2 \mathbb{E} \| \nabla_{\bx} \lossFunc(\bx^{*},\bz) \|^{2} + 2 \mathbb{E}\| \nabla_{\bx} \lossFunc(\bx,\bz) - \nabla_{\bx} \lossFunc(\bx^{*},\bz) \|^{2} \nonumber \\
&\leq& 2 \mathbb{E} \| \nabla_{\bx} \lossFunc(\bx^{*},\bz) \|^{2} + 2 M^{2}  \| \bx - \bx^{*} \|^{2} \nonumber
\end{eqnarray}
Thus, we can set
\[
B = 2 M^{2}
\]
and
\[
A = 2 \mathbb{E} \| \nabla_{\bx} \lossFunc(\bx^{*},\bz) \|^{2}
\]

This suggests that given an estimate $\tilde{M}_{n}$ for $M$, we set
\[
\tilde{B}_{n} = 2 \tilde{M}_{n}^{2}
\]
Then by Jensen's inequality, we have
\begin{eqnarray}
\mathbb{E}[\tilde{B}_{n} \;|\; \mathcal{K}_{n-1}] &=&  2 \mathbb{E}[\tilde{M}_{n}^{2} \;|\; \mathcal{K}_{n-1}] \nonumber \\
&\geq& 2 \left( \mathbb{E}[\tilde{B}_{n} \;|\; \mathcal{K}_{n-1}] \right)^{2}  \nonumber \\
&\geq& 2 M^{2} \nonumber \\
&=& B \nonumber
\end{eqnarray}

\begin{lem}
\label{estRho:paramEst:gradB}
Choose $t_{n}$ such that for all $C > 0$ it holds that
\[
\sum_{n=1}^{\infty} e^{-Cnt_{n}^{2}} < + \infty
\]
Then for all $n$ large enough it holds that
\[
\hat{B}_{n} + t_{n} \geq B
\]
almost surely.
\end{lem}
\begin{proof}
By identical reasoning for the strong convexity and Lipschitz continuous gradients, it holds that
\[
\mathbb{P}\left\{ \hat{B}_{n} < B - t_{n} \right\} \leq \exp\left\{ - \frac{n t_{n}^{2}}{2 \sigma_{B}^{2}} \right\}
\]
Since we have
\[
\sum_{n=1}^{\infty} \exp\left\{ - \frac{n t_{n}^{2}}{2 \sigma_{B}^{2}} \right\} < + \infty
\]
for all $n$ large enough it holds that
\[
\hat{B}_{n} + t_{n} \geq B
\]
almost surely.
\end{proof}

To estimate $A$, consider using a point $\bx$ to approximate $\bx^{*}$. It holds that
\begin{eqnarray}
\mathbb{E} \| \nabla_{\bx} \lossFunc(\bx^{*},\bz) \|^{2} &=& \mathbb{E} \| \nabla_{\bx} \lossFunc(\bx,\bz) + \left( \nabla_{\bx} \lossFunc(\bx^{*},\bz) - \nabla_{\bx} \lossFunc(\bx,\bz) \right) \|^{2} \nonumber \\
&\leq&  2 \mathbb{E} \| \nabla_{\bx} \lossFunc(\bx,\bz) \|^{2} + 2 \mathbb{E} \| \nabla_{\bx} \lossFunc(\bx^{*},\bz) - \nabla_{\bx} \lossFunc(\bx,\bz) \|^{2} \nonumber \\
&\leq&  2 \mathbb{E} \| \nabla_{\bx} \lossFunc(\bx,\bz) \|^{2} + 2 M^{2} \mathbb{E} \| \bx - \bx^{*} \|^{2} \nonumber \\
&\leq&  2 \mathbb{E} \| \nabla_{\bx} \lossFunc(\bx,\bz) \|^{2} + 2 \left(\frac{M}{m} \right)^{2} \| \nabla f(\bx)\|^{2} \nonumber \\
&\leq&  2 \mathbb{E} \| \nabla_{\bx} \lossFunc(\bx,\bz) \|^{2} + 2 \left(\frac{M}{m} \right)^{2} \| \nabla f(\bx)\|^{2} \nonumber
\end{eqnarray}
This suggests the estimate
\[
\tilde{A}_{n}(\bx) = \frac{2}{\numIter_{n}} \sum_{\iterIndex=1}^{\numIter_{n}}  \| \nabla_{\bx} \lossFunc(\bx,\bz_{n}(\iterIndex)) \|^{2} + 4 \left( \frac{\tilde{M}_{n-1} + t_{n-1}}{\tilde{m}_{n-1} - t_{n-1}} \right)^{2} \bigg\| \frac{1}{\numIter_{n}} \sum_{\iterIndex=1}^{\numIter_{n}} \nabla_{\bx} \lossFunc(\bx,\bz_{n}(\iterIndex))  \bigg\|^{2}
\]

\begin{lem}
\label{estRho:paramEst:gradAExp}
For any $\bx$ possibly random but not a function of $\{\bz_{n}(\iterIndex)\}_{\iterIndex=1}^{\numIter_{n}}$ and all $n$ large enough, it holds that
\[
\mathbb{E}[\tilde{A}_{n} \;|\; \mathcal{K}_{n-1} ] \geq A
\]
\end{lem}
\begin{proof}
For any $\bx$ possibly random but not a function of $\{\bz_{n}(\iterIndex)\}_{\iterIndex=1}^{\numIter_{n}}$, it holds that
\begin{align}
\mathbb{E}&[\tilde{A}_{n} \;|\; \mathcal{K}_{n-1} ] \nonumber \\
&= \mathbb{E}\left[ \frac{2}{\numIter_{n}} \sum_{\iterIndex=1}^{\numIter_{n}}  \| \nabla_{\bx} \lossFunc(\bx,\bz_{n}(\iterIndex)) \|^{2} +4 \left( \frac{\tilde{M}_{n-1} + t_{n-1}}{\tilde{m}_{n-1} - t_{n-1}} \right)^{2} \bigg\| \frac{1}{\numIter_{n}} \sum_{\iterIndex=1}^{\numIter_{n}} \nabla_{\bx} \lossFunc(\bx,\bz_{n}(\iterIndex))  \bigg\|^{2} \;\bigg|\; \mathcal{K}_{n-1} \right] \nonumber \\
&= \mathbb{E}\left[ \frac{2}{\numIter_{n}} \sum_{\iterIndex=1}^{\numIter_{n}}  \| \nabla_{\bx} \lossFunc(\bx,\bz_{n}(\iterIndex)) \|^{2} \;\bigg|\; \mathcal{K}_{n-1} \right] + 4 \left( \frac{\tilde{M}_{n-1} + t_{n-1}}{\tilde{m}_{n-1} - t_{n-1}} \right)^{2} \mathbb{E}\left[ \bigg\| \frac{1}{\numIter_{n}} \sum_{\iterIndex=1}^{\numIter_{n}} \nabla_{\bx} \lossFunc(\bx,\bz_{n}(\iterIndex))  \bigg\|^{2} \;\bigg|\; \mathcal{K}_{n-1} \right] \nonumber \\
&\geq 2 \mathbb{E}\|\nabla_{\bx} \lossFunc(\bx,\bz_{n}) \|^{2} +  4 \left( \frac{\tilde{M}_{n-1} + t_{n-1}}{\tilde{m}_{n-1} - t_{n-1}} \right)^{2} \| \nabla f_{n} (\bx) \|^{2} \nonumber
\end{align}
The last inequality uses Jensen's inequality. Then by our prior analysis, almost surely for all $n$ sufficiently large it holds that
\[
\frac{\tilde{M}_{n-1} + t_{n-1}}{\tilde{m}_{n-1} - t_{n-1}} \geq \frac{M}{m}
\]
and so for all $n$ sufficiently large
\begin{eqnarray}
\mathbb{E}[\tilde{A}_{n} \;|\; \mathcal{K}_{n-1} ] &\geq& 2 \mathbb{E}\|\nabla_{\bx} \lossFunc(\bx,\bz_{n}) \|^{2} +  4 \left( \frac{M}{m} \right)^{2} \| \nabla f_{n} (\bx) \|^{2} \nonumber \\
&=& 2 \mathbb{E}\| \nabla_{\bx} \lossFunc(\bx_{n}^{*},\bz_{n}) \|^{2} \nonumber \\
&=& A \nonumber
\end{eqnarray}
Therefore, for all $n$ sufficiently large (dependent on estimation of $m$ and $M$), it holds that
\[
\mathbb{E}[\tilde{A}_{n} \;|\; \mathcal{K}_{n-1} ] \geq A
\]
\end{proof}

\noindent \emph{Combining Estimates for $A$:} In practice, we use $\tilde{A}_{n}(\bx_{n})$, which complicates the analysis due to the fact that $\bx_{n}$ is computed using the same samples $\{\bz_{n}(\iterIndex)\}_{\iterIndex = 1}^{\numIter_{n}}$. 
\begin{lem}
\label{estRho:paramEst:gradA}
Choose $t_{n}$ such that for all $C > 0$ it holds that
\[
\sum_{n=1}^{\infty} e^{-Cnt_{n}^{2}} < + \infty
\]
Then for all $n$ large enough it holds that
\[
\hat{A}_{n} + t_{n} \geq A
\]
almost surely.
\end{lem}
\begin{proof}
Consider the following three estimates of $A$ all computed with knowledge of $m$ and $M$ and $\tilde{\bx}_{n}$ as in Lemma~\ref{estRho:directEst:combineOneLemma}:
\begin{eqnarray}
\tilde{A}_{i}^{(2)} &=& \frac{2}{\numIter_{i}} \sum_{\iterIndex=1}^{\numIter_{i}} \| \nabla_{\bx} \lossFunc(\bx_{i},\bz_{i}(\iterIndex))\|^{2}  + 4 \left( \frac{M}{m} \right)^{2} \bigg\| \frac{1}{\numIter_{i}} \sum_{\iterIndex=1}^{\numIter_{i}} \nabla_{\bx} \lossFunc(\bx_{i},\bz_{i}(\iterIndex)) \bigg\|^{2} \nonumber \\
\tilde{A}_{i}^{(3)} &=& \frac{2}{\numIter_{i}} \sum_{\iterIndex=1}^{\numIter_{i}} \| \nabla_{\bx} \lossFunc(\tilde{\bx}_{i},\bz_{i}(\iterIndex))\|^{2}  + 4 \left( \frac{M}{m} \right)^{2} \bigg\| \frac{1}{\numIter_{i}} \sum_{\iterIndex=1}^{\numIter_{i}} \nabla_{\bx} \lossFunc(\tilde{\bx}_{i},\bz_{i}(\iterIndex)) \bigg\|^{2} \nonumber \\
\tilde{A}_{i}^{(4)} &=& 2 \mathbb{E}\| \nabla_{\bx} \lossFunc(\tilde{\bx}_{i},\bz_{i}) \|^{2} + 4 \left( \frac{M}{m} \right)^{2} \| \nabla f_{i}(\tilde{\bx}_{i}) \|^{2} \nonumber
\end{eqnarray}
Define the averaged estimates
\begin{eqnarray}
\hat{A}_{n}^{(2)} &=& \frac{1}{n} \sum_{i=1}^{n} \tilde{A}_{i}^{(2)} \nonumber \\
\hat{A}_{n}^{(3)} &=& \frac{1}{n} \sum_{i=1}^{n} \tilde{A}_{i}^{(3)} \nonumber \\
\hat{A}_{n}^{(4)} &=& \frac{1}{n} \sum_{i=1}^{n} \tilde{A}_{i}^{(4)} \nonumber
\end{eqnarray}
We always have
\[
\tilde{A}_{i}^{(4)} \geq A
\]
so
\[
\hat{A}_{n}^{(4)} \geq A
\]
First, we show that $\hat{A}_{n}^{(2)}$ is close to $A_{n}^{(3)}$. We have
\begin{align}
&|\tilde{A}_{i}^{(2)} - \tilde{A}_{i}^{(3)}| \nonumber \\
&\;\;\; \leq 2 \bigg| \frac{1}{\numIter_{i}} \sum_{\iterIndex = 1}^{\numIter_{i}} \left( \| \nabla_{\bx} \lossFunc(\bx_{i},\bz_{i}(\iterIndex))\|^{2} - \| \nabla_{\bx} \lossFunc(\tilde{\bx}_{i},\bz_{i}(\iterIndex))\|^{2}  \right) \bigg| \nonumber \\
&\;\;\;\;\;\;\;\;\;\;\; + 4 \left( \frac{M}{m} \right)^{2} \bigg| \bigg\| \frac{1}{\numIter_{i}} \sum_{\iterIndex=1}^{\numIter_{i}} \nabla_{\bx} \lossFunc(\bx_{i},\bz_{i}(\iterIndex)) \bigg\|^{2} - \bigg\| \frac{1}{\numIter_{i}} \sum_{\iterIndex=1}^{\numIter_{i}} \nabla_{\bx} \lossFunc(\tilde{\bx}_{i},\bz_{i}(\iterIndex)) \bigg\|^{2} \bigg| \nonumber \\
&\;\;\; \leq 4G \frac{1}{\numIter_{i}} \sum_{\iterIndex = 1}^{\numIter_{i}} \| \nabla_{\bx} \lossFunc(\bx_{i},\bz_{i}(\iterIndex)) - \nabla_{\bx} \lossFunc(\tilde{\bx}_{i},\bz_{i}(\iterIndex))\|  + 8 G \left( \frac{M}{m} \right)^{2} \bigg\|\frac{1}{\numIter_{i}} \sum_{\iterIndex=1}^{\numIter_{i}} \left(  \nabla_{\bx} \lossFunc(\bx_{i},\bz_{i}(\iterIndex)) - \nabla_{\bx} \lossFunc(\tilde{\bx}_{i},\bz_{i}(\iterIndex)) \right) \bigg\|^{2} \nonumber \\
& \;\;\; \leq \left( 4 + 8 \left( \frac{M}{m} \right)^{2} \right) G M \| \bx_{i} - \tilde{\bx}_{i} \| \nonumber
\end{align}
yielding
\[
| \hat{A}_{n}^{(2)} - \hat{A}_{n}^{(3)}| \leq  \left( 4 + 8 \left( \frac{M}{m} \right)^{2} \right) G M \left( \frac{1}{n} \sum_{i=1}^{n} \| \bx_{i} - \tilde{\bx}_{i} \| \right)
\]
Second, we have
\begin{align}
&| \hat{A}_{n}^{(3)} - \hat{A}_{n}^{(4)}| \nonumber \\
& \;\;\; \leq \bigg| \frac{1}{n} \sum_{i=1}^{n} \left( \frac{2}{\numIter_{i}} \sum_{\iterIndex=1}^{\numIter_{i}} \left( \| \nabla_{\bx} \lossFunc(\tilde{\bx}_{i},\bz_{i}(\iterIndex))  \|^{2} - \mathbb{E}\left[ \| \nabla_{\bx} \lossFunc(\tilde{\bx}_{i},\bz_{i})  \|^{2} \;|\; \mathcal{F}_{n-1} \right]  \right)  \right)  \bigg| \nonumber \\
& \;\;\;\;\;\;\;\;\;\;\; + 8 \left( \frac{M}{m} \right)^{2} G \frac{1}{n} \sum_{i=1}^{n} \bigg\| \frac{1}{\numIter_{i}} \sum_{\iterIndex = 1}^{\numIter_{i}} \left( \nabla_{\bx}\lossFunc(\tilde{\bx}_{i},\bz_{i}(\iterIndex)) - \nabla f_{i}(\tilde{\bx}_{i}) \right)  \bigg\| \nonumber
\end{align}
Combining both inequalities, we know that
\begin{align}
&|\hat{A}_{n}^{(2)} - \hat{A}_{n}^{(4)}| \nonumber \\
& \;\;\; \leq \left( 4 + 8 \left( \frac{M}{m} \right)^{2} \right) G M \left( \frac{1}{n} \sum_{i=1}^{n} \| \bx_{i} - \tilde{\bx}_{i} \| \right) \nonumber \\
& \;\;\;\;\;\;\;\;\;\;\; + \bigg| \frac{1}{n} \sum_{i=1}^{n} \left( \frac{2}{\numIter_{i}} \sum_{\iterIndex=1}^{\numIter_{i}} \left( \| \nabla_{\bx} \lossFunc(\tilde{\bx}_{i},\bz_{i}(\iterIndex))  \|^{2} - \mathbb{E}\left[ \| \nabla_{\bx} \lossFunc(\tilde{\bx}_{i},\bz_{i})  \|^{2} \;|\; \mathcal{F}_{n-1} \right]  \right)  \right)  \bigg| \nonumber \\
& \;\;\;\;\;\;\;\;\;\;\; + 8 \left( \frac{M}{m} \right)^{2} G \frac{1}{n} \sum_{i=1}^{n} \bigg\| \frac{1}{\numIter_{i}} \sum_{\iterIndex = 1}^{\numIter_{i}} \left( \nabla_{\bx}\lossFunc(\tilde{\bx}_{i},\bz_{i}(\iterIndex)) - \nabla f_{i}(\tilde{\bx}_{i}) \right)  \bigg\| \nonumber 
\end{align}
The first and third terms in this bound can be controlled by the analysis of the direct estimate and the second term by Lemma~\eqref{subgauss:subgaussDepLem}. This shows that
\begin{align}
\mathbb{P}&\left\{ \hat{A}_{n}^{(2)} < A - \frac{1}{n} \sum_{i=1}^{n} \frac{C_{i}}{\sqrt{\numIter_{i}}} - t_{n} \right\} \nonumber \\
& \;\; \leq \mathbb{P}\left\{ \hat{A}_{n}^{(2)} < \hat{A}_{n}^{(4)} - \frac{1}{n} \sum_{i=1}^{n} \frac{C_{i}}{\sqrt{\numIter_{i}}} - t_{n} \right\} \nonumber \\
& \;\; \leq \mathbb{P}\left\{ |\hat{A}_{n}^{(2)} - \hat{A}_{n}^{(4)}| > \frac{1}{n} \sum_{i=1}^{n} \frac{C_{i}}{\sqrt{\numIter_{i}}}  t_{n} \right\} \nonumber \\
& \;\; \leq 2\exp\left\{ -\frac{n t_{n}^{2}}{2 \sigma^{2}_{A2}}  \right\} \nonumber
\end{align}
Since
\[
\sum_{n=1}^{\infty} \mathbb{P}\left\{ \hat{A}_{n}^{(2)} < A - \frac{1}{n} \sum_{i=1}^{n} \frac{C_{i}}{\sqrt{\numIter_{i}}} - t_{n} \right\} \leq \sum_{n=1}^{\infty} C\exp\left\{ -\frac{n t_{n}^{2}}{2 \sigma^{2}_{A2}}  \right\} < + \infty \nonumber 
\]
almost surely for all $n$ large enough, it holds that
\[
\hat{A}_{n}^{(2)} + \frac{1}{n} \sum_{i=1}^{n} \frac{C_{i}}{\sqrt{\numIter_{i}}} + t_{n} \geq A
\]
In addition, we have
\[
\hat{A}_{n}^{(2)} + \frac{1}{n} \sum_{i=1}^{n} \frac{C_{i}}{\sqrt{\numIter_{i}}} + 2t_{n} \geq A
\]
There exists a random variable $\tilde{N}$ such that
\[
n \geq \tilde{N} \;\;\Rightarrow \;\; \frac{M_{n}+t_{n}}{m_{n}-t_{n}} \geq \frac{M}{m}
\]
Then for $n \geq \tilde{N}$, it holds that
\begin{align}
&\hat{A}_{n} - \hat{A}_{n}^{(2)} \nonumber \\
&\;\;\; = \frac{4}{n} \sum_{i=1}^{n} \left[ \left( \frac{\hat{M}_{i-1}+t_{i-1}}{\hat{m}_{i-1}-t_{i-1}} \right)^{2} - \left( \frac{M}{m} \right)^{2} \right] \bigg\| \frac{1}{\numIter_{i}} \sum_{\iterIndex=1}^{\numIter_{i}} \nabla_{\bx} \lossFunc(\bx_{i},\bz_{i}(\iterIndex)) \bigg\|^{2} \nonumber \\
&\;\;\; \geq \frac{4}{n} \sum_{i=1}^{\tilde{N}-1} \left[ \left( \frac{\hat{M}_{i-1}+t_{i-1}}{\hat{m}_{i-1}-t_{i-1}} \right)^{2} - \left( \frac{M}{m} \right)^{2}  \right] \bigg\| \frac{1}{\numIter_{i}} \sum_{\iterIndex=1}^{\numIter_{i}} \nabla_{\bx} \lossFunc(\bx_{i},\bz_{i}(\iterIndex)) \bigg\|^{2} \nonumber
\end{align}
Since our choice of $t_{n}$ can decay only as fast as $C / \sqrt{n}$, it follows that
\[
\frac{4}{n} \sum_{i=1}^{\tilde{N}-1} \left[ \left( \frac{\hat{M}_{i-1}+t_{i-1}}{\hat{m}_{i-1}-t_{i-1}} \right)^{2} - \left( \frac{M}{m} \right)^{2}  \right] \bigg\| \frac{1}{\numIter_{i}} \sum_{\iterIndex=1}^{\numIter_{i}} \nabla_{\bx} \lossFunc(\bx_{i},\bz_{i}(\iterIndex)) \bigg\|^{2} - t_{n} < 0
\]
for all $n$ large enough. This implies that
\begin{align}
\hat{A}_{n} & + \frac{1}{n} \sum_{i=1}^{n} \frac{C_{i}}{\sqrt{\numIter_{i}}} + t_{n} \nonumber \\
&\geq \hat{A}_{n} - \left( \frac{4}{n} \sum_{i=1}^{\tilde{N}-1} \left[ \left( \frac{M}{m} \right)^{2} - \left( \frac{\hat{M}_{i-1}+t_{i-1}}{\hat{m}_{i-1}+t_{i-1}} \right)^{2} \right] \bigg\| \frac{1}{\numIter_{i}} \sum_{\iterIndex=1}^{\numIter_{i}} \nabla_{\bx} \lossFunc(\bx_{i},\bz_{i}(\iterIndex)) \bigg\|^{2} - t_{n} \right) + \frac{1}{n} \sum_{i=1}^{n} \frac{C_{i}}{\sqrt{\numIter_{i}}} +  t_{n} \nonumber \\
& \geq \hat{A}_{n}^{(2)} + \frac{1}{n} \sum_{i=1}^{n} \frac{C_{i}}{\sqrt{\numIter_{i}}} + 2 t_{n} \nonumber \\
& \geq A \nonumber
\end{align}
for $n$ large enough.
\end{proof}

Using these estimates, we have constructed estimates $\hat{\psi}_{n}$ such that for all $n$ large enough it holds that
\[
\hat{\psi}_{n} + C_{n} + t_{n} \bm{1} \geq \psi^{*}
\]
for appropriate constants $C_{n}$ almost surely. Therefore, by assumption for all $n$ large enough it holds that
\[
b(d_{0},\numIter,\psi^{*}) \leq b(d_{0},\numIter,\hat{\psi}_{n} + t_{n})
\]

\subsubsection{Effect on $\rho$ Estimation}
Our analysis of estimating $\rho$ assumes that we know the parameters of the function and in particular the strong convexity parameter $m$. We now argue that the effect of using estimated parameters instead is minimal. This happens because we know that for all $n$ large enough it holds that
\[
\hat{\psi}_{n} \geq \psi^{*}
\]
almost surely. 

\begin{lem}
We want to estimate a non-negative parameter $\phi^{*}$ by producing a sequence of estimates $\phi_{i}$ for all $i \geq 1$ and averaging to produce
\[
\hat{\phi}_{n} = \frac{1}{n} \sum_{i=1}^{n} \phi_{i}
\]
where the estimates $\phi_{i}$ are dependent on an auxiliary sequence $\psi_{i}$ in the sense that $\phi_{i}(\psi_{i})$. Suppose that the following conditions hold:
\begin{enumerate}
\item Suppose that there exists a random variable $\tilde{N}$ such that $n \geq \tilde{N}$ implies that $\hat{\psi}_{n} \geq \psi^{*}$
\item $\mathbb{E}[\phi_{i}(\psi^{*})] \geq \phi^{*}$
\end{enumerate}
Then it follows that
\[
\liminf_{n \to \infty} \mathbb{E}\left[ \frac{1}{n} \sum_{i=1}^{n} \phi_{i} \right] \geq \phi^{*}
\]
\end{lem}
\begin{proof}
It holds that
\begin{eqnarray}
\frac{1}{n} \sum_{i=1}^{n} \phi_{i} &=& \frac{1}{n} \sum_{i=1}^{\tilde{N}-1} \phi_{i}(\psi_{i}) + \frac{1}{n}\sum_{i=\tilde{N}}^{n} \phi_{i}(\psi_{i}) \nonumber \\
\label{estRho:paramEst:keyIneq} &\geq& \frac{1}{n} \sum_{i=1}^{\tilde{N}-1} \phi_{i}(\psi_{i}) + \frac{1}{n}\sum_{i=\tilde{N}}^{n} \phi_{i}(\psi_{i}^{*})
\end{eqnarray}
Therefore, it follows that
\begin{eqnarray}
\liminf_{n \to \infty} \mathbb{E}\left[ \frac{1}{n} \sum_{i=1}^{n} \phi_{i} \right] &\geq& \liminf_{n \to \infty} \mathbb{E}\left[ \frac{1}{n}\sum_{i=\tilde{N}}^{n} \phi_{i}(\psi_{i}^{*}) \right] \nonumber \\
&\geq& \phi^{*} \nonumber
\end{eqnarray}
\end{proof}
We can extend all the concentration inequalities for estimating $\rho$ as well by extending the inequality in \eqref{estRho:paramEst:keyIneq} to yield
\begin{eqnarray}
\frac{1}{n} \sum_{i=1}^{n} \phi_{i} &=& \frac{1}{n} \sum_{i=1}^{\tilde{N}-1} \phi_{i}(\psi_{i}) + \frac{1}{n}\sum_{i=\tilde{N}}^{n} \phi_{i}(\psi_{i}) \nonumber \\
&\geq& \frac{1}{n} \sum_{i=1}^{\tilde{N}-1} \phi_{i}(\psi_{i}) + \frac{1}{n}\sum_{i=\tilde{N}}^{n} \phi_{i}(\psi_{i}^{*}) \nonumber \\
&\geq& \frac{1}{n} \sum_{i=1}^{\tilde{N}-1} \left( \phi_{i}(\psi_{i}) - \phi_{i}(\psi^{*}) \right) + \frac{1}{n}\sum_{i=1}^{n} \phi_{i}(\psi_{i}^{*}) \nonumber \\
&=& \frac{1}{n}\sum_{i=1}^{n} \phi_{i}(\psi_{i}^{*}) + o(1) \nonumber
\end{eqnarray}
Before, we have analyzed
\[
 \frac{1}{n}\sum_{i=1}^{n} \phi_{i}(\psi_{i}^{*})
\]
so for large enough $n$, we recover previous results, since the $o(1)$ term goes to $0$.

\section{Adaptive Sequential Optimization With $\rho$ Unknown}
\label{withRhoUnknown}
We now examine the case with $\rho$ unknown. We extend the work of Section~\ref{withRhoKnown} using the estimates of $\rho$ in Section~\ref{estRho}. Our analysis depends on the following crucial assumption:
\begin{description}
\item[C.1 \label{probState:assumpC1}] For appropriate sequences $\{t_{n}\}$, for all $n$ sufficiently large it holds that $\hat{\rho}_{n} + t_{n} \geq \rho$ almost surely.
\item[C.2 \label{probState:assumpC2}] $b(d_{0},\numIter_{n})$ factors as $b(d_{0},\numIter_{n}) = \alpha(\numIter_{n}) d_{0} + \beta(\numIter_{n})$
\end{description}
We have demonstrated that assumption~\ref{probState:assumpC1} that holds for the direct and IPM estimates of $\rho$ under \eqref{probState:slowChangeConstDef} and \eqref{probState:slowChangeDef}. Note that whether we assume \eqref{probState:slowChangeConstDef} or \eqref{probState:slowChangeDef} does not matter for analysis.

\subsection{General Condition on $\numIter_{n}$}

We start with a general result showing that for any choice of $\numIter_{n}$ such that $\numIter_{n} \geq \numIter^{*}$ for all n large enough the excess
risk is controlled in the sense that
\[
\limsup_{n \to \infty} \left( \mathbb{E}[f_{n}(\bx_{n})] - f_{n}(\bx_{n}^{*}) \right) \leq \epsilon
\]
We then apply this result to two different selection rules for Kn.

Consider the function
\[
\phi_{\numIter}(v) = \alpha(\numIter) \left( \sqrt{\frac{2}{m}v} + \rho \right)^{2} + \beta(\numIter)
\]
derived from assumption C.2. Note that as a function of $v$, $\phi_{\numIter}(v)$ is clearly increasing and strictly concave. First, suppose that we select $\numIter^{*}$ defined in \eqref{withRhoKnown:KChoice}. Then by definition it holds that
\[
\phi_{\numIter^{*}}(\epsilon) \leq \epsilon
\]
We study fixed points of the function $\phi_{\numIter^{*}}(v)$:
\begin{lem}
\label{withRhoUnknown:fixedPointBasic}
The function $\phi_{\numIter^{*}}(v)$ has a unique positive fixed point $\bar{v}$ with
\begin{enumerate}
\item $\bar{v} = \phi_{\numIter^{*}}(\bar{v}) \leq \epsilon$
\item $\phi'_{\numIter^{*}}(\bar{v}) < 1$
\end{enumerate}
\end{lem}
\begin{proof}
We have
\[
\phi_{\numIter^{*}}(0) = \alpha(\numIter^{*})\rho^{2} + \beta(\numIter^{*}) >0
\]
Since 
\[
\lim_{v \to 0} \phi_{\numIter^{*}}(v) = \phi_{\numIter^{*}}(0)
\]
and $\phi_{\numIter^{*}}(0) > 0$, there exists a positive $a$ sufficiently small that
\[
\phi_{\numIter^{*}}(a) > a
\]
Next, expanding $\phi_{\numIter}(v)$ yields
\[
\phi_{\numIter}(v) =  \frac{2}{m} \alpha(\numIter) v + 2 \alpha(\numIter) \rho \sqrt{\frac{2}{m}} \sqrt{v} + \alpha(\numIter) \rho^{2}  + \beta(\numIter)
\]
Since $\phi_{\numIter^{*}}(\epsilon) \leq \epsilon$, we obviously must have $\frac{2}{m} \alpha(\numIter^{*})  \leq 1$. Suppose that
\[
\frac{2}{m} \alpha(\numIter^{*})  = 1
\]
Then it holds that
\[
\phi_{\numIter^{*}}(\epsilon) =  \epsilon + \sqrt{2m} \rho \sqrt{\epsilon} + \frac{m}{2} \rho^{2}  + \beta(\numIter) > \epsilon
\]
This is a contradiction, so it holds that
\[
\frac{2}{m} \alpha(\numIter^{*}) < 1
\]
It is thus readily apparent that
\[
v - \phi_{\numIter^{*}}(v) \to \infty
\]
as $v \to \infty$. Therefore, there exists a point $b > a$ such that
\[
\phi_{\numIter^{*}}(b) < b
\]
It is easy to check that $\phi_{\numIter^{*}}(v)$ is increasing and strictly concave. Therefore, we can apply Theorem 3.3 from \cite{Kennan2001} to conclude that there exists a unique, positive fixed point $\bar{v}$ of $\phi_{\numIter^{*}}(v)$.

Next, suppose that $\phi'_{\numIter^{*}}(\bar{v}) > 1$. Then by Taylor's Theorem for $v > \bar{v}$ sufficiently close to $\bar{v}$, we have
\[
\phi_{\numIter^{*}}(v) > v
\]
However, we know that as $v \to \infty$, it holds that $v - \phi_{\numIter^{*}}(v) \to \infty$. By the Intermediate Value Theorem, this implies that there is another fixed point on $[v,\infty)$. This is a contradiction, since $\bar{v}$ is the unique, positive fixed point. Therefore, it holds that $\phi'_{\numIter^{*}}(\bar{v}) \leq 1$. Now, suppose that $\phi'_{\numIter^{*}}(\bar{v}) = 1$. Since $\phi_{\numIter^{*}}(v)$ is strictly concave, its derivative is decreasing \cite{Boyd2004}. Therefore, on $[0,\bar{v})$, it holds that
\[
\phi'_{\numIter^{*}}(v) > 1
\]
This implies that
\begin{eqnarray}
\phi_{\numIter^{*}}(\bar{v}) &=& \phi_{\numIter^{*}}(0) + \int_{0}^{\bar{v}} \phi'_{\numIter^{*}}(v) dx \nonumber \\
&\geq& \phi_{\numIter^{*}}(0) + \bar{v} \nonumber \\
&>& \bar{v} \nonumber
\end{eqnarray}
This is a contradiction, so it must be that $\phi'_{\numIter^{*}}(\bar{v}) < 1$.
\end{proof}
As a simple consequence of the concavity of $\phi_{\numIter^{*}}(v)$, we can study a fixed point iteration involving $\phi_{\numIter}(v)$. Define the $n$-fold composition mapping
\[
\phi^{(n)}_{\numIter}(v) \triangleq \left( \phi_{\numIter} \circ \cdots \circ \phi_{\numIter} \right)(v)
\]
\begin{lem}
\label{withRhoUnknown:fixedPointIter}
For any $v > 0$, it holds that
\[
\lim_{n \to \infty} \phi^{(n)}_{\numIter^{*}}(v) = \bar{v}
\]
\end{lem}
\begin{proof}
Following \cite{Granas2003}, for any fixed point $\bar{v}$, it holds that 
\[
|\phi_{\numIter^{*}}(v) - \bar{v}| \leq \phi'_{\numIter^{*}}(\bar{v})|v - \bar{v}|
\]
Therefore, applying the fixed point property repeatedly yields
\[
|\phi^{(n)}_{\numIter^{*}}(v) - \bar{v}| \leq (\phi'_{\numIter^{*}}(\bar{v}))^{n}|v - \bar{v}|
\]
By Lemma~\ref{withRhoUnknown:fixedPointBasic}, it holds that
\[
\phi'_{\numIter^{*}}(\bar{v}) < 1
\]
and so the result follows.
\end{proof}

Now, we show that we appropriately control the \meangap{} when we estimate $\rho$. The extension of this argument to the case when we also estimate function parameters $\psi$ is straightforward. If we have
\[
p(\{\bz_{n}(\iterIndex)\}_{\iterIndex=1}^{\numIter_{n}} \;|\; \bx_{n-1},\numIter_{n}) = \prod_{\iterIndex=1}^{\numIter_{n}} p_{n}(\bz_{n}(\iterIndex))
\]
then
\[
\mathbb{E}\left[ f_{n}(\bx_{n}) \;|\; \bx_{n-1},\numIter_{n} \right] - f_{n}(\bx_{n}^{*}) \leq b\left( \left( \sqrt{\frac{2}{m} \left( f_{n-1}(\bx_{n-1}) - f_{n-1}(\bx_{n-1}^{*}) \right)} + \rho \right)^{2} , \numIter_{n} \right)
\]
Therefore, it holds that
\[
\mathbb{E}\left[ f_{n}(\bx_{n}) \right] - f_{n}(\bx_{n}^{*}) \leq \mathbb{E}\left[  b\left( \left( \sqrt{\frac{2}{m} \left( f_{n-1}(\bx_{n-1}) - f_{n-1}(\bx_{n-1}^{*}) \right)} + \rho \right)^{2} , \numIter_{n} \right) \right]
\]
Suppose that we set
\[
\mathcal{K}_{\infty} = \sigma\left( \{\numIter_{n}\}_{n=1}^{\infty} \cup \{\hat{\rho}_{n}\}_{n=2}^{\infty} \right)
\]
This sigma algebra contains all the information about $\{\hat{\rho}_{n}\}$ and thus $\{\numIter_{n}\}$. Then, we do not have
\[
p(\{\bz_{n}(\iterIndex)\}_{\iterIndex=1}^{\numIter_{n}} \;|\; \mathcal{K}_{\infty}) = \prod_{\iterIndex=1}^{\numIter_{n}} p_{n}(\bz_{n}(\iterIndex))
\]
since $\numIter_{n+1},\numIter_{n+2},\ldots$ are a function of $\{\numIter_{n}\}_{\iterIndex=1}^{\numIter_{n}}$. We do not even have
\[
\mathbb{E}\left[ f_{n}(\bx_{n}) \;|\; \mathcal{K}_{\infty} \right] - f_{n}(\bx_{n}^{*}) \leq b\left( \left( \sqrt{\frac{2}{m} \left( f_{n-1}(\bx_{n-1}) - f_{n-1}(\bx_{n-1}^{*}) \right)} + \rho \right)^{2}  , \numIter_{n} \right)
\]
However, we would expect that this is not too far from true. Conceptually, we consider running our approach twice on independent samples. The first run determines the required number of samples $\{\numIter_{n}\}_{n=1}^{\infty}$. We then run our process for a second run with these fixed choices of $\{\numIter_{n}\}_{n=1}^{\infty}$and independent samples as in \figurename{}~\ref{withRhoKnown:secondRunFig}. For the second run, it is true that
\[
p(\{\bz_{n}^{(2)}(\iterIndex)\}_{\iterIndex=1}^{\numIter_{n}} \;|\; \mathcal{K}_{\infty}) = \prod_{\iterIndex=1}^{\numIter_{n}} p_{n}(\bz_{n}^{(2)}(\iterIndex))
\]
and
\[
\mathbb{E}\left[ f_{n}(\bx_{n}^{(2)}) \;|\; \mathcal{K}_{\infty} \right] - f_{n}(\bx_{n}^{*}) \leq b\left( \left( \sqrt{\frac{2}{m} \left( f_{n-1}(\bx_{n-1}^{(2)}) - f_{n-1}(\bx_{n-1}^{*}) \right)} + \rho \right)^{2}  , \numIter_{n} \right)
\]
In practice, we do not need to run our process twice. This is only a proof technique. Now, for the second run the recursion
\begin{equation}
\label{withRhoUnknown:secondRunRecursion}
\epsilon_{n}^{(2)} = b\left( \left(\sqrt{\frac{2}{m} \epsilon_{n-1}^{(2)}} + \rho \right)^{2} , \numIter_{n} \right) \;\;\; \forall n \geq 3
\end{equation}
with $\epsilon_{1}$ and $\epsilon_{2}$ from Assumption~\ref{probState:assump4} bounds the \meangap{} of the second run
\[
\mathbb{E}[f_{n}(\bx_{n}^{(2)}) \;|\; \mathcal{K}_{\infty}] - f_{n}(\bx_{n}^{*}) \leq \epsilon_{n}^{(2)}
\]
Then it follows that
\[
\mathbb{E}[f_{n}(\bx_{n}^{(2)})] - f_{n}(\bx_{n}^{*}) \leq \mathbb{E}[\epsilon_{n}^{(2)}]
\]

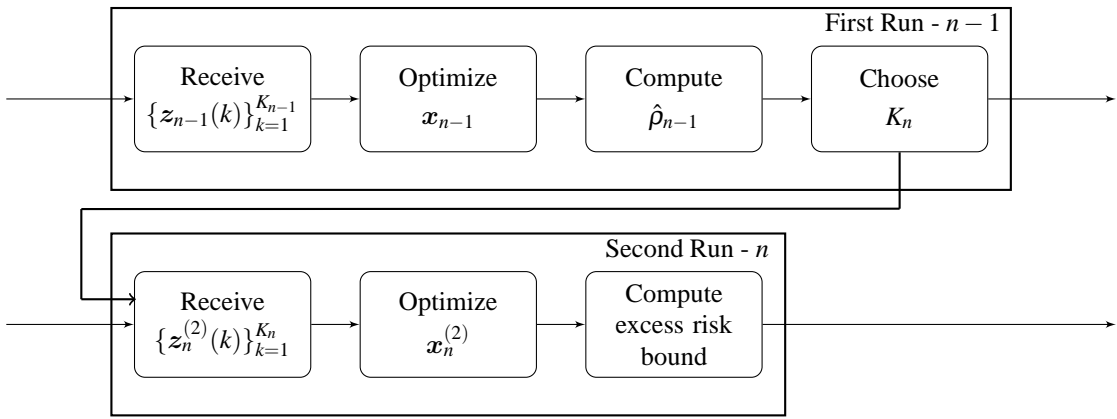
\begin{figure}[!ht]

\begin{tikzpicture}[node distance = 3cm, auto]
    
    
    \node (startOne) {};
    
    \node [block,right of=startOne] (samplesOne) {Receive\\\vspace{0.8mm} $\{\bz_{n-1}(\iterIndex)\}_{\iterIndex=1}^{\numIter_{n-1}}$};
   
    \node [block,right of=samplesOne] (optXn) {Optimize\\\vspace{0.8mm} $\bx_{n-1}$};
    
    \node [block,right of=optXn] (estRn) {Compute\\\vspace{0.8mm} $\hat{\rho}_{n-1}$};
    
    \node [block,right of=estRn] (chooseKn) {Choose\\\vspace{0.8mm} $\numIter_{n}$};
    
    \node [right of=chooseKn] (endOne) {};
    
    
    \node [below of=startOne] (startTwo) {};
    
    \node [block,right of=startTwo] (samplesTwo) {Receive\\\vspace{0.8mm} $\{\bz_{n}^{(2)}(\iterIndex)\}_{\iterIndex=1}^{\numIter_{n}}$};
   
    \node [block,right of=samplesTwo] (optXn2) {Optimize\\\vspace{0.8mm} $\bx_{n}^{(2)}$};
    
    \node [block,right of=optXn2] (estMG) {Compute \\ \meangap{} \\ bound};
   
    \node [below of=endOne] (endTwo) {};


	\draw[thick] ($(samplesOne.north west)+(-0.3,0.5)$)  rectangle ($(chooseKn.south east)+(0.3,-0.5)$);
	
	\draw[thick] ($(samplesTwo.north west)+(-0.3,0.5)$)  rectangle ($(estMG.south east)+(0.3,-0.5)$);
	
	\node at ($(chooseKn.north east)+(-1.0,+0.3)$) {First Run - $n-1$};
	\node at ($(estMG.north east)+(-1.0,+0.3)$) {Second Run - $n$};
	
	\draw[thick] (chooseKn.south) -- ++(0,-0.75) coordinate (ptOne);
	\draw[thick] (ptOne) -- ($(samplesTwo.north west)+(-0.7,0.825)$) coordinate (ptTwo);
	\draw[thick] (ptTwo) -- ++(0,-1.2) coordinate (ptThree);
	\draw[thick,->] (ptThree) -- ++(0.725,0);

    
    \path [line] (startOne) -- (samplesOne);
    \path [line] (samplesOne) -- (optXn);
    \path [line] (optXn) -- (estRn);
    \path [line] (estRn) -- (chooseKn);
    \path [line] (chooseKn) -- (endOne);
    
    \path [line] (startTwo) -- (samplesTwo);
    \path [line] (samplesTwo) -- (optXn2);
    \path [line] (optXn2) -- (estMG);
    \path [line] (estMG) -- (endTwo);
\end{tikzpicture}
\caption{Two Run Process}
\label{withRhoKnown:secondRunFig}
\end{figure}

We now argue that $\mathbb{E}[\epsilon_{n}^{(2)}]$ also bounds the \meangap{} of the first run.
\begin{lem}
For the first run, it holds that
\[
\mathbb{E}[f_{n}(\bx_{n})] - f_{n}(\bx_{n}^{*}) \leq \mathbb{E}[\epsilon^{(2)}_{n}]
\]
\end{lem}
\begin{proof}
We proceed by induction. For $n=1,2$, we know that
\[
\mathbb{E}[f_{n}(\bx_{n})] - f_{n}(\bx_{n}^{*}) \leq \mathbb{E}[\epsilon_{n}^{(2)}]
\]
by definition. Next, suppose that
\[
\mathbb{E}[f_{n-1}(\bx_{n-1})] - f_{n-1}(\bx_{n-1}^{*}) \leq \mathbb{E}[\epsilon_{n-1}^{(2)}]
\]
We have
\[
\mathbb{E}[f_{n}(\bx_{n})] - f_{n}(\bx_{n}^{*}) \leq \mathbb{E}\left[ \alpha(\numIter_{n}) \left( \sqrt{f_{n-1}(\bx_{n-1}) - f_{n-1}(\bx_{n-1}^{*})} + \rho \right)^{2} + \beta(\numIter_{n}) \right]
\]
so it holds that
\begin{align}
\mathbb{E}&[\epsilon_{n}^{(2)}] - \left( \mathbb{E}[f_{n}(\bx_{n})] - f_{n}(\bx_{n}^{*})  \right) \nonumber \\
&\geq \mathbb{E}\left[ \alpha(\numIter_{n}) \left( \sqrt{\epsilon_{n-1}^{(2)}} + \rho \right)^{2} -  \alpha(\numIter_{n}) \left( \sqrt{f_{n-1}(\bx_{n-1}) - f_{n-1}(\bx_{n-1}^{*})} + \rho \right)^{2} \right] \nonumber \\
&= \mathbb{E}\left[ \alpha(\numIter_{n}) \left( \epsilon_{n-1}^{(2)} - \left( f_{n-1}(\bx_{n-1}) - f_{n-1}(\bx_{n-1}^{*}) \right) \right) \right] \nonumber \\
& \;\;\;\;\;\;\; +  \mathbb{E}\left[ 2 \rho \alpha(\numIter_{n}) \left( \sqrt{\epsilon_{n-1}^{(2)}} - \sqrt{ f_{n-1}(\bx_{n-1}) - f_{n-1}(\bx_{n-1}^{*}) } \right) \right] \nonumber
\end{align}
By the Monotone Convergence Theorem, it holds that
\begin{align}
\mathbb{E}&\left[ \alpha(\numIter_{n}) \left( \epsilon_{n-1}^{(2)} - \left( f_{n-1}(\bx_{n-1}) - f_{n-1}(\bx_{n-1}^{*}) \right) \right) \right] \nonumber \\
&= \lim_{q \to \infty} \mathbb{E}\left[ \max\{\alpha(\numIter_{n}),1/q\} \left( \epsilon_{n-1}^{(2)} - \left( f_{n-1}(\bx_{n-1}) - f_{n-1}(\bx_{n-1}^{*}) \right) \right) \right] \nonumber \\
&\geq \liminf_{q \to \infty} \frac{1}{q} \mathbb{E}\left[ \epsilon_{n-1}^{(2)} - \left( f_{n-1}(\bx_{n-1}) - f_{n-1}(\bx_{n-1}^{*}) \right) \right] \nonumber \\
&\geq 0 \nonumber
\end{align}
where the last line follows, since by hypothesis
\[
\mathbb{E}[f_{n-1}(\bx_{n-1})] - f_{n-1}(\bx_{n-1}^{*}) \leq \mathbb{E}[\epsilon_{n-1}^{(2)}]
\]
Similarly, it holds that
\begin{align}
\mathbb{E}&\left[ 2 \rho  \alpha(\numIter_{n}) \left( \sqrt{\epsilon_{n-1}^{(2)}} - \sqrt{ f_{n-1}(\bx_{n-1}) - f_{n-1}(\bx_{n-1}^{*}) } \right) \right] \nonumber \\
&= \mathbb{E}\left[ 2 \rho \alpha(\numIter_{n}) \frac{ \epsilon_{n-1}^{(2)} - \left( f_{n-1}(\bx_{n-1}) - f_{n-1}(\bx_{n-1}^{*}) \right) }{\sqrt{\epsilon_{n-1}^{(2)}} + \sqrt{ f_{n-1}(\bx_{n-1}) - f_{n-1}(\bx_{n-1}^{*}) }}  \right] \nonumber \\
&= \lim_{q \to \infty} \mathbb{E}\left[ 2 \rho \max\{\alpha(\numIter_{n}),1/q\} \frac{ \epsilon_{n-1}^{(2)} - \left( f_{n-1}(\bx_{n-1}) - f_{n-1}(\bx_{n-1}^{*}) \right) }{\sqrt{\epsilon_{n-1}^{(2)}} + \sqrt{ f_{n-1}(\bx_{n-1}) - f_{n-1}(\bx_{n-1}^{*}) }}   \right] \nonumber \\
&\geq \limsup_{q \to \infty} \frac{2\rho }{q} \mathbb{E}\left[ \frac{ \epsilon_{n-1}^{(2)} - \left( f_{n-1}(\bx_{n-1}) - f_{n-1}(\bx_{n-1}^{*}) \right) }{\sqrt{\epsilon_{n-1}^{(2)}} + \sqrt{ f_{n-1}(\bx_{n-1}) - f_{n-1}(\bx_{n-1}^{*}) }}   \right] \nonumber \\
&\geq \limsup_{q \to \infty} \frac{2\rho }{q} \lim_{\tau \to \infty} \mathbb{E}\left[ \frac{ \epsilon_{n-1}^{(2)} - \left( f_{n-1}(\bx_{n-1}) - f_{n-1}(\bx_{n-1}^{*}) \right) }{\sqrt{\epsilon_{n-1}^{(2)}} + \sqrt{ f_{n-1}(\bx_{n-1}) - f_{n-1}(\bx_{n-1}^{*}) }} \mathbbm{1}_{\{ \sqrt{\epsilon_{n-1}^{(2)}} + \sqrt{ f_{n-1}(\bx_{n-1}) - f_{n-1}(\bx_{n-1}^{*}) } \leq \tau \}}  \right] \nonumber \\
&\geq \limsup_{q \to \infty} \frac{2\rho }{q} \limsup_{\tau \to \infty} \frac{1}{\tau} \mathbb{E}\left[ \epsilon_{n-1}^{(2)} - \left( f_{n-1}(\bx_{n-1}) - f_{n-1}(\bx_{n-1}^{*}) \right)  \right] \nonumber \\
&\geq 0 \nonumber
\end{align}
Therefore, we conclude that 
\[
\mathbb{E}[f_{n}(\bx_{n})] - f_{n}(\bx_{n}^{*}) \leq \mathbb{E}[\epsilon^{(2)}_{n}]
\]
\end{proof}

\begin{thm}
\label{withRhoUnknown:meanGapRhoKnownLemma}
Under assumptions~\ref{probState:assumpC1}-~\ref{probState:assumpC2} and with $\numIter_{n} \geq \numIter^{*}$ for all $n$ large enough almost surely with $\numIter^{*}$ from \eqref{withRhoUnknown:KnChoice}, we have  

\centering$\limsup_{n \to \infty} \left( \mathbb{E}[f_{n}(\bx_{n})] - f_{n}(\bx_{n}^{*})  \right) \leq \epsilon$
\end{thm}
\begin{proof}
Let $\bar{v}$ be the fixed point associated with $\phi_{\numIter^{*}}(v)$ from Lemma~\ref{withRhoUnknown:fixedPointBasic}. We know that
\[
\bar{v}  = \phi_{\numIter^{*}}(\bar{v}) \leq \epsilon
\]
and
\[
\phi^{(n)}_{\numIter^{*}}(v) \to \bar{v} \leq \epsilon
\]
with $\bar{v} \leq \epsilon$.  
Since we have $\numIter_{n} \geq \numIter^{*}$ for all $n$ large enough almost surely, there exists a random variable $\tilde{N}$ such that
\[
n \geq \tilde{N} \;\;\Rightarrow\;\; \numIter_{n} \geq \numIter^{*}
\]

Then we have almost surely
\begin{eqnarray}
\limsup_{n \to \infty} \epsilon_{n}^{(2)} &\leq& \limsup_{n \to \infty} (\phi_{\numIter_{n}} \circ \cdots \circ \phi_{\numIter_{\tilde{N}}})(\epsilon_{\tilde{N}-1}) \nonumber \\
&\leq& \limsup_{n \to \infty} \phi^{(n-\tilde{N}+1)}_{\numIter^{*}}(\epsilon_{\tilde{N}-1}) \nonumber \\
&=& \bar{v} \nonumber \\
&\leq& \epsilon \nonumber
\end{eqnarray}
Finally, applying Lemma~\ref{withRhoUnknown:secondRunRecursion} and Fatou's lemma yields
\begin{eqnarray}
\limsup_{n \to \infty} \left( \mathbb{E}[f_{n}(\bx_{n})] - f_{n}(\bx_{n}^{*}) \right) &\leq& \limsup_{n \to \infty} \mathbb{E}\left[ \epsilon^{(2)}_{n}  \right] \nonumber \\
&\leq& \mathbb{E}\left[ \limsup_{n \to \infty} \epsilon^{(2)}_{n}  \right] \nonumber \\
&\leq& \epsilon \nonumber
\end{eqnarray}
\end{proof}

\subsection{Update Past \MeanGap{} Bounds}
\label{withRhoUnknown:updatePast}

We first consider updating all past \meangap{} bounds as we go. At time $n$, we plug-in $\hat{\rho}_{n-1} + t_{n-1}$ in place of $\rho$ and follow the analysis of Section~\ref{withRhoKnown}. Define for $i=1,\ldots,n$
\begin{eqnarray}
\hat{\epsilon}_{i}^{(n)} &=& b\left( \left(  \sqrt{\frac{2}{m} \hat{\epsilon}_{i-1}^{(n)}}  + (\hat{\rho}_{n-1} + t_{n-1}) \right)^{2},\numIter_{i} \right) \nonumber
\end{eqnarray}
If it holds that $\hat{\rho}_{n-1} + t_{n-1} \geq \rho$, then ${\mathbb{E}\left[ f_{n}(\bx_{n})  \right] - f_{n}(\bx_{n}^{*}) \leq \hat{\epsilon}_{n}^{(i)}}$ for ${i=1,\ldots,n}$. Assumption~\ref{probState:assumpC1} guarantees that this holds for all $n$ large enough almost surely. We can thus set $\numIter_{n}$ equal to the smallest $\numIter$ such that
\[
b\left( \left( \sqrt{ \frac{2}{m} \max\{\hat{\epsilon}^{(n-1)}_{n-1},\epsilon\} } + ( \hat{\rho}_{n-1} + t_{n-1} ) \right)^{2} , \numIter  \right) \leq \epsilon
\]
for all $n \geq 3$ to achieve \meangap{} $\epsilon$. The maximum in this definition ensures that when $\hat{\rho}_{n-1} + t_{n-1} \geq \rho$, $\numIter_{n} \geq \numIter^{*}$ with $\numIter^{*}$ from \eqref{withRhoKnown:KChoice}. We can therefore apply Theorem~\ref{withRhoUnknown:meanGapRhoKnownLemma}.

\subsection{Do Not Update Past \MeanGap{} Bounds}
\label{withRhoUnknown:doNotUpdatePast}

Updating all past estimates of the \meangap{} bounds from time $1$ up to $n$ imposes a computational and memory  burden. Suppose that for all $n \geq 3$ we set
\begin{equation}
\label{withRhoUnknown:KnChoice}
\numIter_{n} = \min \left\{ \numIter \geq 1 \;\Bigg|\; b\left( \left( \sqrt{\frac{2\epsilon}{m}} + ( \hat{\rho}_{n-1} + t_{n-1} ) \right)^{2} , \numIter  \right) \leq \epsilon \right\}
\end{equation}
This is the same form as the choice in \eqref{withRhoKnown:KChoice} with $\hat{\rho}_{n-1} + t_{n-1}$ in place of $\rho$. Due to assumption~\ref{probState:assumpC1}, for all $n$ large enough it holds that $\hat{\rho}_{n} + t_{n} \geq \rho$ almost surely. Then by the monotonicity assumption in \ref{probState:assump1}, for all $n$ large enough we pick $\numIter_{n} \geq \numIter^{*}$ almost surely. We can therefore apply Theorem~\ref{withRhoUnknown:meanGapRhoKnownLemma}.

\section{Experiments}
We focus on two regression applications for synthetic and real data as well as two classification applications for synthetic and real data. For the synthetic regression problem, we can explicitly compute $\rho$ and $\bx_{n}^{*}$ and exactly evaluate the performance of our method. It is straightforward to check that all requirements in \ref{probState:assump1}-\ref{probState:assump4} are satisfied for the problems considered in this section. We apply the do not update past \meangap{} choice of $\numIter_{n}$ here.

\subsection{Synthetic Regression}

Consider a regression problem with synthetic data using the penalized quadratic loss 
\[
\lossFunc(\bx,\bz) = \frac{1}{2} \left( y - \bw^{\top} \bx \right)^{2} + \frac{1}{2} \lambda \| \bx \|^{2}
\]
with $\bz = (\bw,y) \in \mathbb{R}^{d+1}$. The distribution of $\bz_{n}$ is zero mean Gaussian with covariance matrix
\[
\left[ \begin{array}{cc}
					\sigma_{\bw}^{2} \bm{I} & r_{\bw_{n},y_{n}} \\
					r_{\bw_{n},y_{n}}^{\top} & \sigma_{y_{n}}^{2} 
				\end{array} \right]
\]
Under these assumptions, we can analytically compute minimizers $\bx_{n}^{*}$ of ${f_{n}(\bx) = \mathbb{E}_{\bz_{n} \sim p_{n}} \left[ \lossFunc(\bx,\bz_{n})  \right]}$. We change only $r_{\bw_{n},y_{n}}$ and $\sigma_{y_{n}}^{2}$ appropriately to ensure that $\| \bx_{n}^{*} - \bx_{n-1}^{*} \| = \rho$ holds for all $n$. We find approximate minimizers using SGD with $\lambda = 0.1$. We estimate $\rho$ using the direct estimate.

We let $n$ range from $1$ to $20$ with $\rho = 1$, a target \meangap{} $\epsilon = 0.1$, and $\numIter_{n}$ from \eqref{withRhoUnknown:KnChoice}. We average over twenty runs of our algorithm. \figurename{}~\ref{exper:synthRegress:rhoEst} shows $\hat{\rho}_{n}$, our estimate of $\rho$, which is above $\rho$ in general. \figurename{}~\ref{exper:synthRegress:Kn} shows the number of samples $\numIter_{n}$, which settles down. We can exactly compute $f_{n}(\bx_{n}) - f_{n}(\bx_{n}^{*})$, and so by averaging over the twenty runs of our algorithm, we can estimate the \meangap{} (denoted ``sample average estimate''). \figurename{}~\ref{exper:synthRegress:meanGap} shows this estimate of the \meangap{}, the target \meangap{}, and our bound on the \meangap{} from Section~\ref{withRhoUnknown:doNotUpdatePast}. We achieve at least our targeted \meangap{}

\begin{figure}[!ht]
\centering
\begin{minipage}[b]{0.48\linewidth} \quad
 	\centering
	\includegraphics[width = \linewidth]{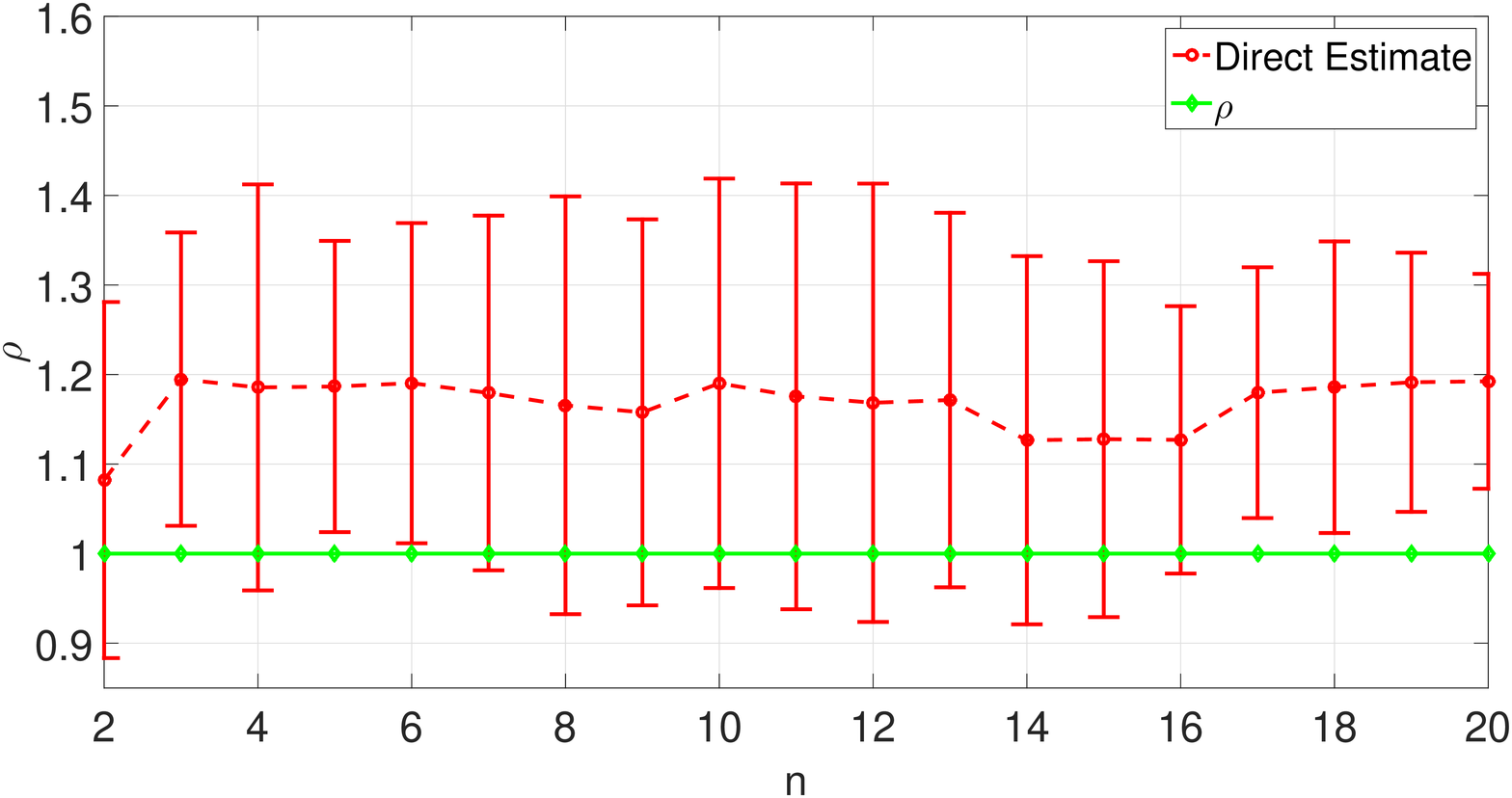}
	\caption{$\rho$ Estimate}
	\label{exper:synthRegress:rhoEst}
\end{minipage}
\begin{minipage}[b]{0.48\linewidth}
 	\centering
	\includegraphics[width = \linewidth]{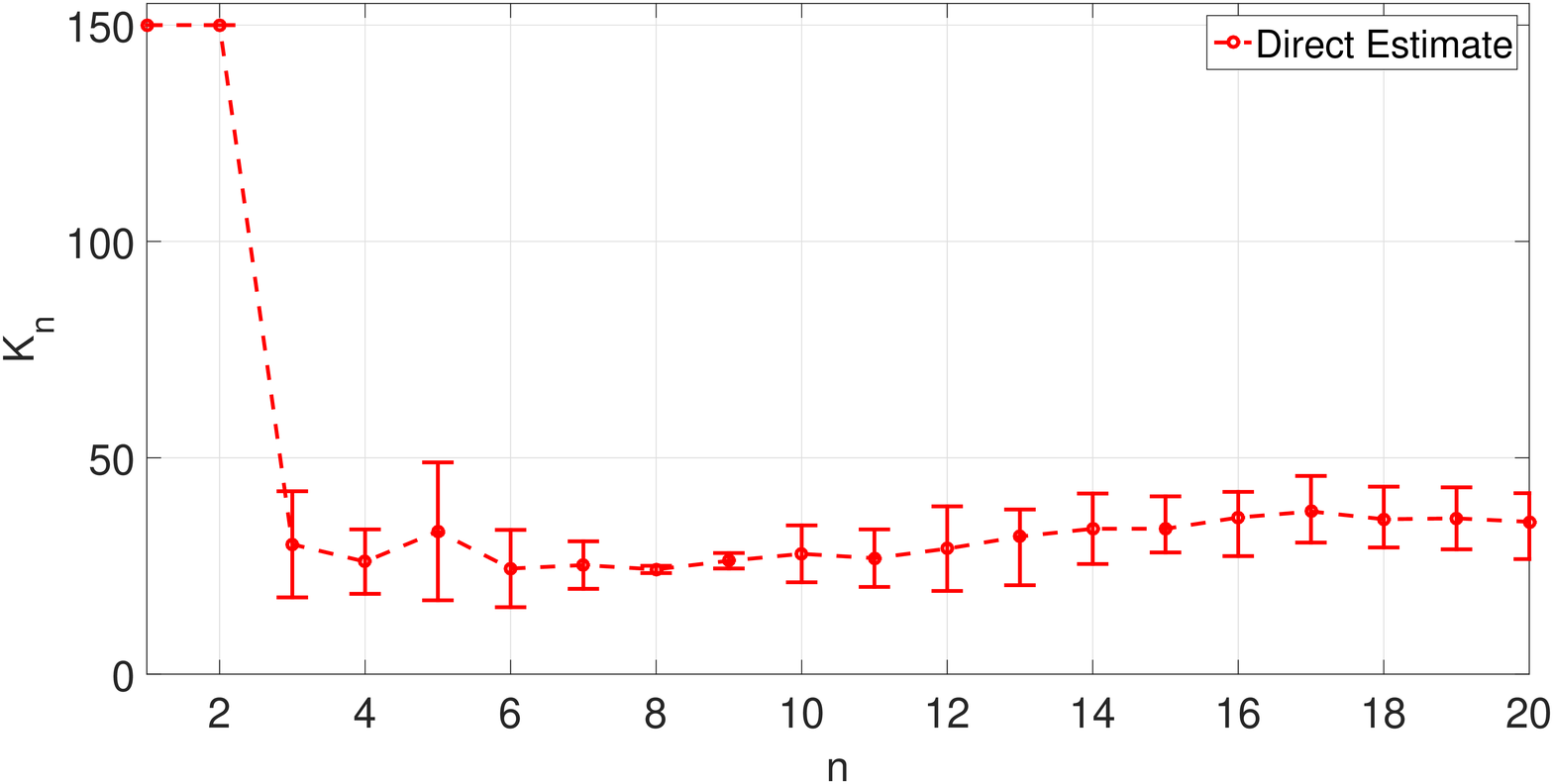}
	\caption{$\numIter_{n}$}
	\label{exper:synthRegress:Kn}
\end{minipage}

\medskip

\begin{minipage}[b]{0.48\linewidth}
 	\centering
 	\iftoggle{useMeanGap}{
 	\includegraphics[width = \linewidth]{synthRegressMeanGap}
 	}{
 	\includegraphics[width = \linewidth]{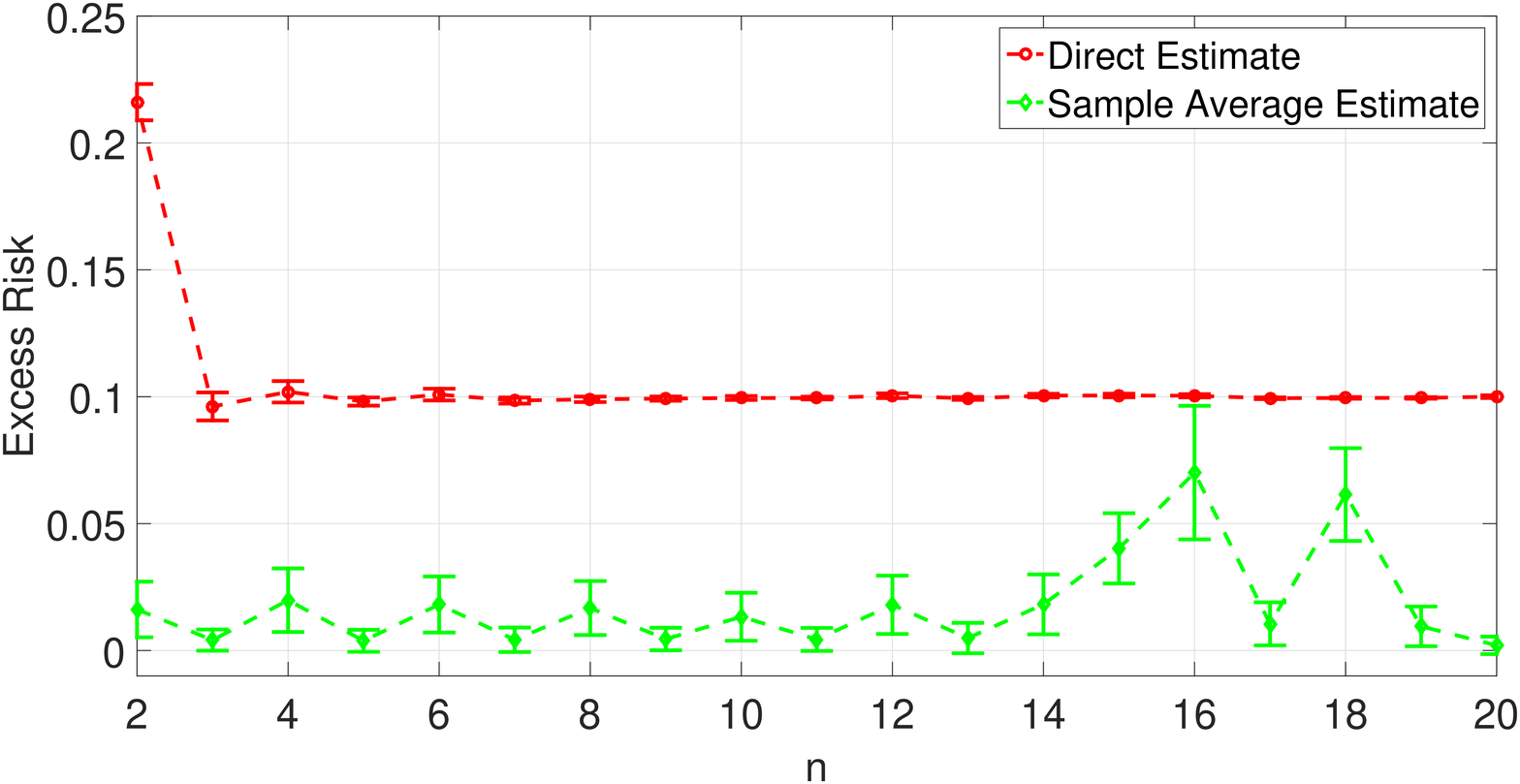}
 	}
	\caption{\MeanGap{}}
	\label{exper:synthRegress:meanGap}
\end{minipage}
\end{figure}

\subsection{Panel Study on Income Dynamics Income - Regression}
The Panel Study of Income Dynamics (PSID) surveyed individuals every year to gather demographic and income data annually from 1981-1997 \cite{PSID2015}. We want to predict an individual's annual income ($y$) from several demographic features ($\bw$) including age, education, work experience, etc. chosen based on previous economic studies in \cite{Jenkins2006}. The idea of this problem conceptually is to rerun the survey process and determine how many samples we would need if we wanted to solve this regression problem to within a desired \meangap{} criterion $\epsilon$.

We use the same loss function, direct estimate for $\rho$, and minimization algorithm as the synthetic regression problem.  The income is adjusted for inflation to 1997 dollars with mean \$20,294. We average over twenty runs of our algorithm by resampling without replacement \cite{Hastie2001}. We compare to taking an equivalent number of samples up front. \figurename{}~\ref{exper:psid:testLosses} shows the test losses over time evaluated over twenty percent of the available samples. The test loss for our approach is substantially less than taking the same number of samples up front. The square root of the average test loss over this time period for our approach and all samples up front are $\$1153 \pm 352$ and $\$2805 \pm 424$ respectively in 1997 dollars.	

\begin{figure}[!ht]
 	\centering
	\includegraphics[width = 0.5\textwidth]{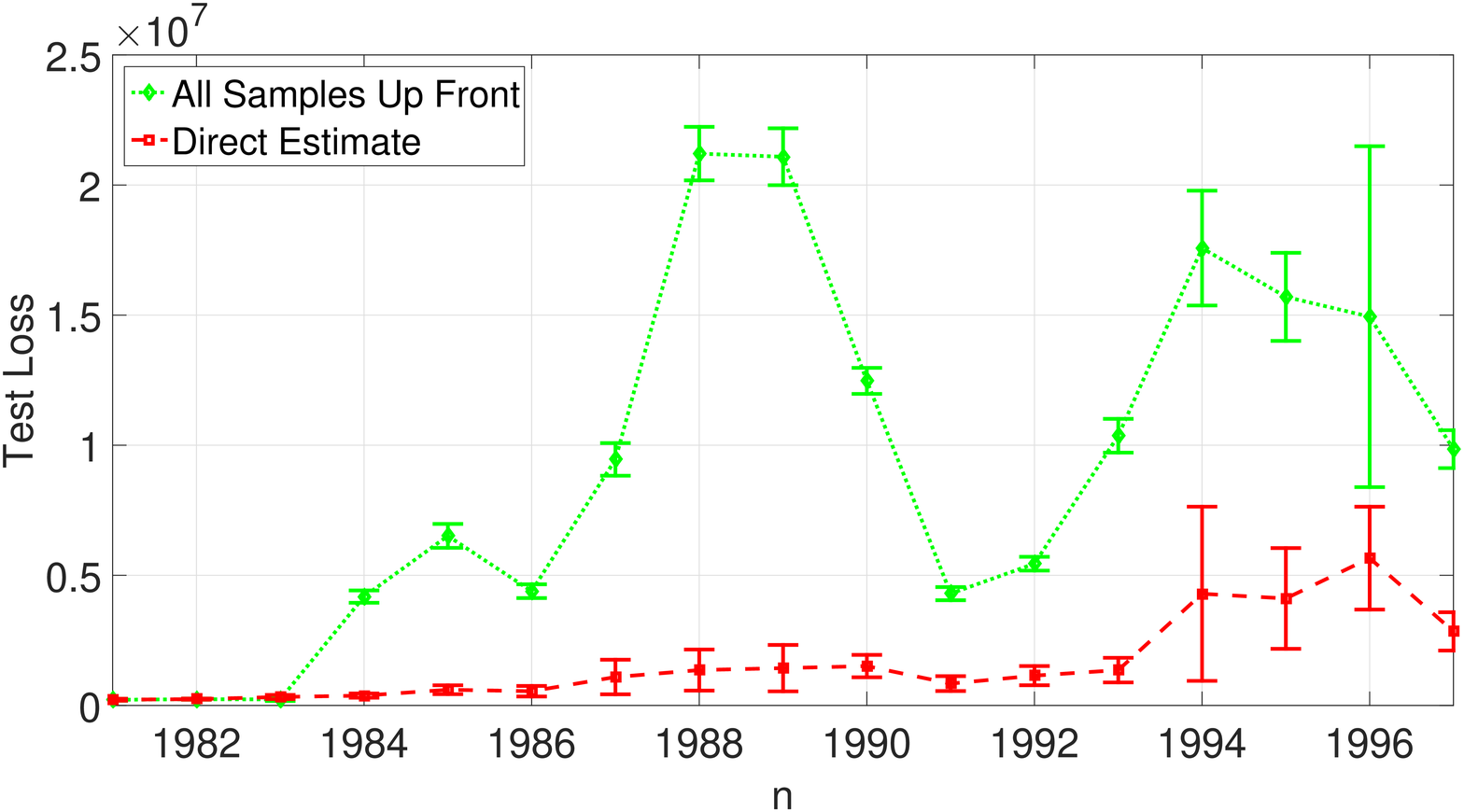}
	\caption{Test Loss}
	\label{exper:psid:testLosses}
\end{figure}

\subsection{Synthetic Classification}

Consider a binary classification problem using ${\lossFunc(\bx,\bz) = \frac{1}{2} ( 1 - y ( \bw^{\top} \bx ) )_{+}^{2} + \frac{1}{2} \lambda \| \bx \|^{2}}$ with ${\bz = (\bw,y) \in \mathbb{R}^{d} \times \mathbb{R}}$ and $(y)_{+} = \max\{y,0\}$. This is a smoothed version of the hinge loss used in support vector machines (SVM) \cite{Hastie2001}. We suppose that at time $n$, the two classes have features drawn from a Gaussian distribution with covariance matrix $\sigma^{2} \bm{I}$ but different means $\mu_{n}^{(1)}$ and $\mu_{n}^{(2)}$, i.e., ${\bw_{n} \;|\; \{y_{n} = i\} \;\sim\; \mathcal{N}(\mu_{n}^{(i)} , \sigma^{2} \bm{I})}$. The class means move slowly over uniformly spaced points on a unit sphere in $\mathbb{R}^{d}$ as in \figurename{}~\ref{exper:synthClass:ClassMeans} to ensure that \eqref{probState:slowChangeConstDef} holds. We find approximate minimizers using SGD with $\lambda = 0.1$. We estimate $\rho$ using the direct estimate with $t_{n} \propto 1/n^{3/8}$.

\begin{figure}[!ht]
\centering
\includegraphics[width = 0.3 \textwidth]{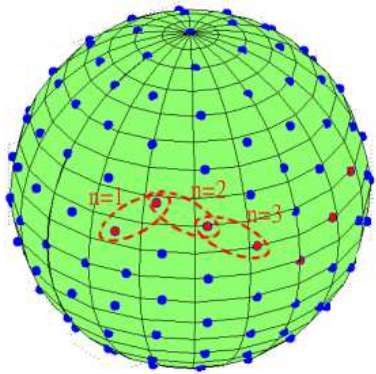}
\caption{Evolution of Class Means}
\label{exper:synthClass:ClassMeans}
\end{figure}

We let $n$ range from $1$ to $20$ and target a \meangap{} $\epsilon = 0.1$. We average over twenty runs of our algorithm. As a comparison, if our algorithm takes $\{\numIter_{n}\}_{n=1}^{20}$ samples, then we consider taking $\sum_{n=1}^{20} \numIter_{n}$ samples up front at $n=1$. This is what we would do if we assumed that our problem is not time varying. \figurename{}~\ref{exper:synthClass:rhoEst} shows $\hat{\rho}_{n}$, our estimate of $\rho$. \figurename{}~\ref{exper:synthClass:testLoss} shows the average test loss for both sampling strategies. To compute test loss we draw $T_{n}$ additional samples $\{\bz_{n}^{\text{test}}(\iterIndex)\}_{\iterIndex=1}^{T_{n}}$ from $p_{n}$ and compute $\frac{1}{T_{n}} \sum_{\iterIndex=1}^{T_{n}} \lossFunc(\bx_{n},\bz_{n}^{\text{test}}(\iterIndex))$. We see that our approach achieves substantially smaller test loss than taking all samples up front.

\vspace{-4mm}
\begin{figure}[!ht]
\centering
\begin{minipage}{0.5\textwidth}
 	\centering
	\includegraphics[width = 1 \textwidth]{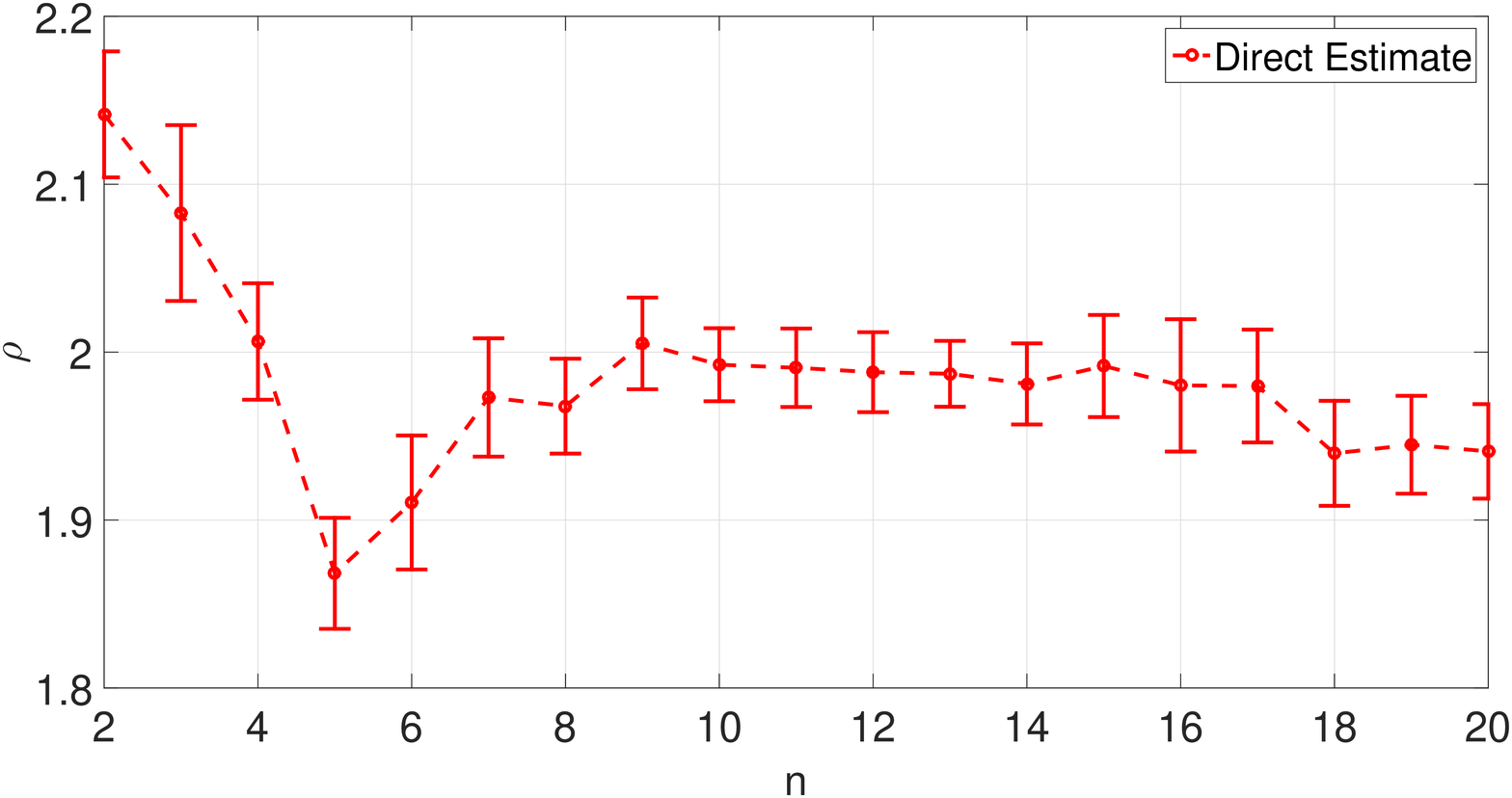}
	\caption{$\rho$ Estimate}
	\label{exper:synthClass:rhoEst}
\end{minipage}%
\begin{minipage}{0.5\textwidth}
 	\centering
	\includegraphics[width = 1 \textwidth]{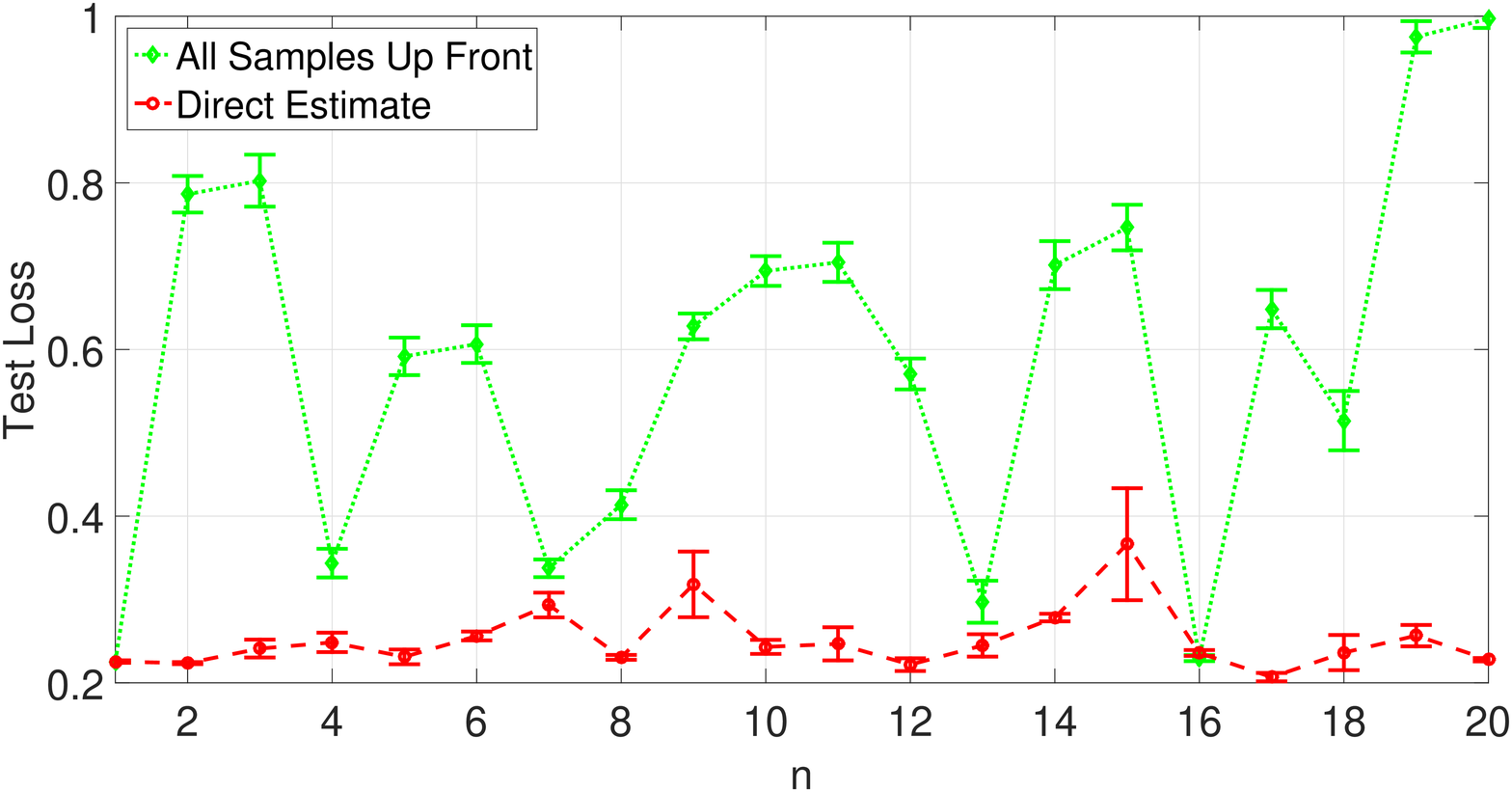}
	\caption{Test Loss}
	\label{exper:synthClass:testLoss}
\end{minipage}%
\end{figure}

\subsection{General Social Survey - Classification}

The General Social Survey (GSS) surveyed individuals every year to gather socio-economic data annually from 1981-2013 \cite{GSS2015}. We want to predict an individual's marital status ($y$) from several demographic features ($\bw$) including age, education, etc. We model this as a binary classification problem using loss
\[
\lossFunc(\bx,\bz) = \frac{1}{2} ( 1 - y ( \bw^{\top} \bx ) )_{+}^{2} + \frac{1}{2} \lambda \| \bx \|^{2}
\]
with ${\bz = (\bw,y) \in \mathbb{R}^{d} \times \mathbb{R}}$ and $(y)_{+} = \max\{y,0\}$. This is a smoothed version of the hinge loss used in support vector machines \cite{Hastie2001}. We find approximate minimizers using SGD with $\lambda = 0.1$. \figurename{}~\ref{exper:gssClass:testLoss} shows the test loss. We see that our approach achieves smaller test loss than taking all samples up front. We also plot receiver operating characteristics (ROC) \cite{Hastie2001} to characterize the performance of our classifiers. In particular we plot the ROC for 1974 in \figurename{}~\ref{exper:gssClass:roc1} and the ROC for 2012 in \figurename{}~\ref{exper:gssClass:roc28}. By examining the ROC, we see that taking all samples up front is much better in 1974 but much worse in 2012.

\begin{figure}[!ht]
\centering
\begin{minipage}[b]{0.48\linewidth} \quad
 	\centering
	\includegraphics[width = \textwidth]{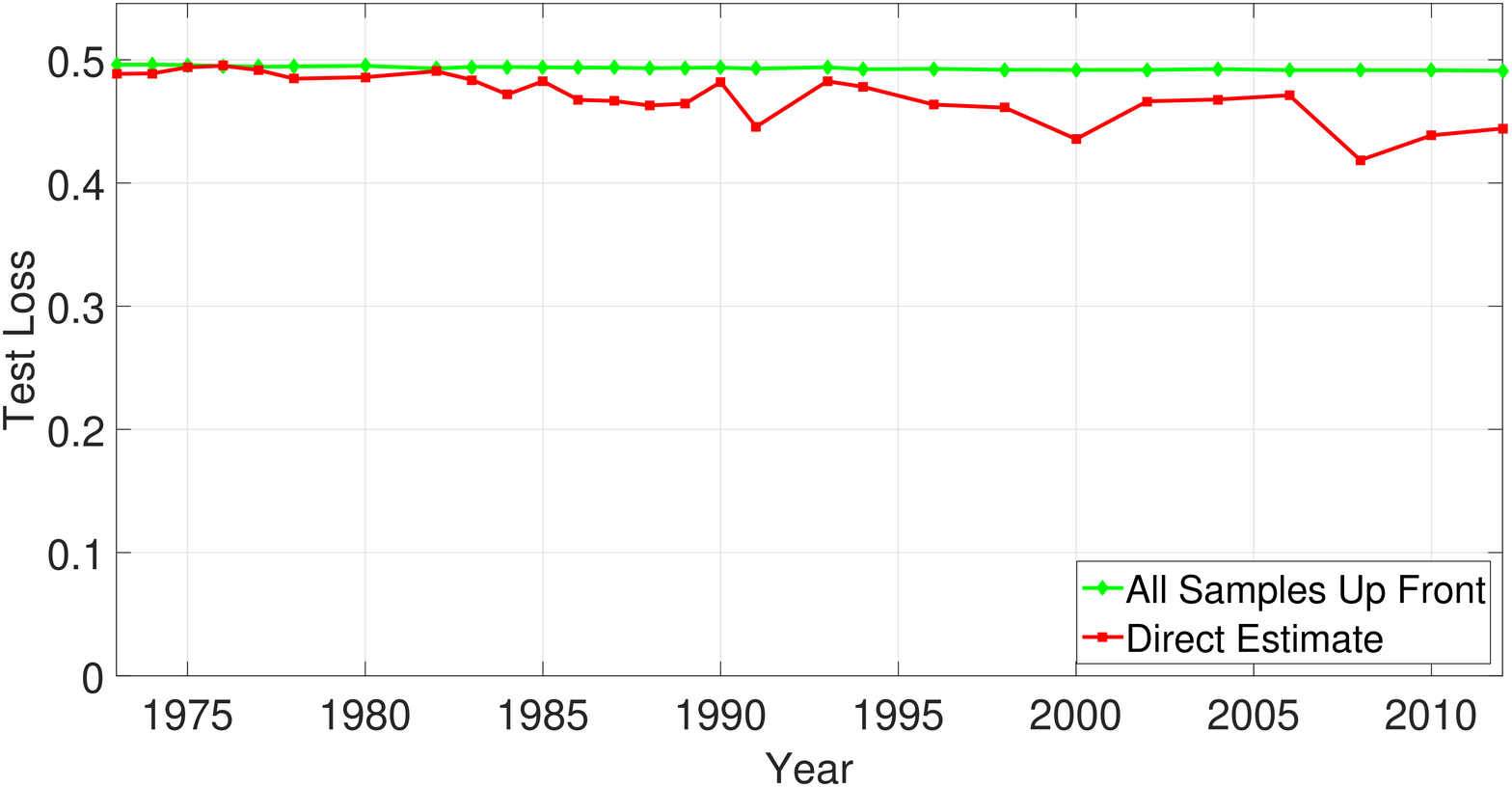}
	\caption{Test Loss}
	\label{exper:gssClass:testLoss}
\end{minipage}

\medskip

\begin{minipage}[b]{0.48\linewidth} \quad
 	\centering
	\includegraphics[width = \linewidth]{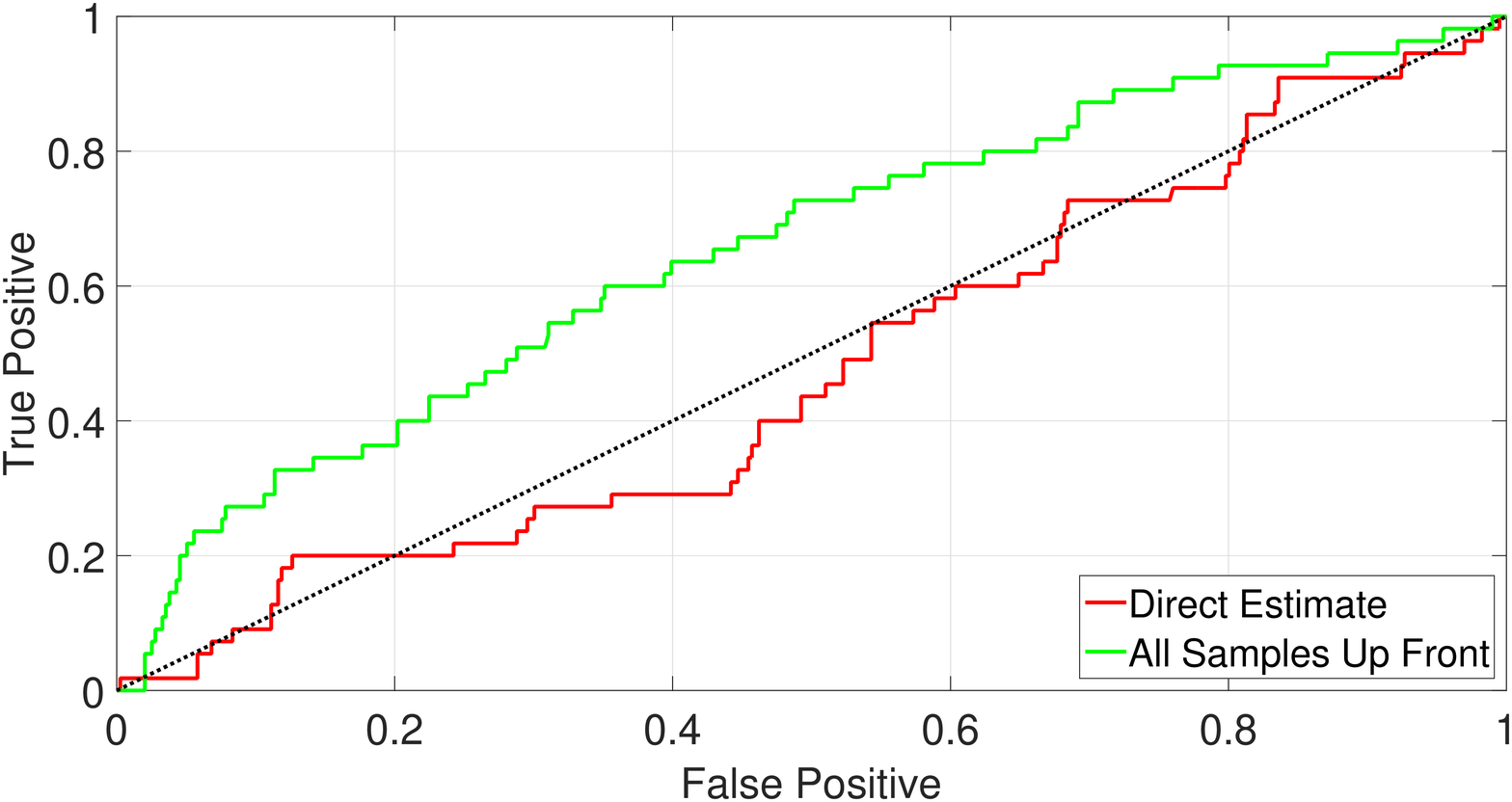}
	\caption{ROC for 1974}
	\label{exper:gssClass:roc1}
\end{minipage}
\begin{minipage}[b]{0.48\linewidth}
 	\centering
 	\includegraphics[width = \linewidth]{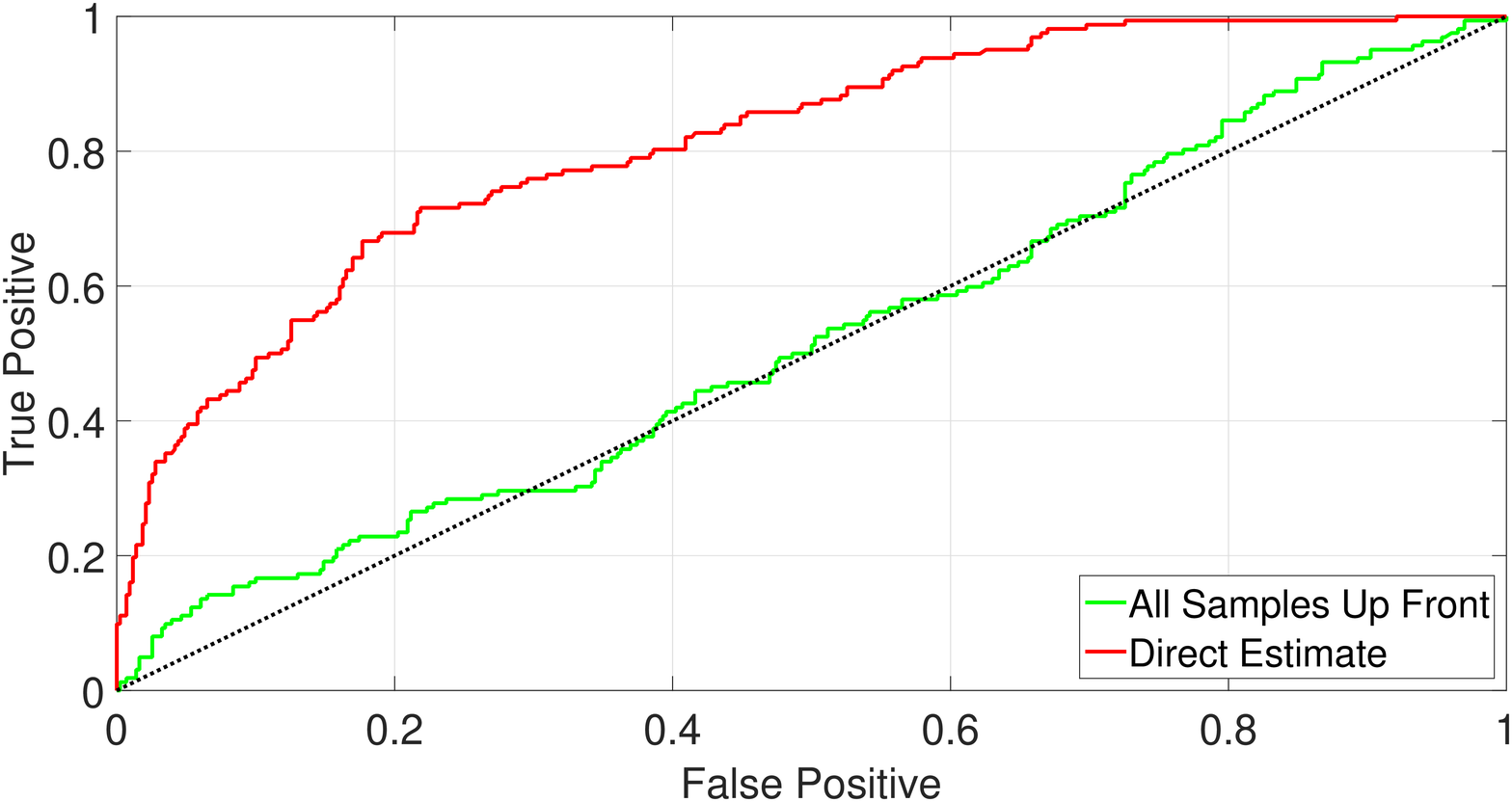}
	\caption{ROC for 2012}
	\label{exper:gssClass:roc28}
\end{minipage}
\end{figure}

\section{Conclusion}

We introduced a framework for adaptively solving a sequence of optimization problems with applications to machine learning. We developed estimates of the change in the minimizers used to determine the number of samples $\numIter_{n}$ needed to achieve a target \meangap{} $\epsilon$. Experiments with synthetic and real data demonstrate that this approach is effective.


\bibliographystyle{IEEEbib}
\bibliography{IEEEabrv,SlowlyChangingTasks_ICASSP2016,online,stochastic_prop}

\begin{thebibliography}{10}

\bibitem{Mohri2012}
M.~Mohri, A.~Rostamizadeh, and A.~Talwalkar,
\newblock {\em Foundations of Machine Learning},
\newblock The MIT Press, 2012.

\bibitem{Agarwal2011}
A.~Agarwal, H.~Daum\'e, and S.~Gerber,
\newblock ``Learning multiple tasks using manifold regularization.,''
\newblock in {\em NIPS}, 2011, pp. 46--54.

\bibitem{Evgeniou2004}
T.~Evgeniou and M.~Pontil,
\newblock ``Regularized multi--task learning,''
\newblock in {\em Proceedings of the Tenth ACM SIGKDD International Conference
  on Knowledge Discovery and Data Mining}, New York, NY, USA, 2004, KDD '04,
  pp. 109--117, ACM.

\bibitem{Zhang2012}
Y.~Zhang and D.~Yeung,
\newblock ``A convex formulation for learning task relationships in multi-task
  learning,''
\newblock {\em CoRR}, vol. abs/1203.3536, 2012.

\bibitem{Pan2010}
S.~Pan and Q.~Yang,
\newblock ``A survey on transfer learning,''
\newblock {\em IEEE Transactions on Knowledge and Data Engineering}, vol. 22,
  no. 10, pp. 1345--1359, Oct 2010.

\bibitem{Agarwal2008}
A.~Agarwal, A.~Rakhlin, and P.~Bartlett,
\newblock ``Matrix regularization techniques for online multitask learning,''
\newblock Tech. {R}ep. UCB/EECS-2008-138, EECS Department, University of
  California, Berkeley, Oct 2008.

\bibitem{Towfic2013}
Z.~Towfic, J.~Chu, and A.~Sayed,
\newblock ``Online distirubted online classifcation in the midst of concept
  drifts,''
\newblock {\em Neurocomputing}, vol. 112, pp. 138--152, 2013.

\bibitem{Tekin2014}
C.~Tekin, L.~Canzian, and M.~van~der Schaar,
\newblock ``Context adaptive big data stream mining,''
\newblock in {\em Allerton Conference}, 2014, pp. 46--54.

\bibitem{Dietterich2002}
T.~Dietterich,
\newblock ``Machine learning for sequential data: A review,''
\newblock in {\em Structural, Syntactic, and Statistical Pattern Recognition},
  2002, pp. 15--30.

\bibitem{Fawcett1997}
T.~Fawcett and F.~Provost,
\newblock ``Adaptive fraud detection.,''
\newblock {\em Data Min. Knowl. Discov.}, vol. 1, no. 3, pp. 291--316, 1997.

\bibitem{Qian1988}
N.~Qian and T.~Sejnowski,
\newblock ``Predicting the secondary structure of globular proteins using
  neural network models,''
\newblock {\em Journal of Molecular Biology}, vol. 202, pp. 865--884, Aug 1988.

\bibitem{BengioFrasconi1996}
Y.~Bengio and P.~Frasconi,
\newblock ``Input-output {HMM}'s for sequence processing,''
\newblock {\em IEEE Transactions on Neural Networks}, vol. 7(5), pp.
  1231--1249, 1996.

\bibitem{Dontchev2009}
A.~Dontchev and R.~Rockafellar,
\newblock {\em Implicit Functions and Solution Mappings: A View from
  Variational Analysis},
\newblock Springer, New York, New York, 2009.

\bibitem{Sriperumbudur12}
B.~Sriperumbudur,
\newblock ``On the empirical estimation of integral probability metrics,''
\newblock {\em Electronic Journal of Statistics}, pp. 1550--1599, 2012.

\bibitem{Very2012}
R.~Veryshin,
\newblock ``Introduction to non-asymptotic analysis of random matrices,''
\newblock Tech. {R}ep., University of Michigan, 2012.

\bibitem{Nemirovski2009}
A.~Nemirovski, A.~Juditsky, G.~Lan, and A.~Shapiro,
\newblock ``Stochastic approximation approach to stochastic programming,''
\newblock {\em SIAM Journal on Optimization}, vol. 19, pp. 1574--1609, 2009.

\bibitem{Buld2010}
V.V Buldygin and E.D. Pechuk,
\newblock ``Inequalities for the distributions of functionals of sub-gaussian
  vectors,''
\newblock {\em Theor. Probability and Math. Statist.}, pp. 25--36, 2010.

\bibitem{Janson04}
S.~Janson,
\newblock ``Large deviations for sums of partly dependent random variables,''
\newblock {\em Random Structures Algorithms}, vol. 24, pp. 234--248, 2004.

\bibitem{Boucheron13}
S.~Boucheron, G.~Lugosi, and P.~Massart,
\newblock {\em Concentration Inequalities: A Nonasymptotic Theory of
  Independence},
\newblock Oxford University Press, 2013.

\bibitem{Trefethen1997}
L.~Trefethen,
\newblock {\em Numerical Linear Algebra},
\newblock SIAM, 1997.

\bibitem{Kennan2001}
J.~Kennan,
\newblock ``Uniqueness of positive fixed points for increasing concave
  functions on rn: An elementary result,''
\newblock {\em Review of Economic Dynamics}, vol. 4, pp. 893–899, 2001.

\bibitem{Boyd2004}
Stephen Boyd and Lieven Vandenberghe,
\newblock {\em Convex Optimization},
\newblock Cambridge University Press, New York, NY, USA, 2004.

\bibitem{Granas2003}
A.~Granas and J.~Dugundji,
\newblock {\em Fixed Point Theory},
\newblock Springer-Verlag, 2003.

\bibitem{PSID2015}
``Panel study of income dynamics: public use dataset,''
\newblock {\em Survey Research Center}, 2015.

\bibitem{Jenkins2006}
S.~Jenkins and P.~Van Kerm,
\newblock ``Trends in income inequality, pro-poor income growth, and income
  mobility,''
\newblock {\em Oxford Economic Papers}, vol. 58, no. 3, pp. 531--548, 2006.

\bibitem{Hastie2001}
T.~Hastie, R.~Tibshirani, and J.H. Friedman,
\newblock {\em The elements of statistical learning: data mining, inference,
  and prediction: with 200 full-color illustrations},
\newblock New York: Springer-Verlag, 2001.

\bibitem{GSS2015}
``General social survey,''
\newblock {\em National Opinion Research Center}, 2015.

\bibitem{BachMoulines2011}
F.~Bach and E.~Moulines,
\newblock ``{Non-Asymptotic Analysis of Stochastic Approximation Algorithms for
  Machine Learning},''
\newblock in {\em {Advances in Neural Information Processing Systems (NIPS)}},
  Spain, 2011.

\bibitem{Bertsekas1999}
D.~Bertsekas,
\newblock {\em Nonlinear Programming},
\newblock Athena Scientific, 1999.

\bibitem{Bottou1998}
L\'eon Bottou,
\newblock ``Online learning and stochastic approximations,'' 1998.

\bibitem{NedicLee12}
A.~Nedic and S.~Lee,
\newblock ``Analysis of mirror descent for strongly convex functions,''
\newblock {\em ArXiV}, 2013.

\bibitem{NesterovBook2004}
Yu. Nesterov,
\newblock {\em Introductory Lectures on Convex Optimization: A Basic Course},
\newblock Kluwer Academic Publishers, Norwell, Massachusetts, USA, 2004.

\bibitem{Antonini2005}
R.~Antonini and Y.~Kozachenko,
\newblock ``A note on the asymptotic behavior of sequences of generalized
  subgaussian random vectors,''
\newblock {\em Random Op. and Stoch. Equ.}, vol. 13, pp. 39--52, 2005.

\end{thebibliography}

\appendix

\section{Examples of $b(d_{0},\numIter)$:}
\label{bBounds}

For this section, we drop the $n$ index for convenience. The bounds of this form depend on the strong convexity parameter $m$ and an assumption on how the gradients grow. In general, we assume that
\[
\mathbb{E}_{\bz \sim p}\| \nabla_{\bx} \lossFunc(\bx,\bz) \|^{2} \leq A + B \| \bx - \bx^{*} \|^{2}
\]
The base algorithm we look at is SGD. First, we generate iterates $\bx(0),\ldots,\bx(\numIter)$ through SGD as follows:
\begin{eqnarray}
\bx(\ell+1) &=& \Pi_{\xSp} \left[ \bx(\ell) - \mu(\ell+1) \nabla_{\bx} \lossFunc(\bx(\ell),\bz(\ell))  \right] \;\;\;\; \ell = 0,\ldots,\numIter-1 \nonumber
\end{eqnarray}
with $\bx(0)$ fixed. We then combine the iterates to yield a final approximate minimizer
\begin{eqnarray}
\bar{\bx}(\numIter) &=& \phi(\bx(0),\ldots,\bx(\numIter)) \nonumber
\end{eqnarray} 
For our choice of $\phi$, we look at two cases:
\begin{enumerate}
\item No iterate averaging, i.e.,
\[
\phi(\bx(0),\ldots,\bx(\numIter)) = \bx(\numIter)
\]
\item Iterate averaging, i.e, for a convex combination $\{\lambda(\ell)\}_{\ell=0}^{\numIter}$
\[
\phi(\bx(0),\ldots,\bx(\numIter)) = \sum_{\ell=0}^{\numIter} \lambda(\ell) \bx(\ell)
\]
\end{enumerate}

Define
\begin{equation}
\label{bBounds:dDef}
d(\ell) \triangleq \| \bx(\ell) - \bx^{*} \|^{2}
\end{equation}
First we bound $\mathbb{E}[d(\ell)]$ in Lemma~\ref{bBounds:dBound}.
\begin{lem}
\label{bBounds:dBound}
Suppose that the function $f(\bx)$ has Lipschitz continuous gradients. Then it holds that
\[
\mathbb{E}[d(\ell)] \leq \prod_{\iterIndex = 1}^{\ell} (1 - 2 m \mu(\ell) + B \mu^{2}(\ell)) + \sum_{\iterIndex = 1}^{\ell} \prod_{i=\iterIndex+1}^{\ell} (1 - 2 m \mu(i) + B \mu^{2}(i)) \mu^{2}(\iterIndex)
\]
\end{lem}
\begin{proof}
Following the standard SGD analysis (see \cite{Nemirovski2009}), it holds that
\begin{eqnarray}
d(\ell) &\leq& \| \bx(\ell-1) - \bx^{*} - \mu(\ell) \nabla_{\bx} \lossFunc(\bx(\ell-1),\bz(\ell)) \|^{2} \nonumber \\
&\leq& d(\ell-1) - 2 \mu(\ell) \inprod{\bx(\ell-1) - \bx^{*}}{\nabla_{\bx} \lossFunc(\bx(\ell-1),\bz(\ell))} + \mu^{2}(\ell) \| \nabla_{\bx} \lossFunc(\bx(\ell-1),\bz(\ell)) \|^{2}  \nonumber
\end{eqnarray}
Then it follows that
\begin{align}
\mathbb{E}&[d(\ell) \;|\; \bx(\ell-1)] \nonumber \\
&\leq d(\ell-1)  - 2 \mu(\ell) \inprod{\bx(\ell-1) - \bx^{*}}{\nabla f(\bx(\ell-1)) } + \mu^{2}(\ell) \mathbb{E}[ \| \nabla_{\bx} \lossFunc(\bx(\ell-1),\bz(\ell)) \|^{2} \;|\; \bx(\ell-1)]  \nonumber \\
&\leq (1 - 2m \mu(\ell) + B \mu^{2}(\ell)) d(\ell-1) + \mu^{2}(\ell-1) A \nonumber
\end{align}
and
\[
\mathbb{E}[d(\ell)] \leq (1 - 2 m \mu(\ell) + B \mu^{2}(\ell)) \mathbb{E}[d(\ell-1)] + \mu^{2}(\ell-1) A
\]
Since $B > m$, we have
\[
2 m \mu - B \mu^2 \leq 2 \sqrt{\frac{B}{2}} \mu \left( 1 - \sqrt{\frac{B}{2}} \mu \right) \leq 2 \frac{1}{4} = \frac{1}{2}
\]
and so
\[
1 - 2m \mu(\ell) + B \mu^{2}(\ell) \geq 1 - \frac{1}{2} = \frac{1}{2}
\]
Since this quantity is non-negative, we can unwind this recursion to yield
\[
\mathbb{E}[d(\ell)] \leq \prod_{\iterIndex = 1}^{\ell} (1 - 2 m \mu(\ell) + B \mu^{2}(\ell)) + \sum_{\iterIndex = 1}^{\ell} \prod_{i=\iterIndex+1}^{\ell} (1 - 2 m \mu(i) + B \mu^{2}(i)) \mu^{2}(\iterIndex)
\]
\end{proof}

The bound in Lemma~\ref{bBounds:dBound} can be further bounded into a closed form as follows from \cite{BachMoulines2011}:
Define
\[
\varphi_{\beta}(t) = \begin{cases}
\frac{t^{\beta} - 1}{\beta}, & \text{if }\beta \neq 0 \\
\log(t), & \text{if }\beta = 0
\end{cases}
\]
Then with $\mu(\ell) = C \ell^{-\alpha}$, it holds that
\[
\mathbb{E}[d(\ell)] \leq \begin{cases}
2 \exp\left\{ 2 B C^{2} \varphi_{1-2\alpha}(\ell) \right\} \exp\left\{ - \frac{mC}{4} \ell^{1-\alpha} \right\} \left( \mathbb{E}[d(0)] + \frac{A}{B} \right) + \frac{2AC}{m \ell^{\alpha}}, & \text{if } 0 \leq \alpha < 1 \\
\frac{\exp\left\{ B C^{2}  \right\}}{\ell^{mC}} \left( \mathbb{E}[d(0)] + \frac{A}{B} \right) + A C^{2} \frac{ \varphi_{mC/2-1}(\ell)}{\ell^{mC/2}},& \text{if } \alpha = 1
\end{cases}
\]
Note that this bound is a closed form but is substantially looser than Lemma~\ref{bBounds:dBound}. In the case that the functions in question have Lipschitz continuous gradients, we introduce a bound on the \meangap{} using Lemma~\ref{bBounds:dBound}. This case corresponds to choosing
\[
\phi(\bx(0),\ldots,\bx(\numIter)) = \bx(\numIter)
\]

\begin{lem}
\label{bBounds:fBoundBasic}
With arbitrary step sizes and assuming that $f(\bx)$ has Lipschitz continuous gradients with modulus $M$, it holds that
\[
\mathbb{E}[f(\bx)] - f(\bx^{*}) \leq \frac{1}{2} M \mathbb{E}[d(\numIter)]
\]
and therefore, we set
\[
b(d_{0},\numIter) = \frac{1}{2} M \left( \prod_{\ell = 1}^{\numIter} (1 - 2 m \mu(\ell) + B \mu^{2}(\ell)) + \sum_{\ell = 1}^{\numIter} \prod_{i=\ell+1}^{\numIter} (1 - 2 m \mu(i) + B \mu^{2}(i)) \mu^{2}(\ell) \right) \nonumber
\]
\end{lem}
\begin{proof}
Using the descent lemma from \cite{Bertsekas1999}, it holds that
\[
\mathbb{E}[f(\bx)] - f(\bx^{*}) \leq \frac{1}{2} M \mathbb{E}[d(\numIter)]
\]
Plugging in the bound from Lemma~\ref{bBounds:dBound} yields the bound $b(d_{0},\numIter)$.
\end{proof}

Next, we introduce a bound inspired by \cite{Bottou1998} for the case where $\phi(\bx(0),\ldots,\bx(\numIter))$ corresponds to forming a convex combination of the iterates.

\begin{lem}
\label{bBounds:fBoundConstantAve}
With a constant step size and averaging with
\[
\lambda(\ell) = \begin{cases}
\frac{\gamma(\ell)}{\sum_{\tau=1}^{\numIter} \gamma(\tau)},& \text{if }\ell >0 \\
0, & \text{if } \ell = 0
\end{cases}
\]
where
\[
\gamma(\ell) = (1 - m \mu + B \mu^{2})^{-\ell}
\]
it holds that
\[
b(d_{0},\numIter) = \frac{d_{0}}{2 \mu \sum_{\ell=0}^{\numIter} \gamma(\ell)} + \frac{1}{2} A \mu
\]
\end{lem}
\begin{proof}
By strong convexity, it holds that
\[
-\inprod{\bx(\ell-1) - \bx^{*}}{\nabla f(\bx(\ell-1))} \leq -m\|\bx(\ell-1) - \bx^{*}\|^{2} - \left( f(\bx(\ell-1)) - f(\bx^{*})  \right)
\]
Following the Lyapunov-style analysis of Lemma~\ref{bBounds:dBound}, it holds that
\[
\mathbb{E}[d(\ell)] \leq (1-m \mu + B \mu^{2}) \mathbb{E}[d(\ell-1)] - 2 \mu \left( \mathbb{E}[f(\bx(\ell-1))] - f(\bx^{*}) \right) + A \mu^{2}
\]
Rearranging, using the telescoping sum, and using convexity, it holds that
\[
\mathbb{E}[f(\bx)] - f(\bx^{*}) \leq \frac{d_{0}}{2 \mu \sum_{\tau=0}^{\numIter} \gamma(\tau)} + \frac{1}{2} A \mu
\]
\end{proof}
If we set $\mu = \frac{1}{\sqrt{\numIter}}$, then it holds that
\[
b(d_{0},\numIter) = \mathcal{O}\left( \frac{1}{\sqrt{\numIter}} \right)
\]
for Lemma~\ref{bBounds:fBoundConstantAve}.

We consider an extension of the averaging scheme in \cite{NedicLee12}. The bound in this paper only works with $B = 0$, so we extend it slightly to handle $B > 0$.

\begin{lem}
\label{bBounds:fBoundNedLeeAve}
 Consider the choice of step sizes given by 
\[
\mu(\ell) = \frac{1}{m \ell}  \;\;\;\; \forall \ell \geq 1
\]
Then
\[
b(d_{0},\numIter) = \frac{\frac{1}{2} d(0) + \frac{1}{2} (\numIter+1) A + \frac{1}{2} B \sum_{\ell=0}^{\numIter} \gamma(\ell)}{ 1 + \frac{1}{2} m (\numIter+1)(\numIter+2) }
\]
where
\[
\mathbb{E}[d(\ell)] \leq \gamma(\ell)
\] 
Note that we can use the bound in Lemma~\ref{bBounds:dBound} here.
\end{lem}
\begin{proof}
We have using Lyapunov style analysis
\[
\mathbb{E}[d(\ell)] \leq (1 - 2 m \mu(\ell) + B \mu^{2}(\ell)) \mathbb{E}[d(\ell-1)] - 2 \mu(\ell) ( \mathbb{E}[f(\bx(\ell))] - f(\bx^{*})) + A \mu^{2}(\ell)
\]
Then we have
\[
\frac{1}{\mu^{2}(\ell)} \mathbb{E}[d(\ell)] \leq \left( \frac{1 - 2 m \mu(\ell)}{\mu^{2}(\ell)} + B  \right) \mathbb{E}[d(\ell-1)] - \frac{2}{\mu(\ell)} ( \mathbb{E}[f(\bx(\ell))] - f(\bx^{*}) + A 
\]
It holds that
\begin{eqnarray}
\frac{1 - 2 m \mu(\ell)}{\mu^{2}(\ell)} - \frac{1}{\mu^{2}(\ell-1)} &=& \frac{1}{\mu^{2}(\ell)} - 2 m \frac{1}{\mu(\ell)} - \frac{1}{\mu^{2}(\ell-1)} \nonumber \\
&=& \frac{\ell^{2}}{C^{2}} - \frac{2m\ell}{C} - \frac{(\ell-1)^{2}}{C^{2}} \nonumber \\
&=& \frac{2(mC -1)L - 1}{C^{2}} \nonumber
\end{eqnarray}
As long as we have
\[
mC - 1 \leq 1 \;\; \Leftrightarrow \;\; C \leq \frac{2}{m} 
\]
then we get
\[
\frac{1}{\mu^{2}(\ell)} \mathbb{E}[d(\ell)] - \frac{1}{\mu^{2}(\ell-1)} \mathbb{E}[d(\ell-1)] \leq  B \mathbb{E}[d(\ell-1)] - \frac{2}{\mu(\ell)} ( \mathbb{E}[f(\bx(\ell))] - f(\bx^{*}) + A 
\]
Summing an rearranging yields
\[
\sum_{\ell=0}^{\numIter} \frac{1}{\mu(\ell)} \left( \mathbb{E}[f(\bx(\ell))] - f(\bx^{*}) \right) \leq \frac{1}{2} d(0) + \frac{1}{2} (\numIter+1) A + \frac{1}{2} B \sum_{\ell=0}^{\numIter} \mathbb{E}[d(\ell)]
\]
with $\mu(0) = 1$ by convention. With the weights
\[
\gamma(\ell) = \frac{\frac{1}{\mu(\ell)}}{\sum_{j=0}^{\ell} \frac{1}{\mu(j)}}
\]
we have
\[
\mathbb{E}[f(\bar{\bx}(\numIter))] - f(\bx^{*}) \leq \frac{\frac{1}{2} d(0) + \frac{1}{2} (\numIter+1) A + \frac{1}{2} B \sum_{\ell=0}^{\numIter} \mathbb{E}[d(\ell)]}{\sum_{\tau=0}^{\numIter} \frac{1}{\mu(\tau)}}
\]
Then it holds that
\[
\sum_{\tau = 0}^{\numIter} = 1 + \sum_{\tau=1}^{\numIter} m \tau = 1 + \frac{1}{2} m (\numIter+1)(\numIter+2) 
\]
so
\[
\mathbb{E}[f(\bar{\bx}(\numIter))] - f(\bx^{*}) \leq \frac{\frac{1}{2} d(0) + \frac{1}{2} (\numIter+1) A + \frac{1}{2} B \sum_{\ell=0}^{\numIter} \mathbb{E}[d(\ell)]}{ 1 + \frac{1}{2} m (\numIter+1)(\numIter+2) }
\]
\end{proof}

For the choice of step sizes in Lemma~\ref{bBounds:fBoundNedLeeAve} from Lemma~\ref{bBounds:dBound}, it holds that
\[
\mathbb{E}[d(\ell)] = \mathcal{O}\left( \frac{1}{\ell} \right)
\]
Since
\[
\sum_{\ell=1}^{\numIter} \frac{1}{\ell} = \mathcal{O}\left( \log \numIter \right)
\]
it holds that
\[
\mathbb{E}[f(\bar{\bx}(\numIter))] - f(\bx^{*}) = \mathcal{O}\left( \frac{d(0)}{\numIter^{2}} + \frac{\log(\numIter)}{\numIter^{2}} + \frac{1}{\numIter}  \right)
\]
Note that a rate of $\mathcal{O}(\frac{1}{\numIter})$ is minimax optimal for stochastic minimization of a strongly convex function \cite{NesterovBook2004}.

Next, we look at a special case of averaging for functions such that
\[
\mathbb{E}\| \nabla_{\bx} \lossFunc(\bx,\bz) - \nabla_{\bx} \lossFunc(\tilde{\bx},\bz) - \nabla_{\bx \bx}^{2} \lossFunc(\tilde{\bx},\bz) \left( \bx - \tilde{\bx} \right) \|^{2} = 0
\]
from \cite{BachMoulines2011}. For example, quadratics satisfy this condition.
\begin{lem}
\label{bBounds:fBoundBachAve}
Assuming that
\[
\mathbb{E}\| \nabla_{\bx} \lossFunc(\bx,\bz) - \nabla_{\bx} \lossFunc(\tilde{\bx},\bz) - \nabla_{\bx \bx}^{2} \lossFunc(\tilde{\bx},\bz) \left( \bx - \tilde{\bx} \right) \|^{2} = 0,
\]
we select step sizes
\[
\mu(\ell) = C \ell^{-\alpha}
\]
with $\alpha > 1/2$, and
\[
\lambda(\ell) = \begin{cases}
\frac{1}{\numIter}, & \text{if } \ell > 0 \\
0, & \text{if } \ell = 0
\end{cases}
\]
it holds that
\begin{align}
&\left(\mathbb{E}[\bar{d}(\numIter)] \right)^{1/2} \nonumber \\
&\;\; \leq \frac{1}{m^{1/2}} \sum_{\iterIndex=1}^{\numIter-1} \bigg| \frac{1}{\mu(\iterIndex+1)}  - \frac{1}{\mu(\iterIndex)} \bigg| \left( \mathbb{E}[d(\iterIndex)] \right)^{1/2} + \frac{1}{m^{1/2} \mu(1)}  \left( \mathbb{E}[d(0)] \right)^{1/2} + \frac{1}{m^{1/2} \mu(\numIter)} \left( \mathbb{E}[d(\numIter)] \right)^{1/2} \nonumber \\
&\;\;\;\;\;\;\;\;\;\;\;\; + \sqrt{\frac{A}{ m \numIter}} +  \sqrt{\frac{2B}{m \numIter^{2}} \sum_{\iterIndex=1}^{\numIter} \mathbb{E}[d(\iterIndex-1)]} \nonumber
\end{align}
with $\bar{d}(\numIter) = \| \bar{\bx}(\numIter) - \bx^{*}\|^{2}$. If in addition $f$ has Lipschitz continuous gradients with modulus $M$, then it holds that
\[
\mathbb{E}[f(\bar{\bx}(\numIter))] - f(\bx^{*}) \leq \frac{1}{2} M  \mathbb{E}[\bar{d}(\numIter)]
\]
\end{lem}
\begin{proof}
Suppose that we set
\[
\bar{\bx}(\numIter) = \frac{1}{n} \sum_{\iterIndex=1}^{\numIter} \bx(\iterIndex)
\]
Then it holds that
\begin{align}
\nabla^{2}_{\bx\bx}f(\bx^{*})(\bx(\iterIndex) - \bx^{*}) &= \nabla_{\bx} \lossFunc(\bx(\iterIndex-1),\bz(\iterIndex-1)) - \nabla_{\bx} \lossFunc(\bx^{*},\bz(\iterIndex-1)) \nonumber \\
&\;\;\;\;\;\;\;\;\;\;\;\;\;\;\;\; + \left[ \nabla_{\bx \bx}^{2} f(\bx^{*}) - \nabla_{\bx \bx}^{2} \lossFunc(\bx^{*},\bz(\iterIndex-1))  \right] \left( \bx(\iterIndex-1) - \bx^{*} \right) \nonumber
\end{align}
yielding
\begin{align}
\nabla^{2}_{\bx\bx}f(\bx^{*})(\bar{\bx}(\iterIndex) - \bx^{*}) &= \frac{1}{\numIter} \sum_{\iterIndex=1}^{\numIter} \nabla_{\bx} \lossFunc(\bx(\iterIndex-1),\bz(\iterIndex-1)) - \frac{1}{\numIter} \sum_{\iterIndex=1}^{\numIter} \nabla_{\bx} \lossFunc(\bx^{*},\bz(\iterIndex-1)) \nonumber \\
&\;\;\;\;\;\;\;\;\;\;\;\; + \frac{1}{\numIter} \sum_{\iterIndex=1}^{\numIter} \left[ \nabla_{\bx \bx}^{2} f(\bx^{*}) - \nabla_{\bx \bx}^{2} \lossFunc(\bx^{*},\bz(\iterIndex-1))  \right] \left( \bx(\iterIndex-1) - \bx^{*} \right) \nonumber
\end{align}
First, we have
\begin{eqnarray}
\frac{1}{\numIter} \sum_{\iterIndex=1}^{\numIter} \nabla_{\bx} \lossFunc(\bx(\iterIndex-1),\bz(\iterIndex-1)) &=& \frac{1}{\numIter} \sum_{\iterIndex=1}^{\numIter} \nabla_{\bx} \lossFunc(\bx(\ell-1),\bz(\ell-1)) \nonumber \\
&=& \frac{1}{\numIter} \sum_{\iterIndex=1}^{\numIter} \frac{1}{\mu(\iterIndex)} (\bx(\ell-1)-\bx(\ell)) \nonumber \\
&=& \frac{1}{\numIter} \sum_{\iterIndex=1}^{\numIter} \frac{1}{\mu(\iterIndex)} (\bx(\ell-1)-\bx^{*}) - \frac{1}{\numIter} \sum_{\iterIndex=1}^{\numIter} \frac{1}{\mu(\iterIndex)} (\bx(\ell) - \bx^{*}) \nonumber \\
&=& \frac{1}{\numIter} \sum_{\iterIndex=1}^{\numIter-1} \left( \frac{1}{\mu(\iterIndex+1)} - \frac{1}{\mu(\iterIndex)} \right) (\bx(\ell)-\bx^{*}) + \frac{1}{\mu(1)} (\bx(0)-\bx^{*}) \nonumber \\
&& \;\;\;\;\; - \frac{1}{\mu(\numIter)} (\bx(\numIter)-\bx^{*}) \nonumber
\end{eqnarray}
Second, we have
\begin{eqnarray}
\mathbb{E}\bigg\| \frac{1}{\numIter} \sum_{\iterIndex=1}^{\numIter} \nabla_{\bx} \lossFunc(\bx^{*},\bz(\iterIndex-1))  \bigg\|^{2} &=& \frac{1}{\numIter^{2}} \sum_{\iterIndex=1}^{\numIter} \mathbb{E}\| \nabla_{\bx} \lossFunc(\bx^{*},\bz(\iterIndex-1)) \|^{2} \nonumber \\
&\leq& \frac{A}{n^{2}} \nonumber
\end{eqnarray}
Third, we have
\begin{eqnarray}
\mathbb{E}\bigg\| \frac{1}{\numIter} \sum_{\iterIndex=1}^{\numIter} \left[ \nabla_{\bx \bx}^{2} f(\bx^{*}) - \nabla_{\bx \bx}^{2} \lossFunc(\bx^{*},\bz(\iterIndex-1))  \right] ( \bx(\iterIndex-1) - \bx^{*})   \bigg\|^{2} &\leq& \frac{2B}{\numIter^{2}} \sum_{\iterIndex=1}^{\numIter} \mathbb{E}[d(\iterIndex-1)] \nonumber
\end{eqnarray} 
Combining these bounds with Minkowski's inequality yields
\begin{align}
&\left( m \mathbb{E}[\bar{d}(\numIter)] \right)^{1/2} \nonumber \\
&\;\; \leq \left(  \mathbb{E}\| \nabla_{\bx \bx}^{2} f(\bx^{*}) (\bar{\bx}(\numIter) - \bx^{*}) \|^{2} \right)^{1/2} \nonumber \\
&\;\; \leq  \sum_{\iterIndex=1}^{\numIter-1} \bigg| \frac{1}{\mu(\iterIndex+1)}  - \frac{1}{\mu(\iterIndex)} \bigg| \left( \mathbb{E}[d(\iterIndex)] \right)^{1/2} + \frac{1}{\mu(1)}  \left( \mathbb{E}[d(0)] \right)^{1/2} + \frac{1}{\mu(\numIter)} \left( \mathbb{E}[d(\numIter)] \right)^{1/2} \nonumber \\
&\;\;\;\;\;\;\;\;\;\;\;\; + \sqrt{\frac{A}{\numIter}} +  \sqrt{\frac{2B}{\numIter^{2}} \sum_{\iterIndex=1}^{\numIter} \mathbb{E}[d(\iterIndex-1)]} \nonumber
\end{align}
Then we have
\begin{align}
&\left(\mathbb{E}[\bar{d}(\numIter)] \right)^{1/2} \nonumber \\
&\;\; \leq \frac{1}{m^{1/2}} \sum_{\iterIndex=1}^{\numIter-1} \bigg| \frac{1}{\mu(\iterIndex+1)}  - \frac{1}{\mu(\iterIndex)} \bigg| \left( \mathbb{E}[d(\iterIndex)] \right)^{1/2} + \frac{1}{m^{1/2} \mu(1)}  \left( \mathbb{E}[d(0)] \right)^{1/2} + \frac{1}{m^{1/2} \mu(\numIter)} \left( \mathbb{E}[d(\numIter)] \right)^{1/2} \nonumber \\
&\;\;\;\;\;\;\;\;\;\;\;\; + \sqrt{\frac{A}{ m \numIter}} +  \sqrt{\frac{2B}{m \numIter^{2}} \sum_{\iterIndex=1}^{\numIter} \mathbb{E}[d(\iterIndex-1)]} \nonumber
\end{align}
\end{proof}
This decays at rate $\mathcal{O}\left(\frac{1}{\numIter}\right)$ as long as $\mu(\ell) = C\ell^{-\alpha}$ with $\frac{1}{2} \leq \alpha \leq 1$.

\section{Useful Concentration Inequalities}
\label{usefulConcIneq} 

For our analysis of both the direct and IPM estimates, we need the following key technical lemma from \cite{Antonini2005}. This lemma controls the concentration of sums of random variables that are sub-Gaussian conditioned on a particular filtration $\{\mathcal{F}_{i}\}_{i=0}^{n}$. Such a collection of random variables is referred to as a \emph{sub-Gaussian martingale sequence}. We include the proof for completeness.

\begin{lem}[Theorem 7.5 of \cite{Antonini2005}]
\label{subgauss:subgaussDepLem}
Suppose we have a collection of random variables $\{V_{i}\}_{i=1}^{n}$ and a filtration $\{\mathcal{F}_{i}\}_{i=0}^{n}$ such that for each random variable $V_{i}$ it holds that
\begin{enumerate}
\item $\mathbb{E}\left[ e^{s V_{i}} \;\big|\; \mathcal{F}_{i-1}  \right] \leq e^{\frac{1}{2}\sigma_{i}^{2}s^{2}}$ with $\sigma_{i}^{2}$ a constant
\item $V_{i}$ is $\mathcal{F}_{i}$-measurable
\end{enumerate}
Then for every $\bm{a} \in \mathbb{R}^{n}$ it holds that
\[
\mathbb{P}\left\{ \sum_{i=1}^{n} a_{i} V_{i} > t \right\} \leq \exp\left\{ -\frac{t^{2}}{2 \nu} \right\} \;\;\;\; \forall t > 0
\]
and
\[
\mathbb{P}\left\{ \sum_{i=1}^{n} a_{i} V_{i} < -t \right\} \leq \exp\left\{ -\frac{t^{2}}{2 \nu} \right\} \;\;\;\; \forall t > 0
\]
with
\[
\nu = \sum_{i=1}^{n} \sigma_{i}^{2} a_{i}^{2}
\]
\end{lem}
\begin{proof}
We bound the moment generating function of $\sum_{i=1}^{n} a_{i} V_{i}$ by induction. As a base case, we have
\begin{eqnarray}
\mathbb{E}\left[ e^{s a_{1} V_{1}}  \right] &=& \mathbb{E}\left[ \mathbb{E}\left[  e^{s a_{1} V_{1}} \;\Big|\; \mathcal{F}_{0} \right]  \right] \nonumber \\
&\leq& e^{\frac{1}{2} \sigma_{1}^{2} a_{1}^{2} s^{2}} \nonumber
\end{eqnarray}
Assume for induction that we have
\[
\mathbb{E}\left[ \exp\left\{ s \sum_{i=1}^{j} a_{i} V_{i} \right\}  \right] \leq \exp\left\{ \frac{1}{2} \left( \sum_{i=1}^{j} \sigma_{i}^{2} a_{i}^{2} \right) s^{2} \right\}
\]
Then we have
\begin{eqnarray}
\mathbb{E}\left[ \exp\left\{ \sum_{i=1}^{j+1} a_{i} V_{i} \right\}  \right] &=& \mathbb{E}\left[ \exp\left\{ s \sum_{i=1}^{j} a_{i} V_{i} \right\} e^{s a_{j+1}X_{j+1}}  \right] \nonumber \\
&=& \mathbb{E} \left[ \mathbb{E}\left[ \exp\left\{ s \sum_{i=1}^{j} a_{i} V_{i} \right\} e^{s a_{j+1}X_{j+1}} \;\Big|\; \mathcal{F}_{j+1} \right] \right] \nonumber \\
&\overset{(\text{a})}{=}& \mathbb{E} \left[ \exp\left\{ s \sum_{i=1}^{j} a_{i} V_{i} \right\} \mathbb{E}\left[ e^{s a_{j+1}X_{j+1}} \;\Big|\; \mathcal{F}_{j+1} \right] \right] \nonumber \\
&\overset{(\text{b})}{\leq}& \mathbb{E} \left[ \exp\left\{ s \sum_{i=1}^{j} a_{i} V_{i} \right\} \right] e^{\frac{1}{2} \sigma_{j+1}^{2} a_{j+1}^{2} s^{2}} \nonumber \\
&\overset{(\text{c})}{\leq}& \exp\left\{ \frac{1}{2} \left( \sum_{i=1}^{j+1} \sigma_{i}^{2} a_{i}^{2} \right) s^{2} \right\} \nonumber
\end{eqnarray}
where (a) follows since $\sum_{i=1}^{j} a_{i} V_{i}$ is $\mathcal{F}_{j}$ measurable, (b) follows since
\[
\mathbb{E}\left[ e^{s a_{j+1}X_{j+1}} \;\Big|\; \mathcal{F}_{j+1} \right] \leq e^{\frac{1}{2} \sigma_{j+1}^{2} a_{j+1}^{2} s^{2}},
\]
and (c) is the inductive assumption. This proves that
\[
\mathbb{E}\left[ \exp\left\{ s \sum_{i=1}^{n} a_{i} V_{i} \right\} \right] \leq \exp\left\{ \frac{1}{2} \left( \sum_{i=1}^{n} \sigma_{i}^{2} a_{i}^{2}  \right) s^{2} \right\} \leq \exp\left\{\frac{1}{2} \nu s^{2} \right\}
\]
Using the Chernoff bound \cite{Boucheron13}, we have
\[
\mathbb{P}\left\{ \sum_{i=1}^{n} a_{i} V_{i} > t \right\} \leq e^{-st} \mathbb{E}\left[ \exp\left\{ s \sum_{i=1}^{n} a_{i} V_{i} \right\}  \right] \leq \exp\left\{ -st + \frac{1}{2} \nu s^{2} \right\}
\]
Optimizing the bound over $s$ yields
\[
\mathbb{P}\left\{ \sum_{i=1}^{n} a_{i} V_{i} > t \right\} \leq \exp\left\{ - \frac{t^{2}}{2 \nu} \right\}
\]
The proof for the other tail is similar.
\end{proof}
If the random variables instead satisfy
\begin{enumerate}
\item $\mathbb{E}\left[ \exp\left\{ s \left( V_{i} - \mathbb{E}\left[ V_{i}  \;\big|\; \mathcal{F}_{i-1}  \right]\right) \right\} \;\big|\; \mathcal{F}_{i-1}  \right] \leq e^{\frac{1}{2}\sigma_{i}^{2}s^{2}}$ with $\sigma_{i}^{2}$ a constant
\item $V_{i}$ is $\mathcal{F}_{i}$-measurable
\end{enumerate}
then Lemma~\ref{subgauss:subgaussDepLem} can be applied to $\{V_{i} - \mathbb{E}\left[ V_{i}  \;\big|\; \mathcal{F}_{i-1}  \right]\}_{i=1}^{n}$ to yield
\[
\mathbb{P}\left\{ \sum_{i=1}^{n} a_{i} V_{i} > \sum_{i=1}^{n} a_{i} \mathbb{E}\left[ V_{i}  \;\big|\; \mathcal{F}_{i-1}  \right] + t \right\} \leq \exp\left\{ - \frac{t^{2}}{2 \nu} \right\}
\]
If we can upper bound the conditional expectations
\[
\mathbb{E}\left[ V_{i}  \;\big|\; \mathcal{F}_{i-1}  \right] \leq C_{i},
\]
by $\mathcal{F}_{i-1}$-measurable random variables $C_{i}$, then we have
\begin{equation*}
\mathbb{P}\left\{ \sum_{i=1}^{n} a_{i} V_{i} > \sum_{i=1}^{n} a_{i} C_{i} + t \right\}  \leq \mathbb{P}\left\{ \sum_{i=1}^{n} a_{i} V_{i} > \sum_{i=1}^{n} a_{i} \mathbb{E}\left[ V_{i}  \;\big|\; \mathcal{F}_{i-1}  \right] + t \right\} \leq \exp\left\{ - \frac{t^{2}}{2 \nu} \right\}
\end{equation*}
For our analysis, we generally cannot compute $ \mathbb{E}\left[ V_{i}  \;\big|\; \mathcal{F}_{i-1}  \right]$, but we can find ``nice'' $C_{i}$.

To find $\sigma^{2}_{i}$ for use in Lemma~\ref{subgauss:subgaussDepLem}, we frequently use the following conditional version of Hoeffding's Lemma.

\begin{lem}[Conditional Hoeffding's Lemma]
\label{estRho:condHoeffdingLemma}
If a random variable $V$ and a sigma algebra $\mathcal{F}$ satisfy $a \leq V \leq b$ and $\mathbb{E}[V|\mathcal{F}] = 0$, then
\[
\mathbb{E}\left[ e^{sV} \;|\; \mathcal{F} \right] \leq \exp\left\{ \frac{1}{8} (b-a)^{2} s^{2} \right\}
\]
\end{lem}
\begin{proof}
We follow standard proof of Hoeffding's Lemma from \cite{Boucheron13}. Since $e^{sx}$ is convex, it follows that
\[
e^{sx} \leq \frac{b-x}{b-a}e^{sa} + \frac{x-a}{b-a} e^{sb} \;\;\;\; a \leq x \leq b
\]
Therefore, taking the conditional expectation with respect to $\mathcal{F}$ yields
\begin{equation}\
\label{estRho:condHoeffdingLemma:firstBnd}
\mathbb{E}\left[ e^{sV} \;\big|\; \mathcal{F} \right] \leq \frac{b-\mathbb{E}\left[ V \;|\; \mathcal{F} \right]}{b-a}e^{sa} + \frac{\mathbb{E}\left[ V \;|\; \mathcal{F} \right]-a}{b-a} e^{sb}
\end{equation}
Let $h=s(b-a)$, $p=-\frac{a}{b-a}$, and $L(h)=-hp+\log(1-p+pe^{h})$. Then we have
\begin{eqnarray}
\label{estRho:condHoeffdingLemma:secondBnd}
e^{L(h)} &=& \frac{b}{b-a}e^{sa} + \frac{-a}{b-a} e^{sb} \nonumber \\
&=& \frac{b-\mathbb{E}\left[ V \;|\; \mathcal{F} \right]}{b-a}e^{sa} + \frac{\mathbb{E}\left[ V \;|\; \mathcal{F} \right]-a}{b-a} e^{sb}
\end{eqnarray}
since $\mathbb{E}\left[ V \;|\; \mathcal{F} \right] = 0$. Since $L(h) = L'(h) = 0$ and $L''(h) \leq \frac{1}{4},$, it holds that $L(h) \leq \frac{1}{8} (b-a)^{2} s^{2}$. Combining this bound on $L(h)$ with \eqref{estRho:condHoeffdingLemma:firstBnd} and \eqref{estRho:condHoeffdingLemma:secondBnd} yields the result.
\end{proof}

Before proceeding with our analysis, we need to introduce a few useful concentration inequalities for sub-Gaussian vector-valued random variables. First, for a scalar random variable $\xi$, define  the sub-Gaussian norm
\begin{equation}
\label{concDep:tauDef}
\tau(\xi) = \inf\left\{ a>0 \;\bigg|\; \mathbb{E}[ e^{s \xi} ] \leq e^{\frac{1}{2}a^{2}s^{2}} \;\; \forall s \geq 0 \right\}
\end{equation}
Clearly, if $\tau(\xi) < + \infty$, then $\xi$ is sub-Gaussian. Second, for a random vector $\bv$ in $\mathbb{R}^{d}$, define
\begin{equation}
\label{concDep:tauBDef}
B(\bv) = \sum_{i=1}^{d} \tau((\bv)_{i})
\end{equation}
where $(\bv)_{i}$ is the $i^{\text{th}}$ component of $\bv$. We define $\bv$ to be sub-Gaussian if $B(\bv) < +\infty$. 

Of crucial importance in our analysis is analyzing the norm of an average of vector-valued sub-Gaussian random variables. The following lemma describes how to control the sub-Gaussian norm in such a situation.
\begin{lem}
\label{subgauss:aveTauLemma}
Suppose that $\{\bv_{i}\}_{i=1}^{\numIter}$ is a collection of independent sub-Gaussian random variables in $\mathbb{R}^{d}$. Then it holds that
\[
B \left( \frac{1}{\numIter} \sum_{i=1}^{\numIter} \bv_{i} \right) \leq \frac{1}{\numIter} \sum_{j=1}^{d} \sqrt{ \sum_{i=1}^{\numIter} \tau^{2}((\bv_{i})_{j})}
\]
If in addition the random variables $\{\bv_{i}\}_{i=1}^{\numIter}$ satisfy
\[
\max_{i=1,\ldots,\numIter} \max_{j=1,\ldots,d} \tau^{2}((\bv_{i})_{j}) \leq \tau^{2}
\]
then it holds that
\[
B\left( \frac{1}{\numIter} \sum_{i=1}^{\numIter} \bv_{i}  \right) \leq \frac{\tau d}{\sqrt{\numIter}}
\]
\end{lem}
\begin{proof}
We analyze one component of the sum $\frac{1}{\numIter} \sum_{i=1}^{\numIter} \bv_{i}$. It holds that
\begin{eqnarray}
\mathbb{E}\left[ \exp\left\{ s \left( \frac{1}{\numIter} \sum_{i=1}^{\numIter} \bv_{i} \right)_{j} \right\}  \right] &=& \mathbb{E}\left[ \exp\left\{ \frac{s}{\numIter} \sum_{i=1}^{\numIter} (\bv_{i})_{j} \right\}  \right] \nonumber \\
&=& \prod_{i=1}^{\numIter} \mathbb{E}\left[ \exp\left\{ \frac{s}{\numIter} (\bv_{i})_{j} \right\}  \right] \nonumber \\
&\leq& \prod_{i=1}^{\numIter} \exp\left\{ \frac{1}{2} \frac{1}{\numIter^{2}} \tau^{2}((\bv_{i})_{j}) s^{2} \right\} \nonumber \\
&=& \exp\left\{ \frac{1}{2} \left( \frac{1}{\numIter^{2}} \sum_{i=1}^{\numIter} \tau^{2}((\bv_{i})_{j}) \right) s^{2} \right\} \nonumber
\end{eqnarray}
This implies that
\[
\tau\left( \left( \frac{1}{\numIter} \sum_{i=1}^{\numIter} \bv_{i} \right)_{j} \right) \leq  \frac{1}{\numIter} \sqrt{ \sum_{i=1}^{\numIter} \tau^{2}((\bv_{i})_{j})}
\]
and so
\[
B \left( \frac{1}{\numIter} \sum_{i=1}^{\numIter} \bv_{i} \right) \leq \frac{1}{\numIter} \sum_{j=1}^{d} \sqrt{ \sum_{i=1}^{\numIter} \tau^{2}((\bv_{i})_{j})}
\]
Finally, if $\tau^{2}((\bv_{i})_{j}) \leq \tau^{2}$, then we have
\begin{eqnarray}
B \left( \frac{1}{\numIter} \sum_{i=1}^{\numIter} \bv_{i} \right) &\leq& \frac{1}{\numIter} \sum_{j=1}^{d} \sqrt{ \sum_{i=1}^{\numIter} \tau^{2}((\bv_{i})_{j})} \nonumber \\
&\leq& \frac{d}{\numIter} \sqrt{ \sum_{i=1}^{\numIter} \tau^{2}} \nonumber \\
&=& \frac{\tau d}{\sqrt{\numIter}} \nonumber
\end{eqnarray}
\end{proof}

Example 3.2 from \cite{Buld2010}, a consequence of Theorem 3.1 in \cite{Buld2010}, is useful for the concentration of the norm of sub-Gaussian vector random variables. 
\begin{lem}[Example 3.2 of \cite{Buld2010}]
\label{subgauss:buldLemma}
If $\bv$ is a random vector in $\mathbb{R}^{d}$ with $B(\bv) < +\infty$, then
\[
\mathbb{P}\left\{ \| \bv \| > t \right\} \leq 2 \exp\left\{ - \frac{t^2}{2 B^{2}(\bv)} \right\}
\]
\end{lem}

Finally, we will also need to deal with dependent random variables that are sub-Gaussian with respect to a particular filtration.
\begin{lem}
\label{subgauss:subgaussConsts}
Suppose that a random variable $V$ and a sigma algebra $\mathcal{F}$ satisfies
\begin{enumerate}
\item $\mathbb{E}\left[ V \;|\; \mathcal{F} \right] = 0$
\item $\mathbb{P}\left\{ |V| > t \;\big|\; \mathcal{F} \right\} \leq 2 e^{-ct^2}$ with $c$ a constant.
\end{enumerate}
Then it holds that
\[
\mathbb{E}[e^{sV} \;\big|\; \mathcal{F} ] \leq \exp\left\{\frac{1}{2} \left( \frac{9}{c} \right) s^{2} \right\}
\]
for all $s \geq 0$.
\end{lem}
\begin{proof}
Adapted from the  characterization of sub-Gaussian random variables in \cite{Very2012}. First, we have for any $a < c$ that
\begin{eqnarray}
\mathbb{E}\left[ e^{a V^{2}} \;\Big|\; \mathcal{F}  \right] &\leq& 1 + \int_{0}^{\infty} 2 a t e^{a t^2} \mathbb{P}\left\{ |V| > t \;|\; \mathcal{F} \right\} dt \nonumber \\
&\leq& 1 + \int_{0}^{\infty} 2 a t e^{-(c -a) t^2} dt \nonumber \\
&=& 1 + \frac{2a}{c - a} \nonumber
\end{eqnarray}
Setting $a = \frac{c}{3}$ yields the bound
\[
\mathbb{E}\left[ e^{a V^{2}} \;\Big|\; \mathcal{F}  \right] \leq 2
\]
Since $\mathbb{E}\left[ V \;|\; \mathcal{F} \right] = 0$, by a Taylor expansion we have
\begin{eqnarray}
\mathbb{E}\left[ e^{sV} \;\big|\; \mathcal{F}  \right] &=& 1 + \int_{0}^{\infty} (1-y) \mathbb{E}\left[ (sV)^{2} e^{ys V}  \;\Big|\; \mathcal{F} \right] dy \nonumber \\
&\leq& \left(1 + \frac{s^{2}}{a} \right) e^{\frac{s^{2}}{2a}} \nonumber \\
&\leq& \exp\left\{ \frac{5 s^{2}}{2 a}  \right\} \nonumber \\
&=& \exp\left\{ \frac{1}{2} \left( \frac{9}{c} \right) s^{2} \right\} \nonumber
\end{eqnarray}
\end{proof}

\end{document}